\documentclass[twoside,11pt]{article}

\usepackage{amsthm}
\usepackage{jmlr2e}

\usepackage{times}
\usepackage{epsfig}
\usepackage{graphicx}
\usepackage{amsmath}
\usepackage{amssymb}

\usepackage{color, listings, comment}
\usepackage{algorithm}
\usepackage{algorithmic}
\usepackage{bm}
\usepackage{subfigure}

\newcommand{\norm}[1]{\left\|#1\right\|}

\newtheorem*{theorem*}{Theorem}
\newtheorem*{corollary*}{Corollary}

\usepackage{lastpage}
\jmlrheading{23}{2022}{1-\pageref{LastPage}}{12/20}{7/22}{20-1438}
{Congliang Chen, Li Shen, Fangyu Zou and Wei Liu}

\ShortHeadings{Towards Practical Adam}{Chen, Shen, Zou and Liu}
\firstpageno{1}

\begin{document}

\title{Towards Practical Adam: Non-Convexity, Convergence Theory, and Mini-Batch Acceleration\footnote{}}

\author{\name Congliang Chen $^\dag$\email congliangchen@link.cuhk.edu.cn \\
       \addr The Chinese University of Hong Kong, Shenzhen
       \AND
       \name Li \ Shen $^\ddag$\thanks{The first three authors contribute equally.\quad $\ddag.$ Corresponding authors.\\
       $*.$ This paper is an extension of \citep{zou2019sufficient}, which was part of our CVPR 2019 paper. In this version, we renew our analysis technique that yields a cleaner convergence rate compared with \citep{zou2019sufficient}. In addition, we also derive the linear speedup properties for mini-batch Adam and distributed Adam in the Parameter-server model.  This work is done when Congliang Chen is a research intern at Tencent AI Lab, China. } \email mathshenli@gmail.com \\
       \addr  JD Explore Academy 
       \AND
       \name Fangyu\ Zou $^\dag$ \email fangyuzou@gmail.com  \\
       \addr Meta
      \AND
       \name Wei \ Liu$^\ddag$ \email wl2223@columbia.edu \\
       \addr Tencent
       }

\editor{Sanjiv Kumar}

\maketitle

\begin{abstract}
Adam is one of the most influential adaptive stochastic algorithms for training deep neural networks, which has been pointed out to be divergent even in the simple convex setting via a few simple counterexamples. 
Many attempts, such as decreasing an adaptive learning rate, adopting a big batch size, incorporating a temporal decorrelation technique, seeking an analogous surrogate, \textit{etc.}, have been tried to promote Adam-type algorithms to converge. 
In contrast with existing approaches, we introduce an alternative easy-to-check sufficient condition, which merely depends on the parameters of the base learning rate and combinations of historical second-order moments, to guarantee the global convergence of generic Adam for solving large-scale non-convex stochastic optimization. 
This observation, coupled with this sufficient condition, gives much deeper interpretations on the divergence of Adam. On the other hand, in practice, mini-Adam and distributed-Adam are widely used without any theoretical guarantee. We further give an analysis on how the batch size or the number of nodes in the distributed system affects the convergence of Adam, which theoretically shows that mini-batch and distributed Adam can be linearly accelerated by using a larger mini-batch size or a larger number of nodes.
At last, we apply the generic Adam and mini-batch Adam with the sufficient condition for solving the counterexample and training several neural networks on various real-world datasets. Experimental results are exactly in accord with our theoretical analysis.

\end{abstract}


\section{Introduction}
Large-scale non-convex stochastic optimization \citep{bottou2018optimization}, covering a slew of applications in statistics and machine learning \citep{jain2017non,bottou2018optimization} such as learning a latent variable from massive data whose probability density distribution is unknown, takes the following generic formulation:
\begin{align}\label{minimization}
\min_{\bm{x} \in \mathbb{R}^{d}}\ f(\bm{x}) = \mathbb{E}_{\xi\sim \mathbb{P}}\,\big[\widetilde{f}(\bm{x},\xi)\big],
\end{align}
where $f(\bm{x})$ is a non-convex function and $\xi$ is a random variable with an unknown distribution $\mathbb{P}$.

Alternatively, a compromised approach to handle this difficulty is to use an unbiased stochastic estimate of $\bm{\nabla}\!{f}(\bm{x})$, denoted as $g(\bm{x},\xi)$, which leads to the stochastic gradient descent (SGD) algorithm \citep{robbins1985stochastic}. Its coordinate-wise version is defined as follows:
\begin{align}\label{sgd}
\bm{x}_{t+1,k} = \bm{x}_{t,k} - \eta_{t,k}\bm{g}_{t,k}(\bm{x}_{t},\xi_{t}),
\end{align}
for $ k =1,2,\ldots,d$, where $\eta_{t,k} \ge 0$ is the learning rate of the $k$-th component of stochastic gradient  $\bm{g}(\bm{x}_{t},\xi_{t})$ at the $t$-th iteration. \textcolor{black}{Under some mild assumptions (e.g., the optimal solution exists), a}  sufficient condition \citep{robbins1985stochastic} to ensure the global convergence of vanilla SGD in Eq.~\eqref{sgd} is to require $\eta_{t}$ to meet the following diminishing condition:  
\begin{align}\label{surfficient-sgd}
\sum_{t=1}^{\infty}\|\eta_{t}\| = \infty\ {\rm\ and\ } \ \sum\limits_{t=1}^{\infty}\|\eta_{t}\|^2 < \infty.
\end{align}
Although the vanilla SGD algorithm with learning rate $\eta_{t}$ satisfying condition \eqref{surfficient-sgd} does converge, its empirical performance could be still stagnating,  since it is difficult to tune an effective learning rate $\eta_{t}$ via condition \eqref{surfficient-sgd}.   

To further improve the empirical performance of SGD, a large variety of adaptive SGD algorithms, including AdaGrad \citep{duchi2011adaptive}, RMSProp \citep{hinton2012neural}, Adam \citep{kingma2014adam}, Nadam \citep{dozat2016incorporating}, AdaBound \citep{luo2019adaptive}, \textit{etc.}, have been proposed to automatically tune the learning rate $\eta_{t}$ by using second-order moments of historical stochastic gradients $\{\bm{g}_{t}\}$. 
Let $\bm{v}_{t,k}$ and $\bm{m}_{t,k}$ be the \textcolor{black}{exponential moving average of } the historical second-order moments $(\bm{g}^2_{1,k},\bm{g}^2_{2,k},\cdots,\bm{g}^2_{t,k})$ and stochastic gradient estimates $(\bm{g}_{1,k},\bm{g}_{2,k},\cdots,\bm{g}_{t,k})$, respectively. \textcolor{black}{More specifically, two groups of hyperparameters ($\beta_t$, $\theta_t$) will be involved into the calculation of $m_{t,k} = \beta_t m_{t-1,k} + (1-\beta_t)g_{t,k}$ and $v_{t,k} = \theta_t v_{t-1,k} + (1-\theta_t)g^2_{t,k} $.}
Then, the generic iteration scheme of these adaptive SGD algorithms \citep{Reddi2018on,chen2018convergence} is summarized as 
\begin{equation}\label{generic-adaptive-sgd}
\!\! \bm{x}_{t+1,k}\!= \bm{x}_{t,k} - \eta_{t,k}\bm{m}_{t,k},\ {\rm\ with\ }\eta_{t,k}\!= {\alpha_{t}}/{\sqrt{\bm{v}_{t,k}}},\!\!
\end{equation}
for $k =1,2,\ldots,d$, where $\alpha_{t} >0 $ is called base learning rate and it is independent of stochastic gradient estimates $(\bm{g}_{1,k},\bm{g}_{2,k},\cdots,\bm{g}_{t,k})$ for all $t\ge 1$. 
Although Adam works well for solving large scale convex and non-convex optimization problems such as training deep neural networks, it has been disclosed to be divergent in some scenarios via counterexamples \citep{Reddi2018on}.
\textcolor{black}{Thus, without any further assumptions for corrections, Adam should not be directly used.}
Recently, developing sufficient conditions to guarantee global convergences of Adam -type algorithms has attracted much attention from both machine learning and optimization communities. The existing successful attempts can be divided into four categories: decreasing a learning rate, adopting a big batch size, incorporating a temporal decorrelation, and seeking an analogous surrogate. However, some of them are either hard to check or impractical. In this work, we will first introduce an alternative easy-to-check sufficient condition to guarantee the global convergences of the original Adam.

Meanwhile, in practice, stochastic Adam, where a single sample is used to estimate gradient, \textcolor{black}{converges slowly to the optimal point}. People usually use mini-batch Adam instead to get faster convergence performance. In SGD, although how the sample size will affect the convergence has been well studied \citep{li2014efficient}, few works give analysis on mini-batch adaptive gradient methods especially on Adam, since mini-batch size largely affects adaptive learning rate $\eta_{t,k}$ in Eq~.\eqref{generic-adaptive-sgd}, which makes the analysis difficult. In this work, we give the first complexity analysis for mini-batch Adam, which shows that mini-batch Adam can also be theoretically accelerated by using a larger mini-batch size.

On the other hand, as the data size goes larger in machine learning problems, it is hard to collect, store and process data in a single machine. Several machines are involved in the optimization process. Hence, distributed optimization methods are proposed, where distributed Adam is also popularly used. Different from mini-batch Adam, where only one machine is used for optimization, several machines are involved. In the distributed setting, machines are connected via a network graph. More specifically, there are two kinds of structures used in distributed Adam: parameter-server structure and decentralized structure. In the parameter-server structure, there is one special machine called as parameter server and the rest called workers. The parameter server connects to all workers, but workers don't connect to each other. Therefore, workers can share information with the parameter server in each communication round but cannot share information with the other workers. However, in the decentralized structure, there is not a server involved in the structure. A pre-defined graph connects all machines. A machine can only share information with its direct neighbors in each communication round. Still, few works answer how the local batch size and number of machines will affect the convergence of distributed Adam. In this work, \textcolor{black}{because the analysis of distributed Adam under the parameter-server model is similar to Mini-batch Adam, } we answer this question and show that distributed Adam \textcolor{black}{under a parameter-server model} can also achieve a linear speedup property as distributed SGD \citep{yu2019linear}.   

In summary, the contributions of this work are five-fold:
\begin{itemize}
\item[(1)] We introduce an easy-to-check sufficient condition to ensure the global convergences \textcolor{black}{(i.e., averaged expected gradient norm converges to 0)} of generic Adam in the \textcolor{black}{common smooth non-convex stochastic setting with mild assumptions}. Moreover, this sufficient condition is distinctive from the existing conditions and is easier to verify.
\item[(2)] We provide a new explanation on the divergences of original Adam and RMSProp, which are possibly due to an incorrect parameter setting of the combinations of historical second-order moments. 
\item[(3)] We find that the sufficient condition extends the restrictions of RMSProp \citep{mukkamala2017variants} and covers many convergent variants of Adam, e.g., AdamNC, AdaGrad with momentum, \textit{etc.} Thus, their convergences in the non-convex stochastic setting naturally hold. 
\item[(4)] We theoretically show that mini-batch Adam can be further accelerated by adopting a larger mini-batch size, and that distributed Adam can achieve a linear speed up property in the parameter-server distributed system by using commonly used sufficient condition parameters. 
\item[(5)] We conduct experiments to validate the sufficient condition for the convergences of Adam and mini-batch Adam.  The experimental results match our theoretical results. 
\end{itemize}

The paper is organized as follows. In Section \ref{relatedwork}, we first give the formulation of generic Adam and then discuss several works related to Adam including several existing sufficient convergence conditions, analysis of mini-batch, and distributed stochastic gradient methods. In Section \ref{sufficient_condition}, we derive the sufficient condition for convergence of Adam and provide several insights for the divergence of vanilla Adam. In Section \ref{sec_pratical}, we give the complexity analysis on practical Adam with a commonly used sufficient condition parameter, including mini-batch Adam and distributed Adam.  At last, in Section \ref{experimental_result}, we conduct some experiments under both theoretical settings and practical settings to verify the established theory. \textcolor{black}{In addition, by practical Adam, we mean that we give a thorough analysis for Adam, mini-batch Adam, and distributed Adam, which have been commonly used for training deep neural networks without theoretical guarantees.}

\section{Related work}\label{relatedwork}

\subsection{Generic Adam}

For readers' convenience, we first clarify a few necessary notations used in the forthcoming Generic Adam.  First, we denote $\bm{x}_{t,k}$ as the $k$-th component of $\bm{x}_{t}\in\mathbb{R}^{d}$, and $\bm{g}_{t,k}$ as the $k$-th component of the stochastic gradient at the $t$-th iteration respectively,
and call $\alpha_{t} > 0$ base learning rate and $\beta_{t}$ momentum parameter, respectively. Let $\epsilon>0$ be a sufficiently small constant. Denote $\bm{0}=(0,\cdots,0)^{\top} \in \mathbb{R}^{d}$, and $\bm{\epsilon}=(\epsilon,\cdots,\epsilon)^{\top} \in \mathbb{R}^{d}$.
All operations, such as multiplying, dividing, and taking the square root, are executed in the coordinate-wise mode.

\begin{algorithm}[H]
\caption{\ Generic Adam}
\label{Adam}
\begin{algorithmic}[1]
   \STATE {\bf Parameters:} Set suitable base learning rate $\{\alpha_t\}$, momentum parameter $\{\beta_t\}$, and exponential moving average parameter $\{\theta_t\}$, respectively. Choose $\bm{x}_1 \in \mathbb{R}^d$ and set initial values $\bm{m}_0=\bm{0} \in \mathbb{R}^d$ and $\bm{v}_0=\bm{\epsilon} \in \mathbb{R}^d$.
   \FOR{$t= 1,\,2,\,\ldots,\,T$}
    \STATE Sample a stochastic gradient $\bm{g}_t$; 
        \FOR {$k=1,\,2,\,\ldots,\,d$}
        \STATE $\bm{v}_{t,k} = \theta_t \bm{v}_{t-1,k} + (1 - \theta_t) \bm{g}_{t,k}^2$;
        \STATE $\bm{m}_{t,k} = \beta_t \bm{m}_{t-1,k} + (1 - \beta_t) \bm{g}_{t,k}$;
        \STATE $\bm{x}_{t+1,k} = \bm{x}_{t,k} - {\alpha_t \bm{m}_{t,k}}/\sqrt{\bm{v}_{t,k}}$;
        \ENDFOR
   \ENDFOR
 \end{algorithmic}
 \end{algorithm}
It is not hard to check that Generic Adam covers RMSProp by setting $\beta_{t} =0$ directly. Moreover, it covers Adam with a bias correction \citep{kingma2014adam} as follows: 
\begin{remark}
The vanilla Adam with the bias correction \citep{kingma2014adam} takes constant parameters $\beta_{t}=\beta$ and $\theta_{t}=\theta$. The iteration scheme is written as
$\bm{x}_{t+1} = \bm{x}_{t} -\widehat{\alpha}_{t} \frac{\widehat{\bm{m}}_{t}}{\sqrt{\widehat{\bm{v}}_{t}}}$, with $\widehat{\bm{m}}_{t} = \frac{\bm{m}_{t}}{1-\beta^{t}}$ and $\widehat{\bm{v}}_{t} = \frac{\bm{v}_{t}}{1-\theta^{t}}$. Let $\alpha_t = \widehat{\alpha}_t\frac{\sqrt{1-\theta^t}}{1-\beta^t}$. Then, the above can be rewritten as $\bm{x}_{t+1} = \bm{x}_{t} - {\alpha_t \bm{m}_{t}}/\sqrt{\bm{v}_{t}}$. Thus, it is equivalent to taking constant $\beta_t$, constant $\theta_t$, and new base learning rate $\alpha_t$ in Generic Adam. 
\end{remark}

\subsection{Convergence Conditions for Adam}
First, because \citet{Reddi2018on} gave counterexamples on divergence of origin Adam, several sufficient conditions have been proposed to guarantee global convergences of Adam that can be summarized into the following four categories:

\smallskip
\noindent
{\bf (C1) Decreasing a learning rate.}\ ~ 
\citet{Reddi2018on} have declared that the core cause of divergences of Adam and RMSProp is largely controlled by the difference between the two adjacent learning rates, \textit{i.e.},  
\begin{equation}\label{Gamma_t}
\Gamma_{t} = {1}/{\bm{\eta}_{t}}-{1}/{\bm{\eta}_{t-1}}= {\sqrt{\bm{v}_{t}}}/{\alpha_{t}}-{\sqrt{\bm{v}_{t-1}}}/{\alpha_{t-1}}.
\end{equation}
Once positive definiteness of $\Gamma_{t}$ is violated, Adam and RMSProp may suffer from divergence \citep{Reddi2018on}. 
Based on this observation, two variants of Adam called AMSGrad and AdamNC have been proposed with convergence guarantees in both the convex \citep{Reddi2018on} and non-convex \citep{chen2018convergence} stochastic settings by requiring $\Gamma_{t}\succ 0$. 
In addition, Padam \citep{zhou2018convergence} extended from AMSGrad has been proposed to contract the generalization gap in training deep neural networks, whose convergence has been ensured by requiring $\Gamma_{t}\succ 0$.
As a relaxation of $\Gamma_{t} \succ 0$, \citet{barakat2020convergence} showed that when $\alpha_t/\sqrt{v_t} \leq \alpha_{t-1}/(c\sqrt{v_{t-1}})$ holds for all $t$ and some positive $c$, the algorithm Adam can converge.
In the strongly convex stochastic setting, by using the long-term memory technique developed in \citep{Reddi2018on}, \citet{huang2018nostalgic} have proposed NosAdam by attaching more weights on historical second-order moments to ensure its convergence. 
Prior to that, the convergence rate of RMSProp \citep{mukkamala2017variants} has already been established in the convex stochastic setting by employing similar parameters to those of AdamNC \citep{Reddi2018on}. 

\smallskip
\noindent
{\bf(C2) Adopting a big batch size.}\ ~ \citet{basu2018convergence}, for the first time, showed that deterministic Adam and RMSProp with original iteration schemes are convergent by using a full-batch gradient. 
On the other hand, both Adam and RMSProp can be reshaped as specific signSGD-type algorithms \citep{balles18aDissecting,bernstein2018signSGG} whose $\mathcal{O}(1/\sqrt{T})$ convergence rates have been provided in the non-convex stochastic setting by setting batch size as large as the number of maximum iterations \citep{bernstein2018signSGG}. 
Recently, \citet{Zaheer2018Adaptive} have established $\mathcal{O}(1/\sqrt{T})$ convergence rate of original Adam directly in the non-convex stochastic setting by requiring the batch size to be the same order as the number of maximum iterations. 
We comment that this type of requirement is impractical when Adam and RMSProp are applied to tackle large-scale problems \eqref{minimization}, since these approaches cost a huge number of computations to estimate big-batch stochastic gradients in each iteration.

\smallskip
\noindent
{\bf(C3) Incorporating a temporal decorrelation.}\ ~ By exploring the structure of the convex counterexample in \citep{Reddi2018on}, \citet{zhou2018adashift} have pointed out that the divergence of RMSProp is fundamentally caused by the imbalanced learning rate rather than the absence of $\Gamma_{t}\succ 0$.
Based on this viewpoint, \citet{zhou2018adashift} have proposed AdaShift by incorporating a temporal decorrelation technique to eliminate the inappropriate correlation between $\bm{v}_{t,k}$ and the current second-order moment $\bm{g}_{t,k}^2$, 
in which the adaptive learning rate $\eta_{t,k}$ is required to be independent of $\bm{g}^2_{t,k}$.
However, the convergence of AdaShift in \citep{zhou2018adashift} was merely restricted to RMSProp for solving the convex counterexample in \citep{Reddi2018on}.  

\smallskip
\noindent
{\bf (C4) Seeking an analogous surrogate.}\ ~Due to the divergences of Adam and RMSProp \citep{Reddi2018on},  \citet{zou2018convergence} proposed a class of new surrogates called AdaUSM to approximate Adam and RMSProp by integrating weighted AdaGrad with a unified heavy ball and Nesterov accelerated gradient momentums. 
Its $\mathcal{O}(\log{(T)}/\sqrt{T})$ convergence rate has also been provided in the non-convex stochastic setting by requiring a non-decreasing weighted sequence. 
Besides, many other adaptive stochastic algorithms without combining momentums, such as AdaGrad \citep{ward2018adagrad,li2019convergence} and stagewise AdaGrad \citep{chen2018universal}, have been guaranteed to be convergent and work well in the non-convex stochastic setting.

In contrast with the above four types of modifications and restrictions, we introduce an alternative easy-to-check sufficient condition (abbreviated as {\bf (SC)}) to guarantee the global convergences of original Adam.
The proposed {\bf (SC)} merely depends on the parameters in estimating $\bm{v}_{t,k}$ and base learning rate $\alpha_{t}$. {\bf (SC)} neither requires the positive definiteness of $\Gamma_{t}$ like {\bf(C1)} nor needs the batch size as large as the same order as the number of maximum iterations like {\bf(C2)} in both the convex and non-convex stochastic settings. 
Thus, it is easier to verify and more practical compared with {\bf (C1)}-{\bf(C3)}. 
On the other hand, {\bf (SC)} is partially overlapped with {\bf(C1)} since the proposed {\bf (SC)} can cover AdamNC \citep{Reddi2018on}, AdaGrad with exponential moving average  (AdaEMA) momentum \citep{chen2018convergence}, and RMSProp \citep{mukkamala2017variants} as instances whose convergences are all originally motivated by requiring the positive definiteness of $\Gamma_{t}$. 
While, based on {\bf (SC)}, we can directly derive their global convergences in the non-convex stochastic setting as byproducts without checking the positive definiteness of $\Gamma_{t}$ step by step.   
Besides, {\bf (SC)} can serve as an alternative explanation on divergences of original Adam and RMSProp, which are possibly due to incorrect parameter settings for accumulating the historical second-order moments rather than the imbalanced learning rate caused by the inappropriate correlation between $\bm{v}_{t,k}$ and $\bm{g}^2_{t,k}$ like {\bf(C3)}. 
In addition, AdamNC and AdaEMA are convergent under {\bf(SC)}, but violate {\bf(C3)} in each iteration.  \textcolor{black}{Meanwhile, there are lots of work improving upper bounds for the above algorithms, e.g., \cite{defossez2020simple} improved the constants related to $\beta$ by introducing a novel average scheme in the analysis.}

\subsection{Mini-batch Stochastic Gradient Methods}
In practice, people usually use mini-batch stochastic gradient methods instead of single sample stochastic gradient methods or full gradient methods for faster convergence. For mini-batch SGD algorithms, \citet{li2014efficient} have shown that mini-batch SGD boosts $\mathcal{O}(\frac{1}{\sqrt[4]{T}})$ convergence rate of SGD to $\mathcal{O}(\frac{1}{\sqrt[4]{sT}})$ where $s$ is the mini-batch size. However, as it is much harder to show the convergence of adaptive gradient methods, few works analyze how sample size will affect the convergence of the adaptive gradient algorithms. \citet{li2019convergence} gave an analysis on Adagrad and showed the convergence rate is linear in the sample size. \citet{Zaheer2018Adaptive} gave the analysis on Adam, showing that large batch size can help convergence, but the batch size should increase with iteration increasing, which may not be practical. In this work, we theoretically show that mini-batch Adam can be accelerated by adopting a larger mini-batch size as mini-batch SGD \citep{li2014efficient} in the same order.

\subsection{Distributed Stochastic Gradient Methods}
Distributed stochastic gradient descent was first introduced in \citet{agarwal2011distributed} in the parameter-server setting. Further, in the decentralized setting, \citet{lian2017can} gave the analysis on the stochastic gradient descent. The analysis shows that the convergence speed will be linear in the number of workers in the parameter-server setting or will be linear to some constant related to the decentralized graph structure. For the adaptive gradient methods, in the parameter-server setting, \citet{reddi2020adaptive} gave algorithms in the federated scenario called FedAdam, FedAdagrad, and FedYogi. Moreover, they showed that the convergence speed will be linear in the number of workers. However, \textcolor{black}{instead divided by $\sqrt{\bm{v}_t}$, they divide the gradient with $\sqrt{\bm{v_t}+\bm{\epsilon}}$.  Meanwhile,} in their assumptions, $\epsilon$ in the algorithm should be in the order of $O(\frac{G}{L})$, where $G$ is the upper bound of gradient norm, and $L$ is the Lipschitz constant of the objective function. However, in practice, $\epsilon$ is always set to be a small value, much smaller than $G/L$. On the other hand, the large $\epsilon$ may dominate the adaptive term in their algorithms. Hence, their methods may degrade to stochastic gradient descent. \citet{carnevale2020distributed} shoed that Adam with gradient tracking method can be linearly accelerated with an increasing number of nodes in the decentralized and strongly convex setting. Still, it is unclear whether, in the nonconvex setting, this linear speedup will hold when Adam is used. Moreover, \citet{chen2020toward} gave an analysis of Adagrad and showed the convergence speed will be linear in the number of workers. Meanwhile, \citet{nazari2019dadam} gave the analysis of Adagrad in the decentralized setting. \citet{xie2019local} also gave a variant on Adagrad algorithm called AdaAlter in the centralized setting and showed the convergence will linearly speed up by increasing the number of workers. 
Recently,  \citet{chen2021quantized,chen2022efficient} extend Adam to the distributed quantized Adam with error compensation technique \cite{stich2018sparsified}. However, the linear speedup property in \citep{chen2021quantized,chen2022efficient} does not hold. 
To the best of our knowledge, whether the distributed Adam can achieve a linear speedup is still open. This paper theoretically demonstrates that the distributed Adam in the parameter-server model can achieve a linear speedup concerning the number of workers.



\section{Novel Sufficient Condition for Convergence of Adam}
\label{sufficient_condition}

In this section, we characterize the upper-bound of gradient residual of problem \eqref{minimization} as a function of parameters $(\theta_t,\alpha_t)$. Then the convergence rate of Generic Adam is derived directly by specifying appropriate parameters $(\theta_{t},\alpha_{t})$. Below, we state the necessary assumptions that are commonly used for analyzing the convergence of a stochastic algorithm for non-convex problems: 
\begin{description}
\item[(A1)] The minimum value of problem \eqref{minimization} is lower-bounded, \textit{i.e.}, $f^* = \min_{\bm{x} \in \mathbb{R}^{d}}\ f(\bm{x}) > -\infty$;
\item[(A2)] The gradient of $f$ is $L$-Lipschitz continuous,
\textit{i.e.}, $\|\bm{\nabla}\!f(\bm{x})-\bm{\nabla}\!f(\bm{y})\|\!\le\! L\|\bm{x}\!-\!\bm{y}\|,\ \forall \bm{x},\bm{y}\!\in\! \mathbb{R}^{d}$;
\item[(A3)] The stochastic gradient $\bm{g}_{t}$ is an unbiased estimate, \textit{i.e.}, $\mathbb{E}\,[\bm{g}_{t}] = \bm{\nabla}\!{f}_{t}(\bm{x}_{t})$;
\item[(A4)] The second-order moment of stochastic gradient $\bm{g}_{t}$ is uniformly upper-bounded, \textit{i.e.}, $\mathbb{E}\,\|\bm{g}_{t}\|^2 \le G$.

\end{description}

In addition, we also suppose that the parameters $\{\beta_t\}$, $\{\theta_t\}$, and $\{\alpha_t\}$ satisfy the following restrictions:
\begin{description}
\item[(R1)] The parameters $\{\beta_t\}$ satisfy $0\leq \beta_t \leq \beta < 1$ for all $t$ for some constant $\beta$; 
\item[(R2)] The parameters $\{\theta_t\}$ satisfy $0 < \theta_t < 1$ and $\theta_t$ is non-decreasing in $t$ with $\theta := \lim\limits_{t\to\infty}\theta_t > \beta^2$;
\item[(R3)] The parameters $\{\alpha_t\}$ satisfy that $\chi_t := \frac{\alpha_t}{\sqrt{1-\theta_t}}$ is ``almost" non-increasing in $t$, by which we mean that there exist a non-increasing sequence $\{a_t\}$ and a positive constant $C_0$ independent of $t$ such that $a_t \leq \chi_t \leq C_0 a_t$.
\end{description}

The restriction (R3) indeed says that $\chi_t$ is the product between some non-increasing sequence $\{a_t\}$ and some bounded sequence. This is a slight generalization of $\chi_t$ itself being non-decreasing. If $\chi_t$ itself is non-increasing, we can then take $a_t = \chi_t$ and $C_0 = 1$. For most of the well-known Adam-type methods, $\chi_t$ is indeed non-decreasing. For instance, for AdaGrad with EMA momentum we have $\alpha_t = \eta/\sqrt{t}$ and $\theta_t = 1 - 1/t$, so $\chi_t = \eta$ is constant; for Adam with constant $\theta_t = \theta$ and non-increasing $\alpha_t$ (say $\alpha_t = \eta/\sqrt{t}$ or $\alpha_t = \eta$), $\chi_t = \alpha_t/\sqrt{1-\theta}$ is non-increasing.  The motivation, instead of $\chi_t$ being decreasing, is that it allows us to deal with the bias correction steps in Adam \citep{kingma2014adam}.

We fix a positive constant $\theta' >0$\footnote{In the special case that $\theta_t = \theta$ is constant, we can directly set $\theta' = \theta$.} such that $\beta^2 < \theta' < \theta$. Let $\gamma := {\beta^2}/{\theta'} < 1$ and
\begin{equation}\label{Constant-C1}
C_1 := \prod_{j=1}^N \big(\frac{\theta_j}{\theta'}\big),
\end{equation}
where $N$ is the maximum of the indices $j$ with $\theta_j < \theta'$. The finiteness of $N$ is guaranteed by the fact that $\lim_{t\to\infty} \theta_t = \theta > \theta'$. When there are no such indices, \textit{i.e.}, $\theta_1 \geq \theta'$, we take $C_1 = 1$ by convention. In general, $C_1 \leq 1$. 
Our main results on estimating gradient residual state as follows:
\begin{theorem}\label{convergence_in_expectation} 
Let $\{\bm{x}_t\}$ be a sequence generated by Generic Adam for initial values $\bm{x}_1$, $\bm{m}_0 =\bm{0}$, and $\bm{v}_0 =\bm{\epsilon}$. Assume that $f$ and stochastic gradients $\bm{g}_t$ satisfy assumptions (A1)-(A4). Let $\tau$ be randomly chosen from $\{1, 2, \ldots, T\}$ with equal probabilities $p_\tau = 1/T$. Then, we have
\begin{equation*}
\mathbb{E}\left[\norm{\bm{\nabla} f(\bm{x}_\tau)} \right]
\leq \sqrt{\frac{C + C'\sum_{t=1}^T\alpha_t\sqrt{1-\theta_t}}{T\alpha_T}},
\end{equation*}
where $C'\!=\!{2C_0^2C_3d\sqrt{G^2\!+\!\epsilon d}}{\big/}{[(1\!-\!\beta)\theta_1]}$ and 
\begin{equation*}
\begin{split}
C=\frac{2C_0\sqrt{G^2\!+\!\epsilon d}}{1-\beta}\big[(C_4\!+\!C_3C_0d\chi_1\log\big(1\!+\! \frac{G^2}{\epsilon d}\big)\big],
\end{split}
\end{equation*}
in which $C_4$ and $C_3$ are defined as $C_4 = f(x_1) - f^*$ and
$C_3 = \frac{C_0}{\sqrt{C_1}(1-\sqrt{\gamma})}\big[\frac{C_0^2\chi_1L}{C_1(1-\sqrt{\gamma})^2} + 2\big(\frac{\beta/(1-\beta)}{\sqrt{C_1(1-\gamma)\theta_1}}+1\big)^2G\big]$, respectively.
\end{theorem}


\begin{remark}
Below, we give two comments on the above Theorem:\\ 
{\bf(i)}\ The constants $C$ and $C'$ depend on priory known constants $C_0, C_1, \beta, \theta', G, L, \epsilon, d, f^*, \theta_1, \alpha_1, \bm{x}_1$. \\
{\bf(ii)}\ 
Convergence in expectation in the above theorem is on the term
$\mathbb{E}[\norm{\nabla f(\bm{x}_\tau)}]$ which
is slightly weaker than $\mathbb{E}[\norm{\nabla\! f(\bm{x}_\tau)}^2]$. The latter form holds for most SGD variants with global learning rates, namely, the learning rate for each coordinate is the same, because $\frac{1}{\sum_{t=1}^T\!\alpha_t}\mathbb{E}\sum_{t=1}^T\!\alpha_t \norm{\nabla\! f(\bm{x}_t)}^2$ is exactly $\mathbb{E}[\norm{\nabla\! f(\bm{x}_\tau)}^2]$ if $\tau$ is randomly selected via distribution $\mathbb{P}(\tau\!=\!k)\!=\!\frac{\alpha_k}{\sum_{t=1}^T\!\alpha_t}$. 
This does not apply to coordinate-wise adaptive stochastic methods because the learning rate for each coordinate is different, and hence unable to randomly select an index according to some distribution uniform for each coordinate. 
On the other hand, the convergence rates of AMSGrad and AdaEMA \citep{chen2018convergence} are able to achieve the bound for $\mathbb{E}[\norm{\nabla\! f(\bm{x}_\tau)}^2]$. 
This is due to the strong assumption $\norm{g_t} \leq G$ which results in a uniform lower bound for each coordinate of the adaptive learning rate $\eta_{t,k}\!\geq\!\alpha_t/G$. Thus, the proof of AMSGrad \citep{chen2018convergence} can be dealt with in a way similar to the case of global learning rate. 
In our paper we use a coordinate-wise adaptive learning rate and assume a weaker assumption $\mathbb{E}[\norm{g_t}^2]\!\leq\!G$ instead of $\norm{g_t}^2\!\leq\!G$. 
\end{remark}

\subsection{Proof Sketch of Theorem \ref{convergence_in_expectation}}
In this section, we will give some important lemmas that will be used to prove Theorem \ref{convergence_in_expectation}. First, for simplicity, we give some notations used in the lemmas and proof. 
Denote $\hat{\bm{v}}_t = \theta_t \bm{v}_{t-1} + (1-\theta_t)\mathbb{E}\left[\bm{g}_t^2|\bm{x}_{t}\right]$, $\hat{\bm{\eta}}_t = \alpha_t/\sqrt{\hat{\bm{v}}_t}$, $\bm{\Delta}_t = \alpha_t\bm{m}_t/\sqrt{\bm{v}_t}$, and $\norm{\bm{x}}^2_{\bm{\hat{\eta}}} = \sum_{i=1}^d \bm{\hat{\eta}}_i\bm{x}_i^2$.

\begin{lemma}\label{lemma1}
Let $\{\bm{x}_t\}$ be the sequence generated by Algorithm \ref{Adam}. For $T \geq 1$,  it holds that
\[
f^* \leq f(\bm{x}_1) + \sum_{t=1}^T \mathbb{E}\left[\left\langle \bm\nabla f(\bm{x}_t), \bm{\Delta}_t \right\rangle + L \norm{\bm{\Delta}_t}^2 \right].
\]
\end{lemma}

Lemma \ref{lemma1} is widely used to prove convergence of the stochastic gradient algorithms. To further prove the convergence related to gradient norm, we introduce the following lemma to bound term $\sum_{t=1}^T \mathbb{E}\left[\left\langle \bm\nabla f(\bm{x}_t), \bm{\Delta}_t \right\rangle + L \norm{\bm{\Delta}_t}^2 \right]$.  
\begin{lemma}\label{lemma2}(Lemma \ref{lem1-006} in Appendix)
Let $\{\bm{x}_t\}$ be the sequence generated by Algorithm \ref{Adam}. For $T \geq 1$,  it holds that
\[
\begin{split}
\sum_{t=1}^T \mathbb{E}\left[\left\langle \bm\nabla f(\bm{x}_t), \bm{\Delta}_t \right\rangle + L \norm{\bm{\Delta}_t}^2 \right]
\leq -\frac{1-\beta}{2} \mathbb{E}\left[\sum_{t=1}^T \norm{\bm \nabla f(\bm{x}_t)}_{\bm{\hat{\eta}}_t}^2\right] + \frac{\zeta_0}{\theta_1} \sum_{t=1}^T \alpha_t \sqrt{1-\theta_t} + \zeta_1,
\end{split}
\]
for some constants $\zeta_0$ and $\zeta_1$.
\end{lemma}

With Lemma \ref{lemma2}, we obtain the upper bound for $\mathbb{E}\left[\sum_{t=1}^T \norm{\bm \nabla f(\bm{x}_t)}_{\bm{\hat{\eta}}_t}^2\right]$ with respect to  $\alpha_t,\ \theta_t,$ and $\beta$. \textcolor{black}{To prove Lemma \ref{lemma2}, the most important step is to remove dependence between the adaptive learning rate and stochastic gradient. Hence, $\hat{\eta}_t$ is introduced, as it is independent of the stochastic gradient. The rest of the proof is just bounding the error of introducing $\hat{\bm{\eta}}_t$ instead of using $\alpha_t/\sqrt{\bm{v}_t}$.} Then, the last step is to build connection between $\frac{1}{T}\mathbb{E}[\norm{\bm \nabla f(\bm{x}_t)}]$ and $\mathbb{E}\left[\sum_{t=1}^T \norm{\bm \nabla f(\bm{x}_t)}_{\bm{\eta}_t}^2\right]$.

\begin{lemma}
\label{lemma3}(Lemma \ref{lem1-008} in Appendix)
Let $\tau$ be an integer that is randomly chosen from $\{1, 2, \cdots, T\}$ with equal probabilities. We have the following estimate
\[
\mathbb{E}[\norm{\bm \nabla f(\bm{x}_\tau)}] \leq \sqrt{\frac{\zeta_2}{T\alpha_T} \mathbb{E}\left[\sum_{t=1}^T \norm{\bm \nabla f(\bm{x}_t)}^2_{\hat{\bm \eta}_t}\right]},
\]
for some constant $\zeta_2$.
\end{lemma}

Thus, using the above three lemmas, we can prove Theorem \ref{convergence_in_expectation}.

\textcolor{black}{
\subsection{Discussion of Theorem \ref{convergence_in_expectation}} }

\begin{corollary}\label{constant-theta}
Take $\alpha_t = \eta/t^s$ with $0\leq s < 1$. Suppose $\lim_{t\to\infty} \theta_t = \theta < 1$. Define $Bound(T) := \frac{C + C'\sum_{t=1}^T\alpha_t\sqrt{1-\theta_t}}{T\alpha_T}$. Then the $Bound(T)$ in Theorem \ref{convergence_in_expectation} is bounded from below by constants
\begin{equation}
Bound(T) \geq C'\sqrt{1-\theta}.
\end{equation}
In particular, when $\theta_t = \theta < 1$, we have the following more subtle estimate of lower and upper-bounds for $Bound(T)$
\begin{equation*}
\frac{C}{\eta T^{1-s}} + C'\sqrt{1-\theta}\leq Bound(T) \!\leq\! \frac{C}{\eta T^{1-s}} \!+\! \frac{C'\sqrt{1\!-\!\theta}}{1-s}.
\end{equation*}
\end{corollary}

\begin{remark}
{\bf(i)}\ Corollary \ref{constant-theta} shows that if $\lim_{t\to\infty}\theta_t = \theta < 1$, the bound in Theorem \ref{convergence_in_expectation} is only $\mathcal{O}(1)$, hence not guaranteeing convergence. 
This result is not surprising as Adam with constant $\theta_t$ has already shown to be divergent \citep{Reddi2018on}. Hence, $\mathcal{O}(1)$ is its best convergence rate we can expect. 
We will discuss this case in more details in Section \ref{insights-for-divergence}.\\
{\bf(ii)}\ Corollary \ref{constant-theta} also indicates that in order to guarantee convergence, the parameter has to satisfy $\lim_{t\to\infty} \theta_t = 1$. 
Although we do not assume this in our restrictions (R1)-(R3), it turns out to be the consequence from our analysis. Note that if $\beta < 1$ in (R1) and $\lim_{t\to\infty}\theta_t =1$, then the restriction $\lim_{t\to\infty}\theta_t > \beta^2$ is automatically satisfied in (R2).
\end{remark}

We are now ready to give the Sufficient Condition for convergence of Generic Adam.
\begin{corollary}[Sufficient Condition(\textbf{SC})]\label{convergence}
Generic Adam is convergent if the parameters $\{\alpha_t\}$, $\{\beta_t\}$, and $\{\theta_t\}$ satisfy the following four conditions:
\begin{itemize}\setlength{\itemsep}{0pt}
\item[1.] $0 \leq \beta_t \leq \beta < 1$;
\item[2.] $0 < \theta_t < 1$ and $\theta_t$ is non-decreasing in $t$;
\item[3.] $\chi_t := \alpha_t/\sqrt{1-\theta_t}$ is ``almost" non-increasing;
\item[4.] $\big({\sum_{t=1}^T\alpha_t\sqrt{1-\theta_t}}\big){\big /}\big({T\alpha_T}\big) = o(1)$.
\end{itemize}
\end{corollary}

\medskip
\subsection{Convergence Rate of Generic Adam}

We now provide the convergence rate of Generic Adam with specific parameters $\{(\theta_{t},\alpha_{t})\}$, \textit{i.e.}, 
\begin{align}\label{convergence-rate-parameter}
\alpha_t = \eta/t^s \text{~~and~~}  \theta_t = \left\{
\begin{aligned}
& 1 - \alpha/K^r, \quad & t < K,\\
& 1 - \alpha/t^r, \quad & t \geq K,
\end{aligned}
\right.
\end{align}
for positive constants $\alpha,\eta,K$, where $K$ is taken such that $\alpha/K^r < 1$. Note that $\alpha$ can be taken bigger than 1. When $\alpha < 1$, we can take $K = 1$ and then $\theta_t = 1 - \alpha/t^r, t\geq 1$. To guarantee (R3), we require $r \leq 2s$. For such a family of parameters we have the following corollary.

\begin{corollary}\label{poly-setting}
Generic Adam with the above family of parameters (i.e. \eqref{convergence-rate-parameter}) converges as long as $0 < r \leq 2s < 2$, and its non-asymptotic convergence rate is given by 
\begin{equation*}
\mathbb{E}[\norm{\bm{\nabla} f(\bm{x}_\tau)}] \leq \left\{\begin{aligned}
& \mathcal{O}(T^{-r/4}), \quad &   r/2 + s < 1 \\
& \mathcal{O}(\sqrt{\log(T)/T^{1-s}}), \quad &  r/2 + s = 1 \\
& \mathcal{O}(1/T^{(1 - s)/2}), \quad &  r/2 + s > 1
\end{aligned}\right..
\end{equation*}
\end{corollary}

\begin{remark}
Corollary \ref{poly-setting} recovers and extends the results of some well-known algorithms below: 
\begin{itemize}
\item{\textbf{AdaGrad with exponential moving average (EMA)}.}
When $\theta_t = 1 - 1/t$, $\alpha_t = \eta/\sqrt{t}$, and $\beta_t = \beta < 1$, Generic Adam is exactly AdaGrad with EMA momentum (AdaEMA) \citep{chen2018convergence}. 
In particular, if $\beta = 0$, this is the vanilla coordinate-wise AdaGrad. 
It corresponds to taking $r = 1$ and $s = 1/2$ in Corollary \ref{poly-setting}. Hence, AdaEMA has convergence rate $\mathcal{O}\left(\sqrt{\log(T)/\sqrt{T}}\right)$.

\item{\textbf{AdamNC}.}
Taking $\theta_t = 1-1/t$, $\alpha_t = \eta/\sqrt{t}$, and $\beta_t = \beta\lambda^t$ in Generic Adam, where $\lambda < 1$ is the decay factor for the momentum factors $\beta_t$, we recover AdamNC \citep{Reddi2018on}.
Its $\mathcal{O}\left(\sqrt{\log{(T)}/\sqrt{T}}\right)$ convergence rate can be directly derived via Corollary \ref{poly-setting}.

\item{\textbf{RMSProp}.}
\citet{mukkamala2017variants} have reached the same $\mathcal{O}\left(\sqrt{\log{(T)}/\sqrt{T}}\right)$ convergence rate for RMSprop with $\theta_t = 1 - \alpha/t$, when $0< \alpha \leq 1$ and $\alpha_t = \eta/\sqrt{t}$ under the convex assumption. Since RMSProp is essentially Generic Adam with all momentum factors $\beta_t = 0$, we recover Mukkamala and Hein's results by taking $r = 1$ and $s = 1/2$ in Corollary \ref{poly-setting}. Moreover, our result generalizes to the non-convex stochastic setting, and it holds for all $\alpha\!>\!0$ rather than only $0\! <\! \alpha\! \leq\! 1$.
\end{itemize}
\end{remark}

The summarization of the above algorithms is provided in Table \ref{algorithms_table}.

\begin{table}[htpb]
    \centering
    \begin{tabular}{c|c|c|c|c}
         Algorithm & $\theta_t$ & $\alpha_t$ & $\beta_t$ & Convergence Result\\
         \hline
         Adagrad with EMA (\cite{chen2018convergence})& $1-1/t$& $\eta/\sqrt{t}$& $\beta$ & $\mathcal{O}(\sqrt{\log T/\sqrt{T}} )$  \\
         AdamNC (\cite{Reddi2018on})& $1-1/t$& $\eta/\sqrt{t}$& $\beta\lambda^t$ & $\mathcal{O}(\sqrt{\log T/\sqrt{T}} )$\\
         RMSProp \citep{mukkamala2017variants} &$1-\alpha/t$& $\eta/\sqrt{t}$& 0 & $\mathcal{O}(\sqrt{\log T/\sqrt{T}} )$\\
    \end{tabular}
    \caption{\textcolor{black}{Convergence results for some well-known algorithms}}
    \label{algorithms_table}
\end{table}

\noindent
{\bf Comparison \textcolor{black}{with \cite{Reddi2018on} }.}\ ~Most of the convergent modifications of original Adam, such as AMSGrad, AdamNC, and NosAdam, all require $\Gamma_t\succ 0$ in Eq.~\eqref{Gamma_t}, which is equivalent to decreasing the adaptive learning rate $\eta_t$ step by step. Since the term $\Gamma_t$ (or adaptive learning rate $\eta_t$) involves the past stochastic gradients (hence not deterministic), the modification to guarantee $\Gamma_t\succ 0$ either needs to change the iteration scheme of Adam (like AMSGrad) or needs to impose some strong restrictions on the base learning rates $\alpha_t$ and $\theta_t$ (like AdamNC). Our sufficient condition provides an easy-to-check criterion for the convergence of Generic Adam in Corollary \ref{convergence}. It is not necessary to require $\Gamma_t\succ 0$. Moreover, we use exactly the same iteration scheme as the original Adam without any modifications. Our work shows that the positive definiteness of $\Gamma_t$ may not be an essential issue for the divergence of the original Adam. \textcolor{black}{The divergence may be due to the incorrect setting of moving average parameters instead of non-positive definiteness of $\Gamma_t$.  }

\subsection{Constant $\theta_t$ case: insights for divergence}\label{insights-for-divergence}

The currently most popular RMSProp and Adam's parameter setting takes constant  $\theta_t$, \textit{i.e.}, $\theta_t = \theta < 1$. The motivation behind is to use the exponential moving average of squares of past stochastic gradients. In practice, parameter $\theta$ is recommended to be set very close to 1. For instance, a commonly adopted $\theta$ is taken as 0.999. 

Although great performance in practice has been observed, such a constant parameter setting has the serious flaw that there is no convergence guarantee even for convex optimization, as proved by the counterexamples in \citep{Reddi2018on}.
Ever since much work has been done to analyze the divergence issue of Adam and to propose modifications with convergence guarantees, as summarized in the introduction section. 
However, there is still not a satisfactory explanation that touches the fundamental reason for the divergence. In this section, we try to provide more insights for the divergence issue of Adam/RMSProp with constant parameter $\theta_t$, based on our analysis of the sufficient condition for convergence. 

\noindent
{\bf From the sufficient condition perspective.}\ ~ 
Let $\alpha_t \!=\! \eta/t^s$ for $0\leq s \!<\! 1$ and $\theta_t \!=\! \theta \!<\! 1$. According to Corollary \ref{constant-theta}, $Bound(T)$ in Theorem \ref{convergence_in_expectation} has the following estimate:
\begin{align*}
\frac{C}{\eta T^{1\!-s}}\!+\! C'\sqrt{1\!-\!\theta}\!\leq\! Bound(T)\! \leq\! \frac{C}{ \eta T^{1-s}} \!+\! \frac{C'\sqrt{1\!-\!\theta}}{(1\!-s)}.
\end{align*}
The bounds tell us some points on Adam with constant $\theta_t$:
\begin{itemize}
    \item[1.] $Bound(T)\!=\!\mathcal{O}(1)$, so the convergence is not guaranteed. This result coincides with the divergence issue demonstrated in \citep{Reddi2018on}. Indeed, since in this case Adam is not convergent, this is the best bound we can have. 
    
    \item[2.] Consider the dependence on parameter $s$. The bound is decreasing in $s$. The best bound in this case is when $s=0$, \textit{i.e.}, the base learning rate is taken constant. This explains why in practice taking a more aggressive constant base learning rate often leads to even better performance, comparing with taking a decaying one.
    
    \item[3.] Consider the dependence on parameter $\theta$. Note that the constants $C$ and $C'$ depend on $\theta_1$ instead of the whole sequence $\theta_t$. We can always set $\theta_t = \theta$ for $t \geq 2$ while fix $\theta_1 < \theta$, by which we can take $C$ and $C'$ independent of constant $\theta$. Then the principal term of $Bound(T)$ is linear in $\sqrt{1-\theta}$, so decreases to zero as $\theta \to 1$. This explains why setting $\theta$ close to 1 often results in better performance in practice.
\end{itemize}

Moreover, Corollary \ref{poly-setting} shows us how the convergence rate continuously changes when we continuously vary parameters $\theta_t$. Let us fix $\alpha_t \!=\! 1/\sqrt{t}$ and consider the following continuous family of parameters $\{\theta_t^{(r)}\}$ with $r\in [0, 1]$: 
\begin{equation*}
\theta_t^{(r)} = 1 - \alpha^{(r)}/t^r, \text{~~where~~} \alpha^{(r)} = r\bar{\theta} + (1-\bar{\theta}),\ 0 < \bar{\theta} < 1.
\end{equation*}
Note that when $r=1$, then $\theta_t = 1 - 1/t$, this is the AdaEMA, which has the convergence rate $\mathcal{O}\left(\sqrt{\log T/\sqrt{T}}\right)$; when $r = 0$, then $\theta_t = \bar{\theta} < 1$, this is the original Adam with constant $\theta_t$, which only has the $\mathcal{O}(1)$ bound; when $0 < r < 1$, by Corollary \ref{poly-setting}, the algorithm has the $\mathcal{O}(T^{-r/4})$ convergence rate. 
Along with this continuous family of parameters, we observe that the theoretical convergence rate continuously deteriorates as the real parameter $r$ decreases from 1 to 0, namely, as we gradually shift from AdaEMA to Adam with constant $\theta_t$. 
In the limiting case, the latter is not guaranteed with convergence anymore. 


\section{Complexity Analysis for Practical Adam: Mini-batch/Distributed Adam}
\label{sec_pratical}
\textcolor{black}{Due to the limited time, limited computational resources, and noise in data collection and processing, it is almost impossible to achieve the accurate stationary point. Thus, instead of achieving the accurate stationary point, people get more attention to achieving some approximated stationary point in practice. The crucial question under this situation will become how much time is needed to achieve some approximated solution. This section will answer this question by answering how many iterations are needed to obtain an $\varepsilon$-stationary point. First, we define $\varepsilon$-stationary point as follows:}
\begin{definition}
\label{def_eps}
We define a random variable $\bm{x}$ as an $\varepsilon$-stationary point of problem \eqref{minimization}, if 
\[
\mathbb{E}_{\bm{x}}\left[\left\|\mathbb{E}_{\xi \sim \mathbb{P}} \left[\nabla \tilde{f}(\bm{x},\xi)\right]\right\|\right] \leq \varepsilon.
\]
\end{definition}


\textcolor{black}{According to Definition \ref{def_eps} and Theorem \ref{convergence_in_expectation}, for generic Adam, we can directly give the following corollary:}

\begin{corollary}
\label{practical_adam}
For any $\varepsilon>0$, if we take $T \geq C_5^2 \varepsilon^{-4}$, $\alpha_t = \frac{\alpha}{\sqrt{T}},\ \beta_t = \beta, \theta_t = 1-\frac{\theta}{T}$, which satisfy $\gamma = \frac{\beta}{1-\frac{\theta}{T}} < 1$ and $\theta_t \geq \frac{1}{4}$, then by taking $\tau$ uniformly from $\{1,2,\cdots, T\}$, it holds that
\begin{equation*}
\mathbb{E}\left[\norm{\bm{\nabla} f(\bm{x}_\tau)}, \right]
\leq \varepsilon 
\end{equation*}
where
\begin{equation*}
\begin{split}
C_5&=\frac{2\sqrt{G^2\!+\!\epsilon d}}{\alpha(1-\beta)} \left[f(x_1) - f^*+C_6d\frac{\alpha}{\sqrt{\theta}}\log\big(1\!+\! \frac{G^2}{\epsilon d}\big) + \frac{4C_6d\alpha}{\sqrt{\theta}}\right],\\
C_6& = \frac{1}{1-\sqrt{\gamma}}\left[\frac{\alpha L}{\sqrt{\theta}(1-\sqrt{\gamma})^2} + 2\big(\frac{2\beta/(1-\beta)}{\sqrt{C_1(1-\gamma)}}+1\big)^2G\right].
\end{split}
\end{equation*}
\end{corollary}

\begin{remark}
\textcolor{black}{Difference from Corollary \ref{constant-theta}. Corollary \ref{constant-theta} shows when $lim_{t\rightarrow \infty} \theta_t <1$ the algorithm cannot converge to a stationary point. However, because the goal of the algorithm switches to an $\varepsilon$-stationary point, the choice of $\alpha_t$ and $\theta_t$ is not contradicting Corollary \ref{constant-theta}. We will use the same choice in the following section.}
\end{remark}

\begin{remark}
\textcolor{black}{With the parameter setting in Corollary \ref{poly-setting}, the number of iterations T should satisfy $\frac{T}{\log^2T}  = \Omega(\varepsilon^{-4})$, instead of $T= \Omega(\varepsilon^{-4})$, which gives a larger iteration number T. However, we can use the same parameter setting in Corollary \ref{poly-setting} when T changes, while we need to change parameters in Corollary \ref{practical_adam} for different T.}
\end{remark}

\begin{remark}
From Theorem \ref{practical_adam}, to achieve an $\varepsilon$-stationary point, only $\Omega(\varepsilon^{-4})$ iterations are needed. Comparing the result with SGD \citep{li2014efficient}, we have the same order of iterations to achieve an $\varepsilon$-stationary point. 
\end{remark}

\textcolor{black}{In the following two sections, we will analyze two practical Adam variations, i.e., mini-batch and distributed Adam. Although they can use the same technique for analysis, we list two algorithms for readers in different communities (single machine learning algorithm (mini-batch Adam) v.s. multi-machine learning algorithm (distributed Adam)).}

\subsection{Convergence Analysis for Mini-Batch Adam}
\label{minibatchadam-convergence}
It has been shown that when $s$ samples are used to estimate the gradient in the stochastic gradient descent algorithm, the convergence speed can be accelerated $s$ times than the single sample algorithms \cite{li2014efficient}. Meanwhile, in practice, the mini-batch technique is widely used to optimize problem \eqref{minimization} such as training a neural network with the Adam algorithm. In this section, we will give the analysis on mini-batch Adam. The Mini-batch Adam algorithm is defined in the following Algorithm \ref{minibatchadam}. Different from Algorithm \ref{Adam}, in Algorithm \ref{minibatchadam} $s$ samples which are identically distributed and independent when the iterate $\bm{x}_t$ is used to estimate the gradient $\bm \nabla f(\bm{x}_t)$. We average the $s$ estimates and use the averaged stochastic gradient, which should be a more accurate estimation to update $\bm{x}_t$.

\begin{algorithm}[H]
\caption{\ Mini-batch Adam}
\label{minibatchadam}
\begin{algorithmic}[1]
   \STATE {\bf Parameters:} Choose $\{\alpha_t\}$, $\{\beta_t\}$, and $\{\theta_t\}$. Choose $\bm{x}_1 \in \mathbb{R}^d$ and set initial values $\bm{m}_0\!=\!\bm{0}$ and $\bm{v}_0=\bm{\epsilon}$.
   \FOR{$t= 1,2,\ldots,T$}
    \STATE Sample $s$ stochastic gradients: $\bm{g}^{(1)}_t$, $\bm{g}^{(2)}_t$, $\cdots$, $\bm{g}^{(s)}_t$;
    \STATE Average $s$ stochastic gradients: $\bar{\bm{g}}_t = \frac{1}{s} \sum_{i=1}^s \bm{g}^{(i)}_t$;
        \FOR {$k=1,2,\ldots,d$}
        \STATE $\bm{v}_{t,k} = \theta_t \bm{v}_{t-1,k} + (1 - \theta_t) \bar{\bm{g}}_{t,k}^2$;
        \STATE $\bm{m}_{t,k} = \beta_t \bm{m}_{t-1,k} + (1 - \beta_t) \bar{\bm{g}}_{t,k}$;
        \STATE $\bm{x}_{t+1,k} = \bm{x}_{t,k} - {\alpha_t \bm{m}_{t,k}}/\sqrt{\bm{v}_{t,k}}$;
        \ENDFOR
   \ENDFOR
 \end{algorithmic}
\end{algorithm}

To link sample size and convergence rate, we give a new assumption on the stochastic gradient and state it as follows:
\begin{description}
\item[(A5)] \textcolor{black}{The variance of stochastic gradient $\bm{g}_{t}$ is uniformly upper-bounded, \textit{i.e.}, $\mathbb{E}\,[\|\bm{g}_t - \nabla f_t(\bm{x}_t)\|^2]\leq \sigma^2$.}
\end{description}

\textcolor{black}{We add {\bf (A5)} to establish the relation between sample size and the convergence rate. Intuitively, with an increasing size of samples, the variance of the gradient estimator should reduce. Utilizing this reduction, we can obtain an $\varepsilon$-stationary point, with fewer iterations but a larger sample size. Also, this assumption is widely used in analysis such as \citet{yan2018unified}. The following results are given under assumptions {\bf (A1)} to {\bf (A5)}, and the result of mini-batch Adam is given as follows:}

\begin{theorem}
\label{minibatchtheorem}
For any $\varepsilon>0$, if we take $\alpha_t = \frac{\alpha}{\sqrt{T}}$, $\beta_t = \beta$ and $\theta_t = 1 - \frac{\theta}{T}$, which satisfy $\gamma = \frac{\beta_t^2}{\theta_t} < 1$, $\theta_t \geq \frac{1}{4}$ and $ F_T(T,s) \leq \varepsilon$, then there exists $t \in \{1,2,\cdots,T\}$ such that 
\[
\mathbb{E} \left[\|\nabla f(\bm{x}_t)\|\right] \leq \varepsilon,
\]
where by taking $\epsilon = \frac{1}{sd}$, it holds that
\[
\begin{split}
&F_T(T,s)=  \mathcal{O}(T^{-1}s^{1/2}d^{1/2} + T^{-1/2}d + T^{-1/3}s^{-1/4} + T^{-1/4}s^{-1/4}d^{1/2}).
\end{split}
\]

Thus, to achieving an $\varepsilon$-stationary point, $\Omega(\varepsilon^{-4}s^{-1})$ iterations are needed.

More specifically, it holds that
\[
\begin{aligned}
F_T(\varepsilon,s) &= \frac{1}{\sqrt{T}} \left( (4C_{8}C_{10})^{1/2} + (4C_{8}C_{11})^{2/3} + (4C_{9}C_{10}) + (4C_{9}C_{11})^2) \right),\\
C_{8} &= \frac{\sqrt{T}\sqrt{2\sigma^2\theta+\epsilon s d}}{\sqrt{s}\alpha} ,\ 
C_{9} = \sqrt{\frac{2\theta}{\alpha^2}},\ 
C_{11} = \frac{2C_4}{1-\beta}\sqrt[4]{\frac{2\theta}{d\epsilon T}},\\
C_{10} &= \frac{2}{1-\beta}\left(f(x_1) - f^* + C_7 d\theta + 2 C_7 d \sqrt[4]{1 + \frac{2\sigma^2}{d\epsilon s}}\right),\\
C_{7} &= \frac{1}{1-\sqrt{\gamma}}\left(\frac{\alpha^2L}{\theta\left(1-\sqrt{\gamma}\right)^2} + \frac{2\left(\frac{2\beta/\left(1-\beta\right)}{\sqrt{1-\gamma }}+1\right)^2G\alpha}{\sqrt{\theta}} \right). \\
\end{aligned}
\]

\end{theorem}

\begin{remark}
\label{Remark_minibatch}
Below, we give three comments on the above results:\\
\textcolor{black}{{\bf (i)}\ From Theorem \ref{minibatchtheorem}, to achieve an $\varepsilon$-stationary point, when we only consider the order with respect to $\varepsilon$, $\Omega(\varepsilon^{-4})$ iterations are needed. Besides, by jointly considering $\varepsilon$ and batch size $s$, we can accelerate the algorithm to achieve an $\varepsilon$-stationary point, where $\Omega(\varepsilon^{-4}s^{-1})$ iterations are needed, which indicates that Mini-batch Adam can be linearly accelerated with respect to the mini-batch size. The result is in the same order of mini-batch SGD in \citep{li2014efficient}.}\\
{\bf (ii)}\ Deriving the linear speedup property of mini-batch Adam with respect to mini-batch size is much more difficult than the analysis techniques for mini-batch SGD \citep{li2014efficient} since the adaptive learning rate in Algorithm \ref{minibatchadam} is highly coupled with mini-batch stochastic gradient estimates. In fact, the adaptive learning rate implicitly adjusts the magnitude of the learning rate with respect to mini-batch size, while the hand-crafted learning rate in mini-batch SGD has to be tuned carefully via a linear LR scaling technique \citep{krizhevsky2014one,you2017large} for a large mini-batch training.  \\   
\textcolor{black}{{\bf (iii)}\ The dimension dependence of the above analysis is $O(\sqrt{d})$. Meanwhile, some analyses on (variants of) Adam (\cite{chen2018convergence, defossez2020simple, zou2018convergence}) achieve the same dimension dependence, while \cite{zhou2018convergence} showed that in AMSGrad, RMSProp and Adagrad the dependence can be $O(d^{-1/4})$.}
\end{remark}

\subsection{Proof Sketch of Theorem \ref{minibatchtheorem}}\label{proof_sketch}

Similar to the proof of Theorem \ref{convergence_in_expectation}, we first try to bound $\sum_{t=1}^T \mathbb{E}\left[\left\langle \bm\nabla f(\bm{x}_t), \bm{\Delta}_t \right\rangle + L \norm{\bm{\Delta}_t}^2 \right]$. Because we introduce a new assumption that stochastic gradient has bounded variance $\sigma^2$, it is easy to verify that the averaged stochastic gradient $\bar{\bm{g}}_t$ with variance $\sigma^2/s$, which is a key property to establish the relation between sample size and convergence rate. To take advantage of variance, together with the constant value assigned to $\alpha_t$ and $\theta_t$, we derive the mini-batch/distributed variants of Lemmas \ref{lemma2} and \ref{lemma3}.

\begin{lemma}
\label{lemma4} (Lemma \ref{multitau} in Appendix)
Let $\{\bm{x}_t\}$ be the sequence generated by Algorithm \ref{minibatchadam} or \ref{distributedadam}. For $T \geq 1$, when $\alpha_t = \frac{\alpha}{\sqrt{T}}$ and $\theta_t = 1- \frac{\theta}{T}$, we have
\[
\begin{split}
   \sum_{t=1}^T \mathbb{E}\left[\left\langle \bm\nabla f(\bm{x}_t), \bm{\Delta}_t \right\rangle + L \norm{\bm{\Delta}_t}^2 \right] 
   \leq  & \zeta_3 + \zeta_4 \sqrt[4]{\frac{\sigma^2}{s}} + \frac{\zeta_5}{\sqrt[4]{T}} \sqrt{\mathbb{E}\left[\sqrt{\sum_{t=1}^T{\norm{\bm \nabla f(\bm{x}_t)}}^2}\right]} \\ 
       &- \frac{1-\beta}{2} \mathbb{E}\left[\sum_{t=1}^T \norm{\bm \nabla f(\bm{x}_t)}_{\bm{\eta}_t}^2\right],
\end{split}
\]
for some constants $\zeta_3$, $\zeta_4$ and $\zeta_5$.
\end{lemma}

Different from Lemma \ref{lemma2} in which $\ell_2$ norm of $\bm \nabla f(\bm{x}_t)$ doesn't occur, in Lemma \ref{lemma4} when taking the benefit of small variance, both $\hat{\bm{\eta}_t}$ norm and $\ell_2$ norm of term $\bm \nabla f(\bm{x}_t)$ occur in the right hand side. \textcolor{black}{The key step for this lemma is in how we bound $\mathbb{E}(\|g_t\|^2)$, where in Lemma \ref{lemma2}, $\mathbb{E}\|g_t\|^2$ is directly bounded by $G^2$. However, as we want to explore the benefit of having mini-batch, we bound $\mathbb{E}(\|g_t\|^2)$ by $\mathbb{E}(\|\nabla f(x_t)\|^2)$ and $\sigma^2$. }Hence, \textcolor{black}{because the formulation becomes much more complicated,} some further calculation on $\bm \nabla f(\bm{x}_t)$ is needed and we will introduce the calculation in Lemma \ref{lemma6}.

\begin{lemma}
\label{lemma5} (Lemma \ref{4ineq} in Appendix)
Let $\{\bm{x}_t\}$ be the sequence generated by Algorithm \ref{minibatchadam} or \ref{distributedadam}. For $T \geq 1$,  it holds that
\[
\mathbb{E}\left[\sqrt{\sum_{t=1}^T \norm{\bm \nabla f(\bm{x}_t)}^2 }\right]^2 \leq \left( \sqrt{T}(\zeta_6\frac{\sigma}{\sqrt{s}} + \zeta_7\sqrt{\epsilon d}) + \zeta_8\mathbb{E}\left[\sqrt{\sum_{t=1}^T\norm{\bm\nabla f(\bm{x}_t)}^2}\right]\right) \mathbb{E}\left[\norm{\bm \nabla f(\bm{x}_t)}^2_{\hat{\bm \eta}_t}\right],
\]
for some constants $\zeta_6$, $\zeta_7$ and $\zeta_8$.
\end{lemma}

Different from Lemma \ref{lemma3}, the left hand side is not the summation of $\ell_2$ norm but square root of the summation of squared $\ell_2$ norm which is smaller than the summation of $\ell_2$. Therefore, Theorem \ref{minibatchtheorem} gives a weaker result than Theorem \ref{convergence_in_expectation}, which only shows the existence of $t\in \{1,2,\cdots,T\}$ such that $\mathbb{E}(\norm{\bm \nabla f(\bm{x}_t)}) = \mathcal{O}(1/\sqrt[4]{sT})$ rather than $\mathbb{E}(\norm{\bm \nabla f(\bm{x}_\tau)}) = \mathcal{O}(1/\sqrt[4]{T})$. After combining Lemma \ref{lemma1}, Lemma \ref{lemma4} and Lemma \ref{lemma5}, we can get an inequality as follows:
\begin{equation}
\label{fxt_inq}
\begin{split}
&\mathbb{E}\left[\sqrt{\sum_{t=1}^T \norm{\bm \nabla f(\bm{x}_t)}^2}\right]^2
\leq \left(\zeta_9 + \zeta_{10} \mathbb{E}\left[\sqrt{\sum_{t=1}^T \norm{\bm \nabla f(\bm{x}_t)}^2}\right]\right) \left(\zeta_{11} + \zeta_{12}\sqrt{\mathbb{E}\left[\sqrt{\sum_{t=1}^T \norm{\bm \nabla f(\bm{x}_t)}^2}\right]}\right).
\end{split}
\end{equation}

\begin{lemma}
\label{lemma6}
When inequality \eqref{fxt_inq} holds, it holds that
\[
\begin{split}
    &\mathbb{E}\left[\sqrt{\sum_{t=1}^T \norm{\bm \nabla f(\bm{x}_t)}^2}\right]\leq (4\zeta_{9}\zeta_{11})^{1/2} + (4\zeta_{9}\zeta_{12})^{2/3} + (4\zeta_{10}\zeta_{11}) + (4\zeta_{10}\zeta_{12})^2.
\end{split}
\]
\end{lemma}

Hence, combining Lemma \ref{lemma1}, Lemma \ref{lemma4}, Lemma \ref{lemma5} and Lemma \ref{lemma6}, we are able to prove $\mathbb{E}\left[\sqrt{\sum_{t=1}^T \norm{\bm \nabla f(\bm{x}_t)}^2}\right] = \mathcal{O}(T^{1/4}s^{-1/4})$. By dividing $\sqrt{T}$ on both sides we can prove Theorem \ref{minibatchtheorem}.

\subsection{Convergence Analysis for Distributed Adam}
\label{distributed_adam_convergence}

For large-scale problems such as training deep convolutional neural networks over the ImageNet dataset \citep{russakovsky2015imagenet}, it is hard to optimize problem \eqref{minimization} on a single machine. In this section, we extend the mini-batch Adam to the distributed Adam like the distributed SGD method \citep{yu2019linear}. The simplest structure is the parameter-server model in the distributed setting, where a parameter server and multiple workers are involved in the optimization process. As it is shown in Algorithm \ref{distributedadam}, in each iteration, a worker receives the iterate $\bm{x}_t$ from the server, samples a stochastic gradient with respect to $\bm{x}_t$, and sends the gradient to the server. Meanwhile, the parameter server receives gradients from workers in each iteration, averages the gradients, and performs an Adam update.

\begin{algorithm}[h!]
\caption{\ Distributed Adam}
\label{distributedadam}
\begin{algorithmic}[1]
   \STATE {\bf Parameters:} Choose $\{\alpha_t\}$, $\{\beta_t\}$, and $\{\theta_t\}$. Choose $\bm{x}_1 \in \mathbb{R}^d$ and set initial values $\bm{m}_0\!=\!\bm{0}$ and $\bm{v}_0=\bm{\epsilon}$.
   \STATE {\bf For Server:} Send $\bm{x}_1$ to workers;
   \FOR{$t= 1,2,\ldots,T$}
    \STATE {\bf For Worker $i$:}
    \STATE Receive iterate $\bm{x}_t$;
    \STATE Sample a stochastic gradient $\bm{g}^{(i)}_t$;
    \STATE Send $\bm{g}^{(i)}_t$ to the server;
    
    \STATE {\bf For the Server:}
    \STATE Receive stochastic gradient $\bm{g}^{(i)}_t$ from worker $i$;
    \STATE Average $s$ stochastic gradient $\bar{\bm{g}}_t = \frac{1}{s} \sum_{i=1}^s \bm{g}^{(i)}_t$;
        \FOR {$k=1,2,\ldots,d$}
        \STATE $\bm{v}_{t,k} = \theta_t \bm{v}_{t-1,k} + (1 - \theta_t) \bar{\bm{g}}_{t,k}^2$;
        \STATE $\bm{m}_{t,k} = \beta_t \bm{m}_{t-1,k} + (1 - \beta_t) \bar{\bm{g}}_{t,k}$;
        \STATE $\bm{x}_{t+1,k} = \bm{x}_{t,k} - {\alpha_t \bm{m}_{t,k}}/\sqrt{\bm{v}_{t,k}}$;
        \ENDFOR
     \STATE Send $\bm{x}_{t+1}$ to workers;
   \ENDFOR
 \end{algorithmic}
\end{algorithm}

\begin{proposition}
Algorithm \ref{distributedadam} with $s$ workers performs the same as Algorithm \ref{minibatchadam} with $s$ i.i.d. samples.
\end{proposition}


\begin{remark}
Below, we give two remarks on the above distributed Adam algorithm:\\
{\bf(i)}\ For distributed Adam, to achieve an $\varepsilon$-stationary point, $\Omega(\varepsilon^{-4}s^{-1})$ iterations are needed, which is a linear speedup with respect to the number of workers in the network, which is in the same order as that is in distributed SGD \citep{yu2019linear}.\\
{\bf(ii)}\ Distributed Adam has been popularly used for training deep neural networks. In addition, there also exist several variants of the distributed Adam algorithm, such as PMD-LAMB \citep{wang2020large}, LAMB \citep{you2019large}, LARS \citep{you2017large}, etc, for training large-scale deep neural networks. However, all these works do not establish the linear speedup property for distributed adaptive methods. 
\end{remark}

\section{Experimental Results}
\label{experimental_result}

In this section, we experimentally validate the proposed sufficient condition by applying Generic Adam and RMSProp to solve the counterexample \citep{chen2018convergence} and to train LeNet \citep{lecun1998gradient} on the MNIST dataset \citep{lecun2010mnist} and ResNet \citep{he2016deep} on the CIFAR-100 dataset \citep{krizhevsky2009learning}, respectively. Moreover, we use different batch sizes to train LeNet \citep{lecun1998gradient} on the MNIST dataset \citep{lecun2010mnist} and ResNet \citep{he2016deep} on the CIFAR-100 dataset \citep{krizhevsky2009learning}, respectively, and validate theory of the mini-batch Adam algorithm.

\subsection{Synthetic Counterexample}

\begin{figure*}[htpb]
\centering
\subfigure{\includegraphics[width=0.32\linewidth]{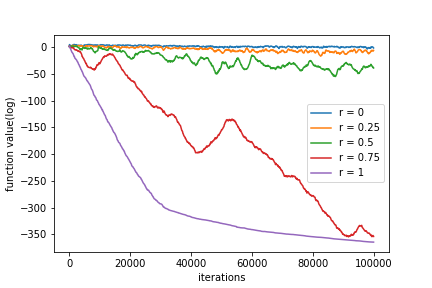}}\label{fig:synthesis_a}
\subfigure{\includegraphics[width=0.32\linewidth]{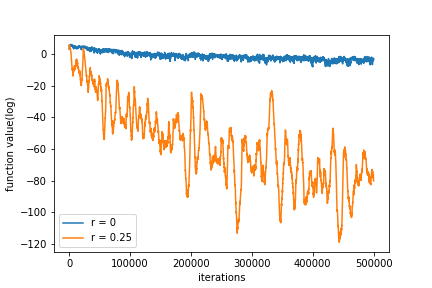}}\label{fig:synthesis_b}
\subfigure{\includegraphics[width=0.32\linewidth]{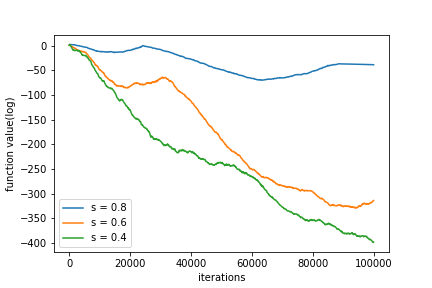}}\label{fig:synthesis_c}
\caption{The above figures are for function value with different $r$ and $s$ values, respectively.
The left figure plots the performance profiles of Generic Adam with different $r$ values, \textcolor{black}{where $\theta_t^{(r)} = 1- (0.01+0.99r^2)/t^r$}.
The middle figure plots the performance profiles of Generic Adam with $\theta_t^{(r)} = 1 - \frac{0.01}{t^r}$ and $r = 0$ and $0.25$. 
The right figure plots the performance profiles of Generic Adam with different $s$ values.
}
\label{fig:sensivitive}
\end{figure*}

In this experiment, we verify the phenomenon described in Section \ref{insights-for-divergence} how the convergence rate of Generic Adam gradually changes along a continuous path of families of parameters on the one-dimensional counterexample in \citet{chen2018convergence}:
\begin{equation}\label{counter-example}
f(x) = \mathbb{E}_\xi [\xi x^2],
\end{equation}
where $\mathbb{P}(\xi = 5.5) = 1/11$ and $\mathbb{P}(\xi=-0.5) = 10/11$.

\smallskip
\noindent
{\bf Sensitivity of parameter $r$.}\ ~
We set  $T = 10^5$, $\alpha_t = 5 / \sqrt{t}$,
$\beta = 0.9$, and $\theta_t$ as $\theta_t^{(r)} = 1 - (0.01 + 0.99r^2)/{t^r}$ with $r \in \{0,\ 0.25,\ 0.5,\ 0.75,\ 1.0\}$, respectively. 
Note that when $r=0$, Generic Adam reduces to the originally divergent Adam \citep{kingma2014adam} with $(\beta, \bar{\theta}) = (0.9, 0.99)$. When $r=1$, Generic Adam reduces to  AdaEMA \citep{chen2018convergence} with $\beta = 0.9$. 

The experimental results are shown in the left figure of Figure \ref{fig:sensivitive}. We can see that for $r=1.0, 0.75$ and $0.5$, Generic Adam is convergent. Moreover, the convergence becomes slower when $r$ decreases, which exactly matches Corollary \ref{poly-setting}. On the other hand, for $r = 0$ and $0.25$, Figure \ref{fig:sensivitive} shows that they do not converge. It seems that the divergence for $r = 0.25$ contradicts our theory. However, this is because when $r$ is very small, the $\mathcal{O}(T^{-r/4})$ convergence rate is so slow that we may not see a convergent trend in even $10^5$ iterations. Indeed, for $r = 0.25$, we actually have 
\[ \theta_t^{(0.25)} \leq 1 - (0.01 + 0.25* 0.25 * 0.99)/10^{5 * 0.25} \approx 0.9960, \]
which is not sufficiently close to 1. 
As a complementary experiment, we fix the numerator and only change $r$ when $r$ is small. We take $\alpha_t$ and $\beta_t$ as the same, while $\theta_t^{(r)} = 1 - \frac{0.01}{t^r}$ for $r = 0$ and $0.25$, respectively. 
The result is shown in the middle figure of Figures \ref{fig:sensivitive}. We can see that Generic Adam with $r= 0.25$ is indeed convergent in this situation.

\smallskip
\noindent
{\bf Sensitivity of parameter $s$.}\ ~
Now, we show the sensitivity of $s$ of the sufficient condition {(\bf SC)} by fixing $r\!=\!0.8$ and selecting $s$ from the collection $s= \{0.4, 0.6, 0.8\}$. 
The right figure in Figure \ref{fig:sensivitive} illustrates the sensitivity of parameter $s$ when Generic Adam is applied to solve the counterexample \eqref{counter-example}. The performance shows that when $s$ is fixed, smaller $r$ can lead to a faster and better convergence speed, which also coincides with the convergence results in Corollary \ref{poly-setting}. 

\subsection{LeNet on MNIST and ResNet-18 on CIFAR-100}
\label{ex52}

\begin{figure*}[htpb]
\centering
\subfigure{\includegraphics[width=0.32\linewidth]{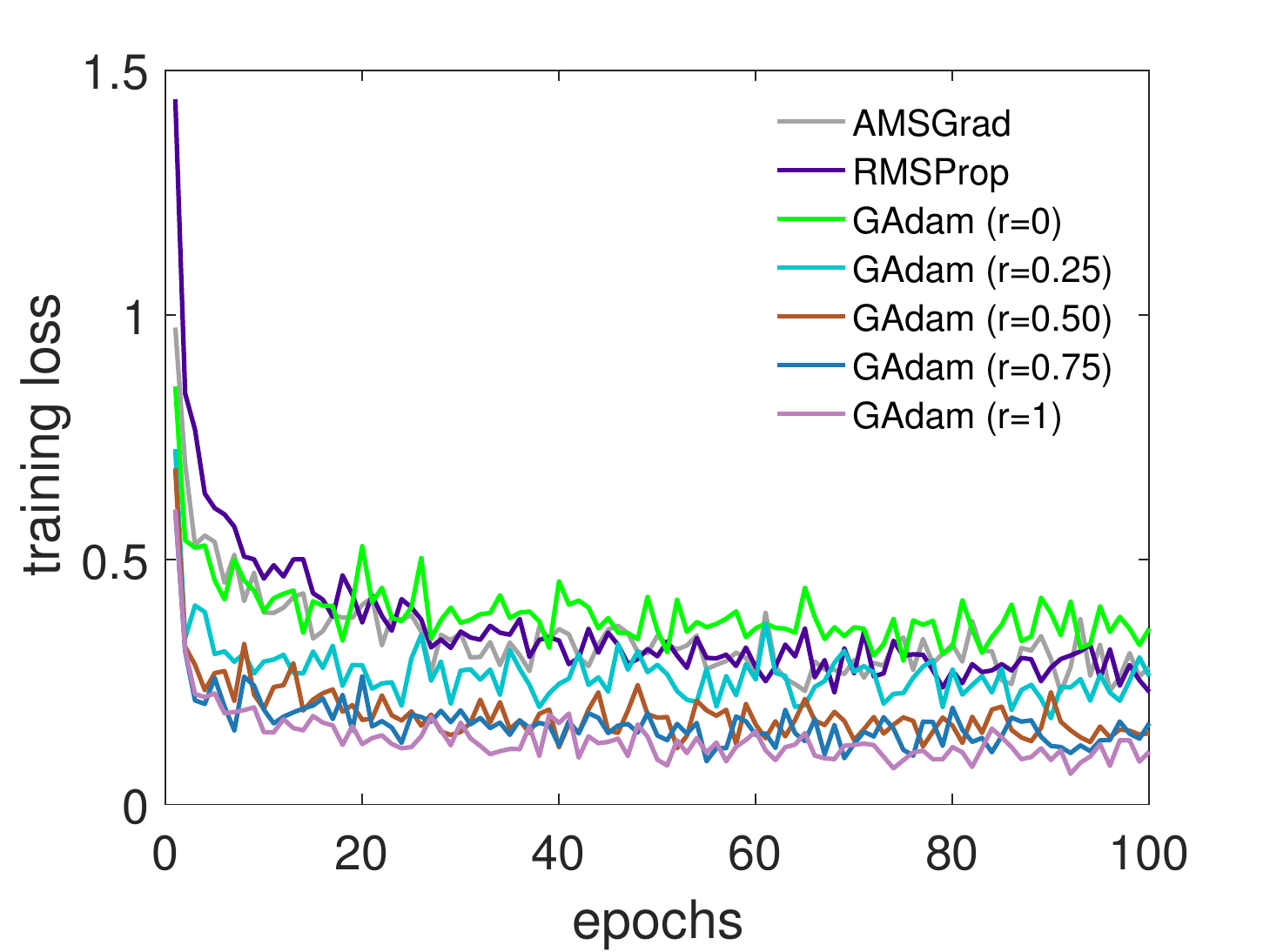}}\label{fig:Lenet_a}
\subfigure{\includegraphics[width=0.32\linewidth]{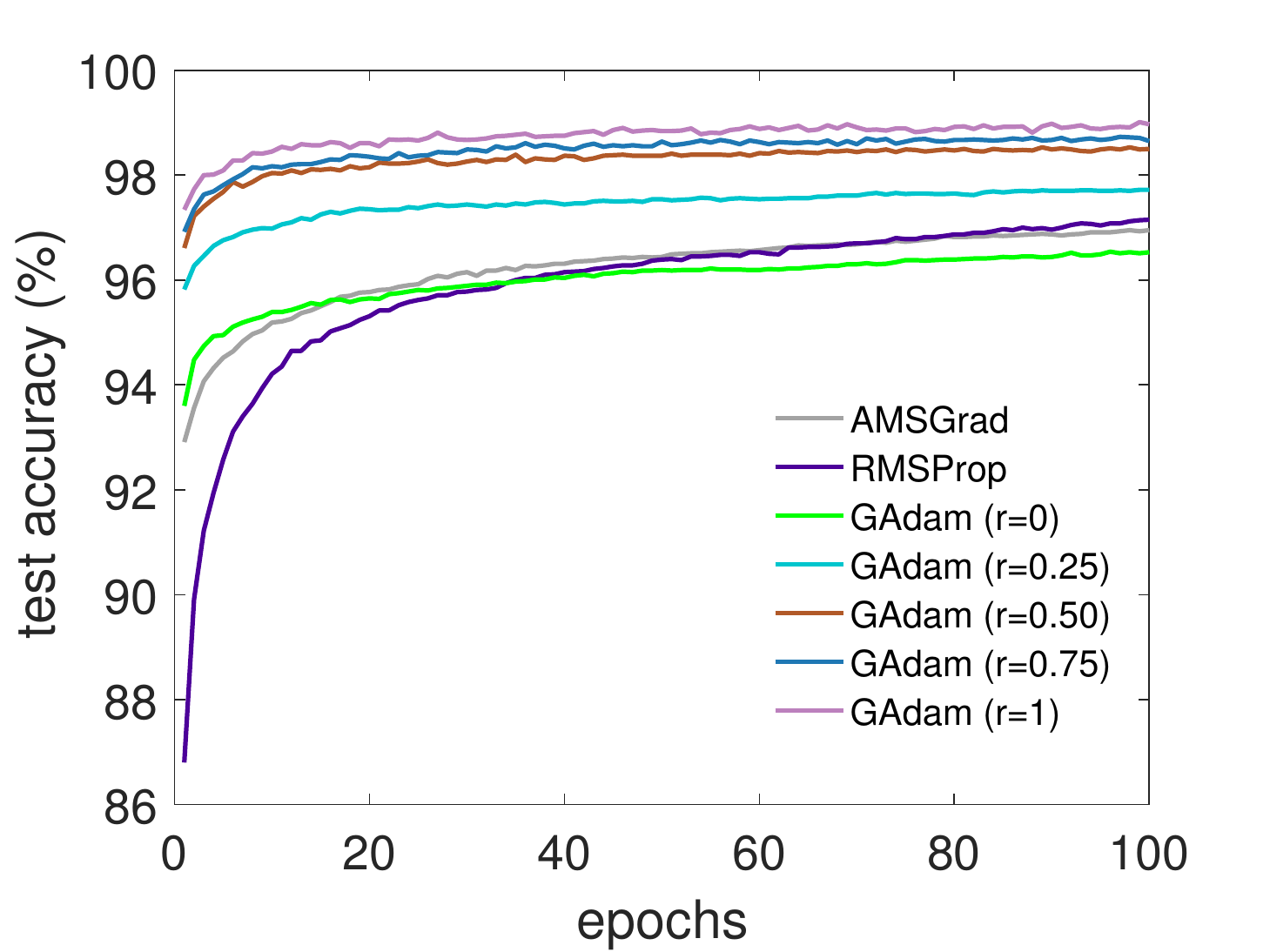}}\label{fig:Lenet_b}
\subfigure{\includegraphics[width=0.32\linewidth]{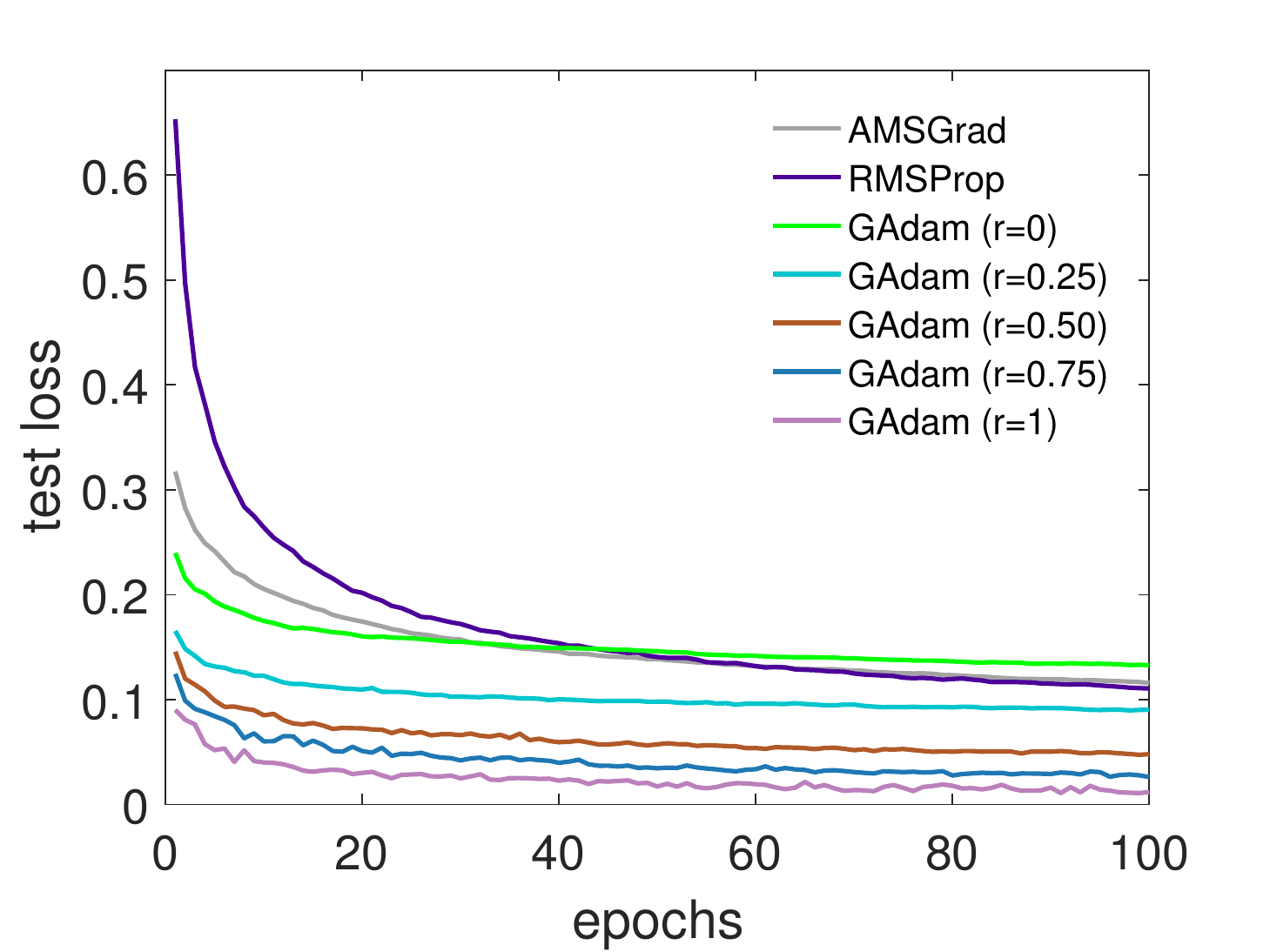}}\label{fig:Lenet_c}
\caption{Performance profiles of Generic Adam with $r=\{0, 0.25, 0.5, 0.75, 1\}$, RMSProp, and AMSGrad for training LeNet on the MNIST dataset. Figures (a), (b), and (c) illustrate training loss vs. epochs, test accuracy vs. epochs, and test loss vs. epochs, respectively.}
\label{fig:LeNet}
\vspace{-0.1cm}
\centering
\subfigure{\includegraphics[width=0.32\linewidth]{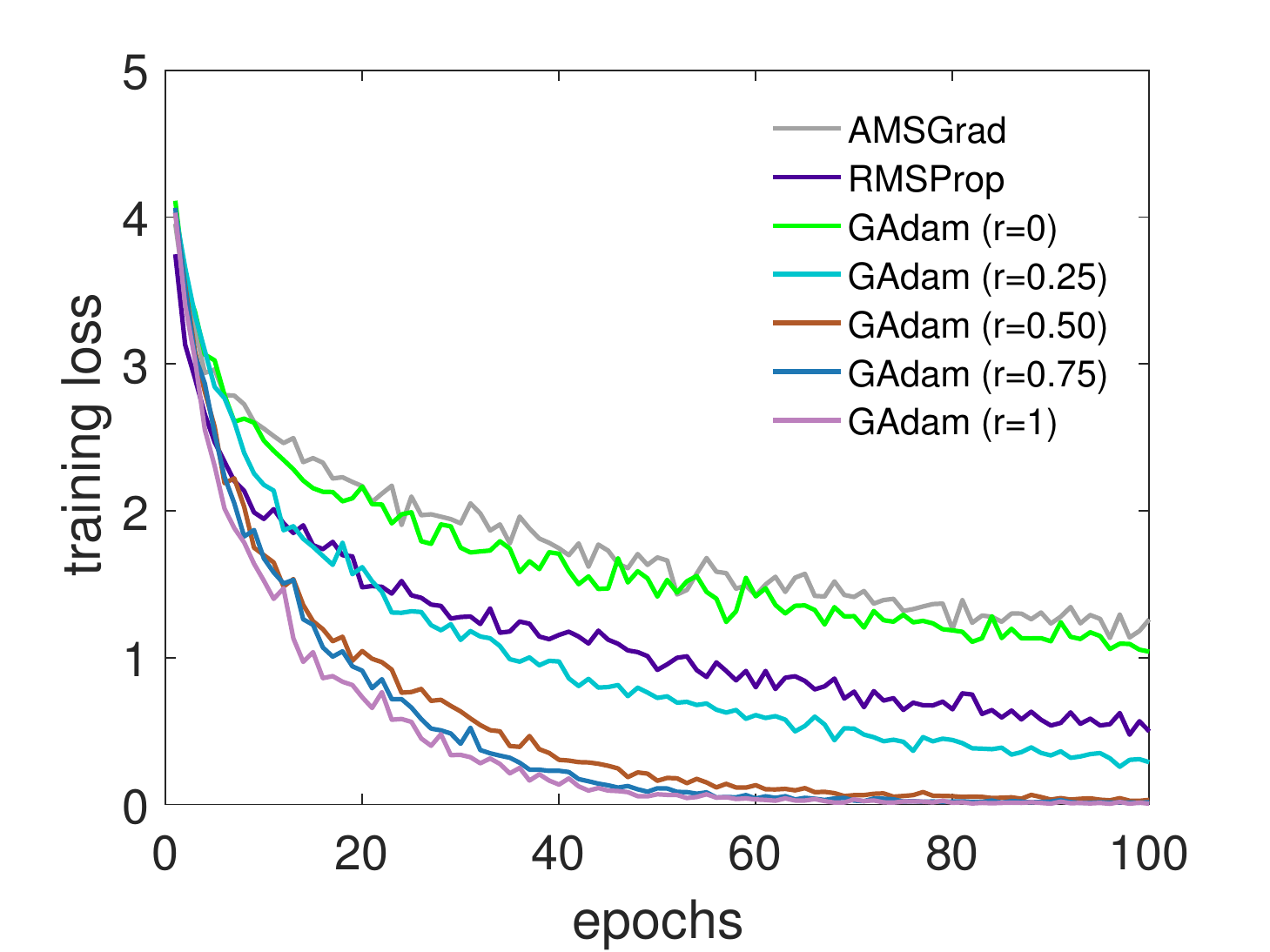}}\label{fig:resnet_a}
\subfigure{\includegraphics[width=0.32\linewidth]{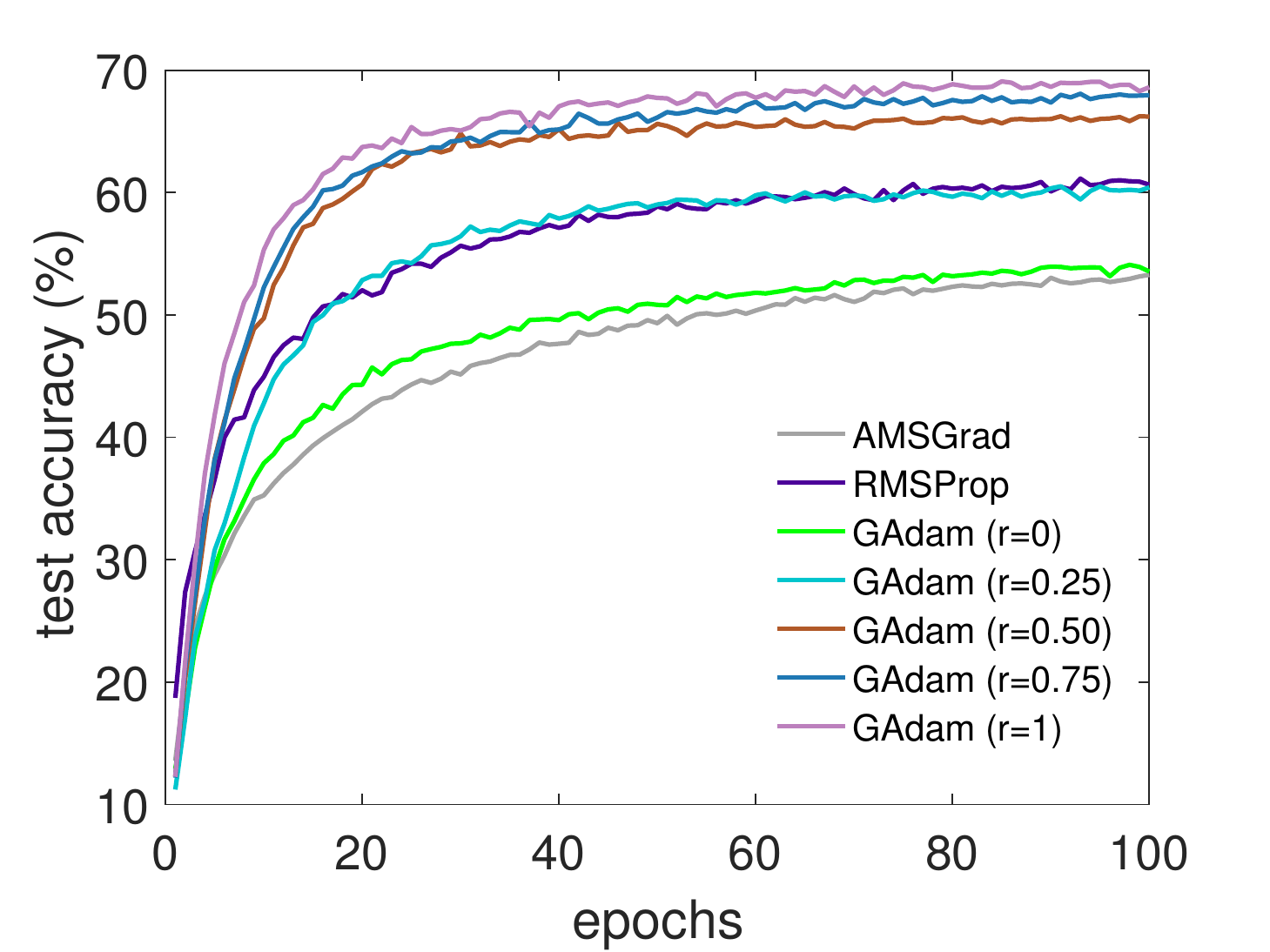}}\label{fig:resnet_b}
\subfigure{\includegraphics[width=0.32\linewidth]{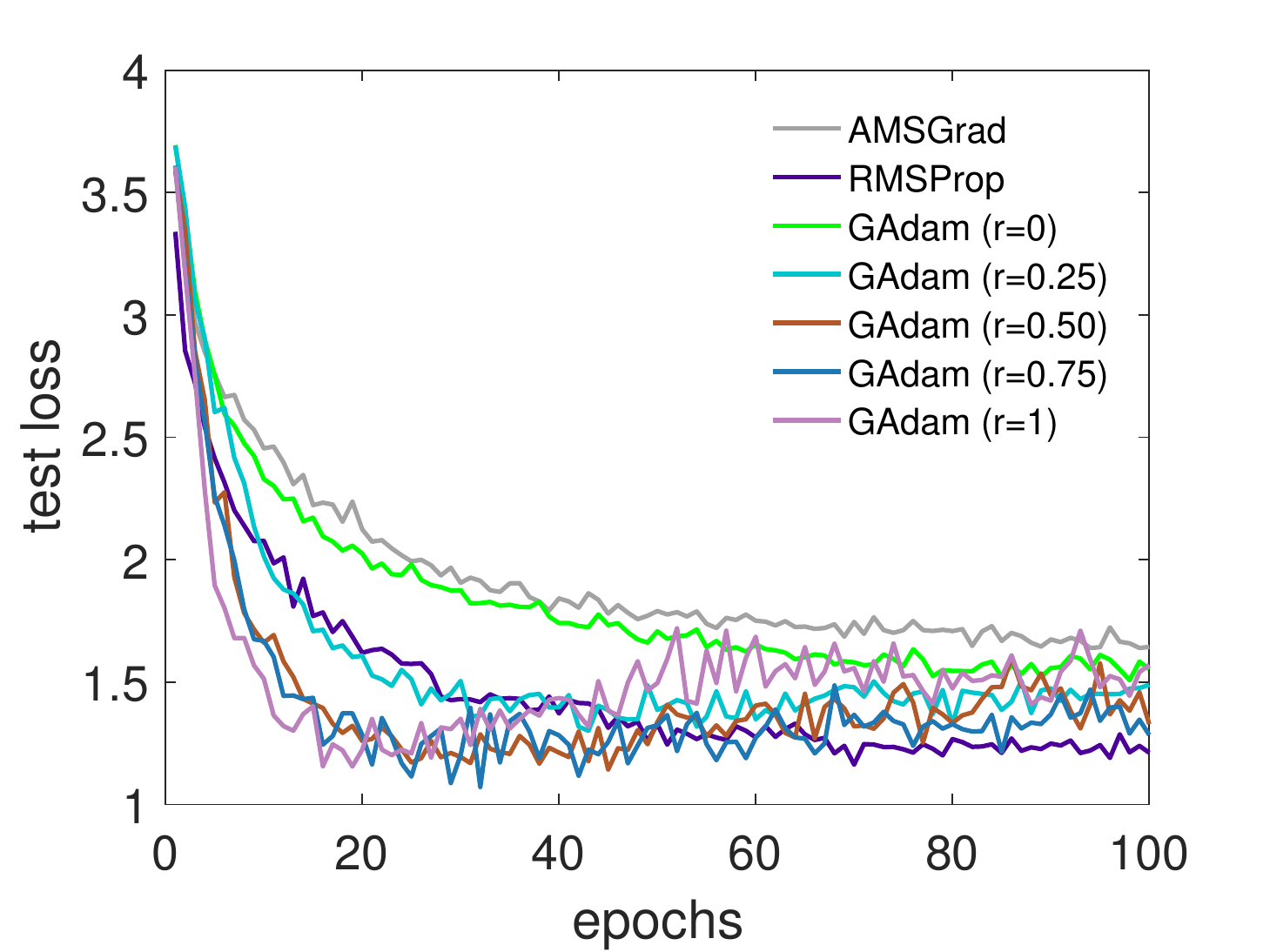}}\label{fig:resnet_c}
\caption{Performance profiles of Generic Adam with $r=\{0, 0.25, 0.5, 0.75, 1\}$, RMSProp, and AMSGrad for training ResNet on the CIFAR-100 dataset. Figures (a), (b), and (c) illustrate training loss vs. epochs, test accuracy vs. epochs, and test loss vs. epochs, respectively.}
\label{fig:ResNet}
\vspace{-0.3cm}
\end{figure*}

In this subsection, we apply Generic Adam 
to train LeNet on the MNIST dataset and ResNet-18 on the CIFAR-100 dataset, respectively, to validate the convergence rates in Corollary \ref{poly-setting}. 
Meanwhile, the comparisons between Generic Adam and AMSGrad \citep{Reddi2018on,chen2018convergence} are also provided to distinguish their differences in training deep neural networks. 
We illustrate the performance profiles in three aspects: training loss vs. epochs, test loss vs. epochs, and test accuracy vs. epochs, respectively. 
MNIST \citep{lecun2010mnist} is composed of ten classes of digits among $\{0, 1, 2, \ldots, 9\}$, which includes 60,000 training examples and 10,000 validation examples. The dimension of each example is $28 \times 28$. 
CIFAR-100 \citep{lecun2010mnist} is composed of 100 classes of $32\times 32$ color images. Each class includes 6,000 images. 
Besides, these images are divided into 50,000 training examples and 10,000 validation examples.
LetNet \citep{lecun1998gradient} used in the experiments is a five-layer convolutional neural network with ReLU activation function whose detailed architecture is described in~\citep{lecun1998gradient}. The batch size is set as $64$. The training stage lasts for $100$ epochs in total. No $\ell_2$ regularization on the weights is used.
ResNet-18 \citep{he2016deep} is a ResNet model containing 18 convolutional layers  for CIFAR-100 classification~\citep{he2016deep}. Input images are down-scaled to $1/8$ of their original sizes after the 18 convolutional layers and then fed into a fully-connected layer for the 100-class classification. The output channel numbers of 1-3 conv layers, 4-8 conv layers, 9-13 conv layers, and 14-18 conv layers are $64$, $128$, $256$, and $512$, respectively.  The batch size is $64$. The training stage lasts for $100$ epochs in total. No $\ell_2$ regularization on the weights is used.

In the experiments, for Generic Adam, we set $\theta_{t}^{(r)} = 1 - (0.001 + 0.999r)/t^r$ with $r \in \{0, 0.25, 0.5, 0.75,1\}$ and $\beta_t =0.9$, respectively; 
for RMSProp, we set $\beta_{t}=0$ and $\theta_{t}=1-\frac{1}{t}$ along with the parameter settings in \citet{mukkamala2017variants}. 
For fairness, the base learning rates $\alpha_t$ in Generic Adam, RMSProp, and AMSGrad are all set as $0.001/\sqrt{t}$. 
Figures \ref{fig:LeNet} and \ref{fig:ResNet} illustrate the results of Generic Adam with different $r$, RMSProp, and AMSGrad for training  LeNet on MNIST and training ResNet-18 on CIFAR-100, respectively. 
We can see that AMSGrad and Adam (Generic Adam with $r=0$) decrease the training loss slowest and show the worst test accuracy among the compared optimizers. 
One possible reason is due to the use of constant $\theta$ in AMSGrad and original Adam.
By Figures \ref{fig:LeNet} and \ref{fig:ResNet}, we can observe that the convergences of Generic Adam are extremely sensitive to the choice of parameter $\theta_{t}$. 
Larger $r$ can contribute to a faster convergence rate of Generic Adam, which corroborates the theoretical result in Corollary \ref{poly-setting}. 
Additionally, the test accuracies in Figures \ref{fig:LeNet}(b) and \ref{fig:ResNet}(b) indicate that a smaller training loss can contribute to a higher test accuracy for Generic Adam.

\subsection{Experiments on Practical Adam}
\textcolor{black}{
\subsubsection{Ablation Study between Batchsize and Optimal Learning Rate}
In this section, we apply mini-batch Adam algorithms and mini-batch SGD algorithms to the following quadratic minimization task:
\[
\min_x \mathbb{E}_{\xi \sim \mathcal{N}(0,100I)} \left[x^TAx - b^Tx + \xi^Tx\right],
\]
where, for simplicity, $A \in \mathbb{R}^{3\times 3}$ and $b \in \mathbb{R}^3$.
We optimize $x$ with batch size 1 to 320, and grid search the best learning rate from $\{k/2000| k = 1,2,\cdots, 2000\}$. For each pair of batch size and learning rate, we randomly sample 500 trials and take the average gradient norm as the criteria. Fig. \ref{toy_prac} shows the result of the best learning rate and averaged gradient norm for different batch sizes after 200 optimization iterations. 
}
\textcolor{black}{
From the figure, we can verify that when the batch size becomes larger, the optimal learning rate for SGD and SGD-momentum increases a lot (from 0.005 to 0.04). Meanwhile, with the adaptive learning rate, the optimal learning rate does not change too much (between 0.02 to 0.04). Hence, it shows the benefit of adaptive learning rate methods compared to SGD and verifies Remark \ref{Remark_minibatch}.
}

\begin{figure}[ht]
    \centering
    \includegraphics[width=0.4\linewidth]{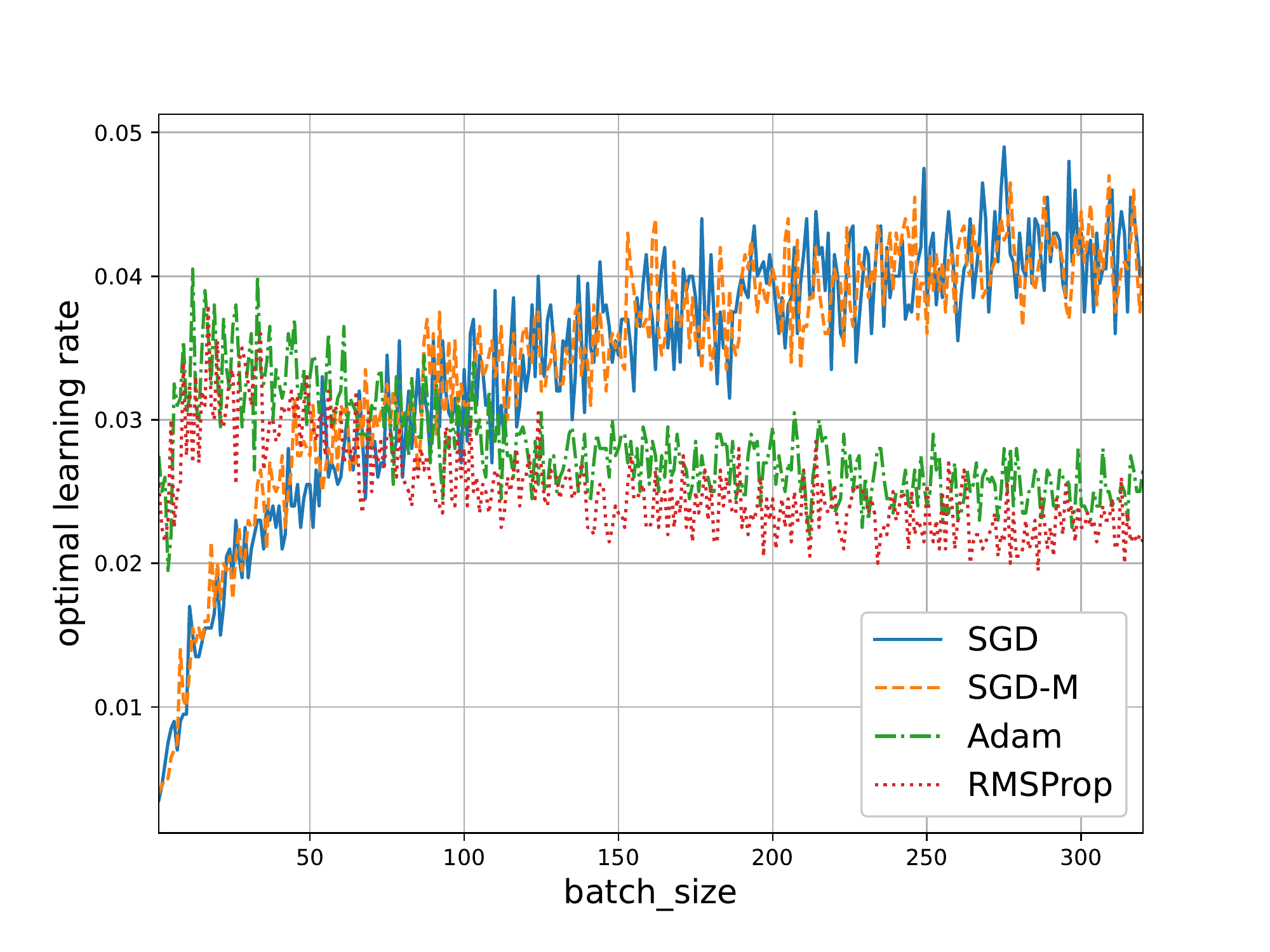}
    \includegraphics[width=0.4\linewidth]{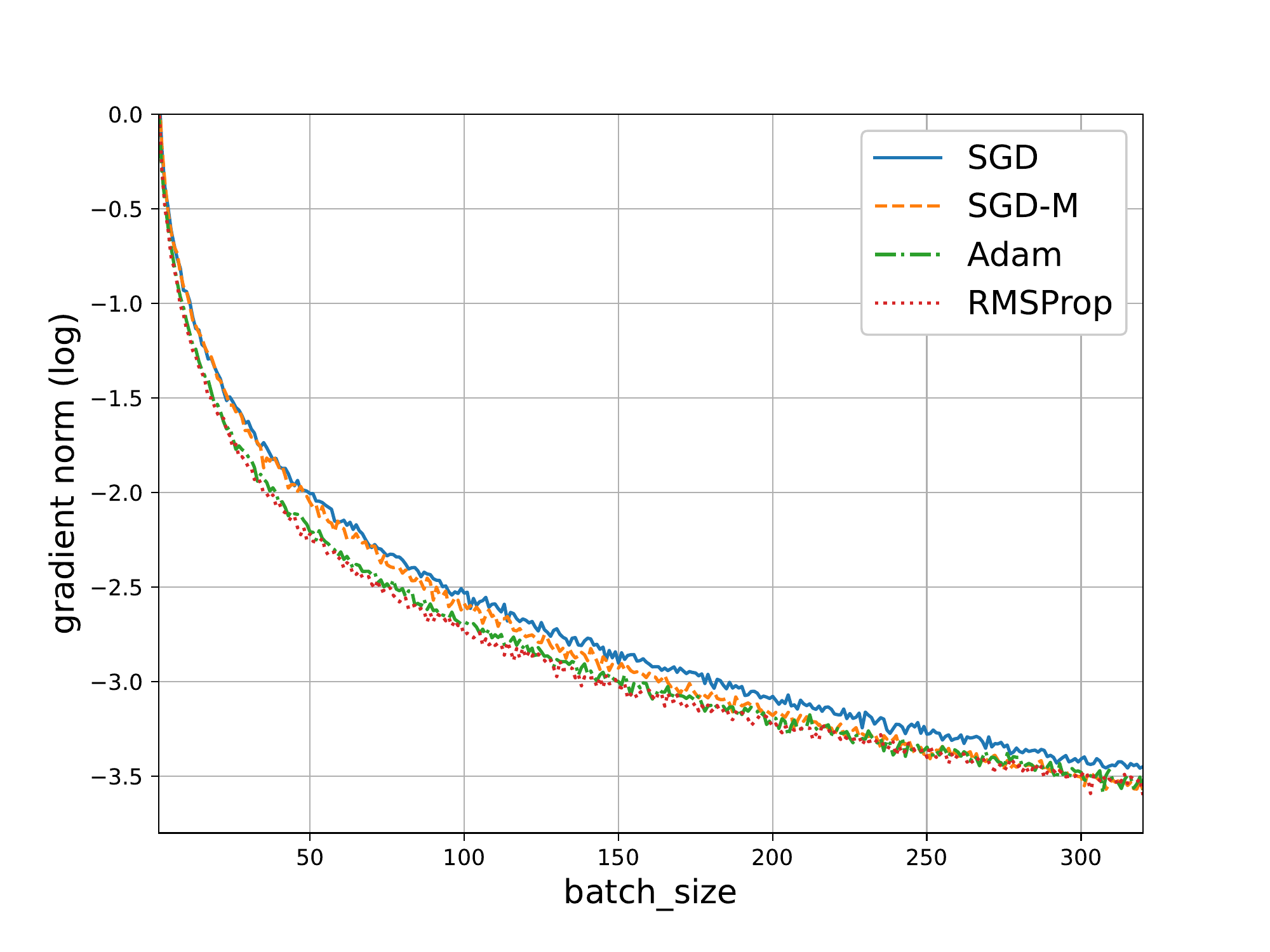}
    \caption{\textcolor{black}{Performance profiles of optimal learning rate v.s. batch size and optimality gap v.s. batch size. }}
    \label{toy_prac}
\end{figure}

\subsubsection{Vision Tasks}
In this section we apply the mini-batch Adam algorithm to train LeNet on the MNIST dataset and ResNet-18 on the CIFAR-100 dataset. Datasets and network architecture are the same as they are described in Section \ref{ex52}. But instead of using $\alpha_t = \alpha/\sqrt{t}$, $\theta_t = 1-\theta/t$, we set $\beta_t = 0.9$ and $\theta_t = 0.99$ for Adam and AMSGrad, and set $\beta_t = 0.9$ for RMSProp. We use different batchsizes \{32, 64, 128\} to train networks. Besides, when training ResNet-18 on the CIFAR100 dataset, we use an $\ell_2$ regularization on weights, the coefficient of the regularization term to $5e-4$. We use grid search in $[1e-2,\ 5e-3,\ 1e-3,\ 5e-4,\ 1e-4]$ for $\alpha_t$ with respect to test accuracy. In addition, when training ResNet-18 on the CIFAR100 dataset, $\alpha_t$ will reduce to $0.2\times \alpha_t$ every 19550 iterations (50 epochs for the 128 batchsize setting), which (learning rate decay) is commonly used in practice . The experimental results are shown in Figure \ref{fig:LeNet1} and Figure \ref{fig:ResNet1}. It can be shown that larger batchsize can give lower training loss in all experiments. However, large batchsize for training does not imply higher test accuracy or lower test loss, which needs to be further explored and examined.

\begin{figure*}[htpb]
\centering
\subfigure{\includegraphics[width=0.32\linewidth]{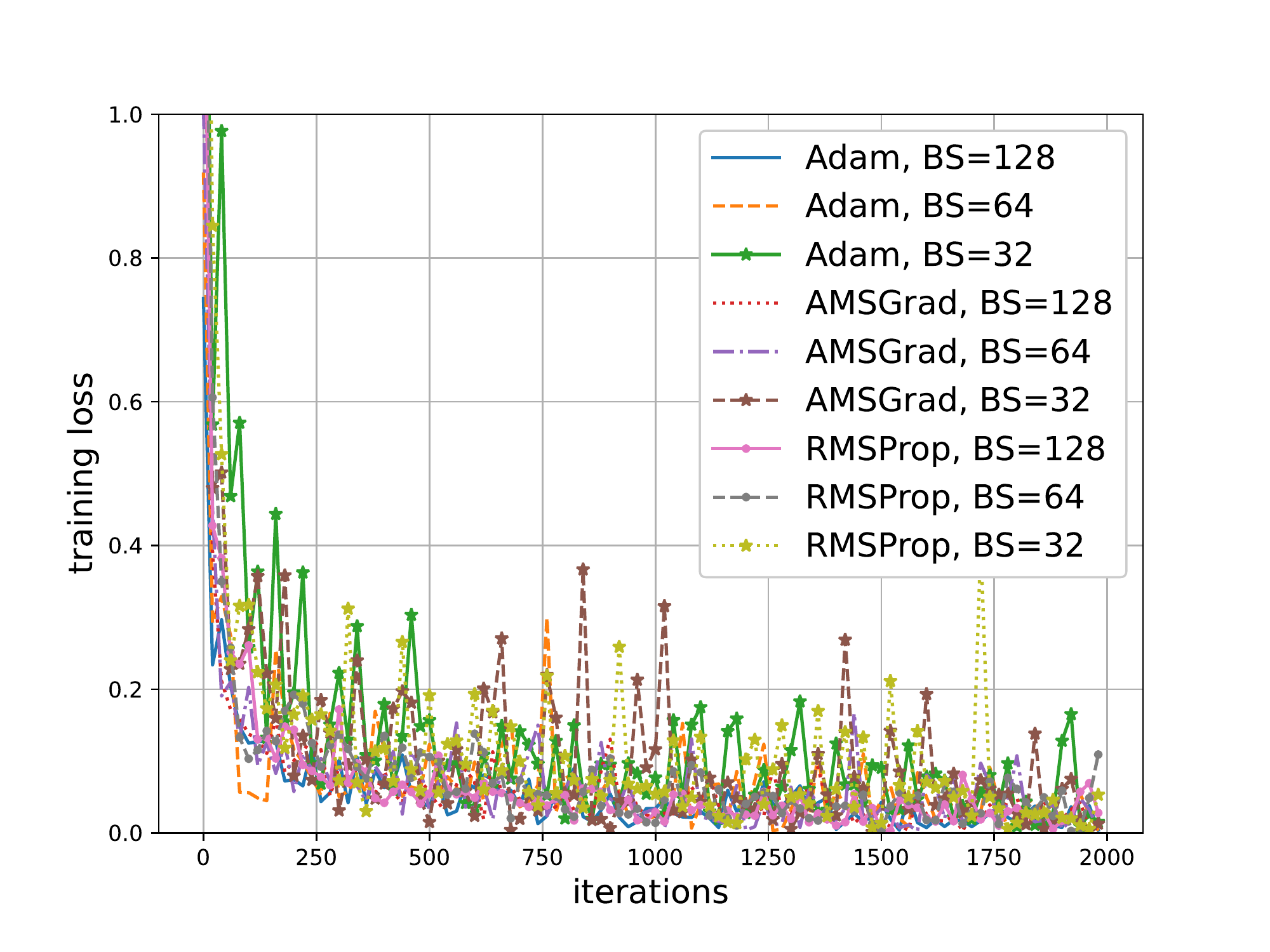}}
\subfigure{\includegraphics[width=0.32\linewidth]{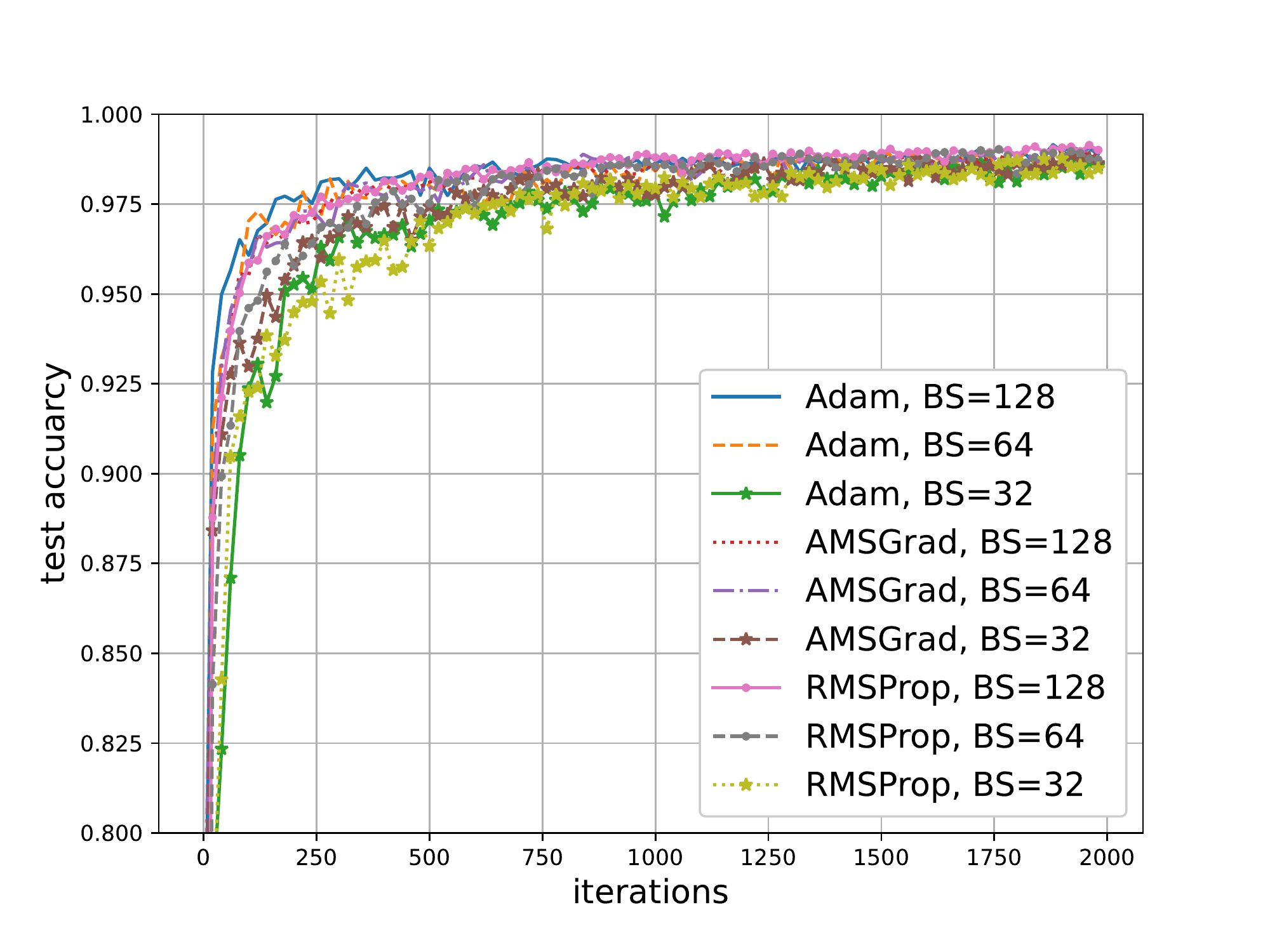}}
\subfigure{\includegraphics[width=0.32\linewidth]{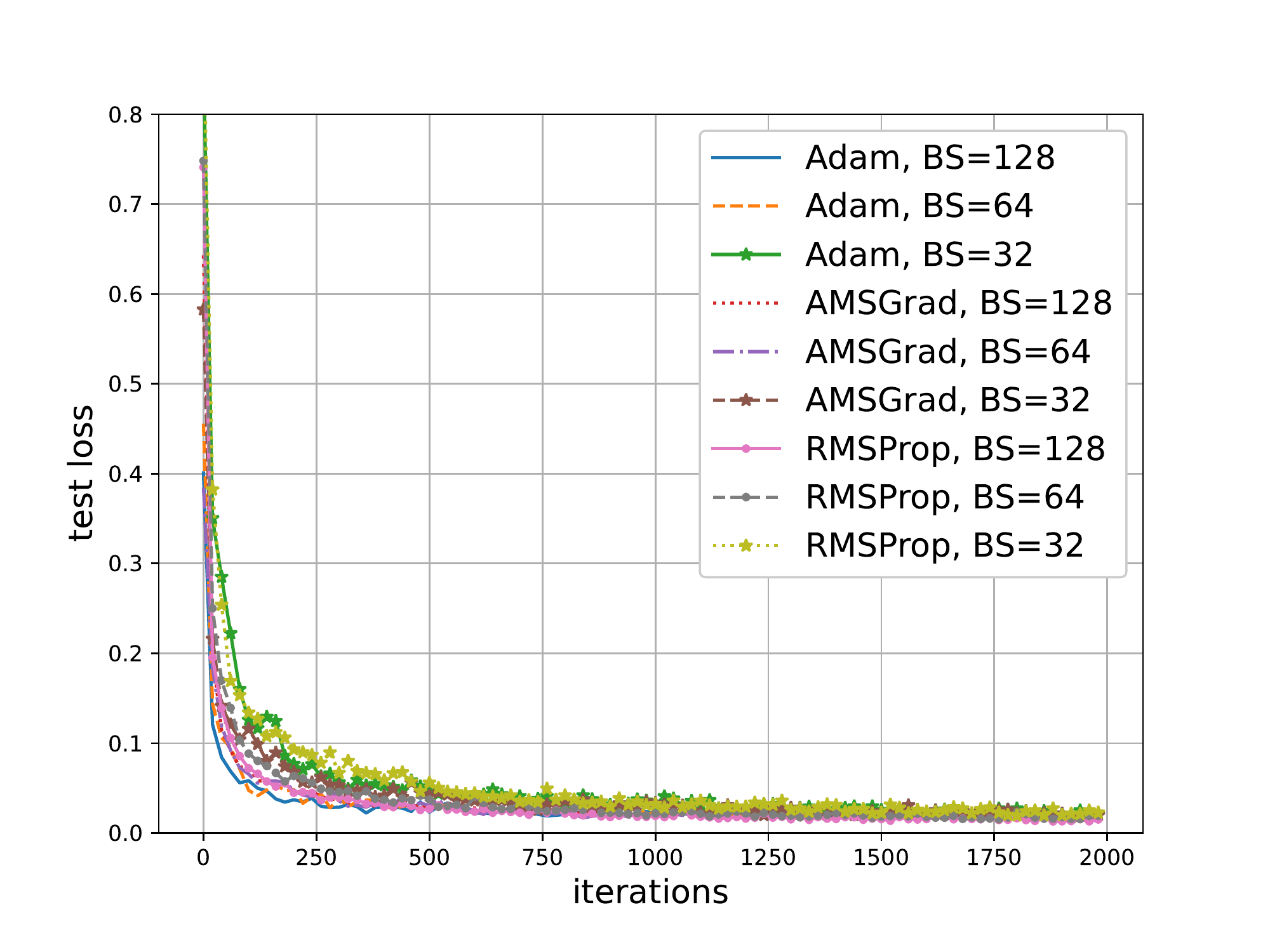}}
\caption{Performance profiles of mini-batch Adam, RMSProp and AMSGrad on MNIST, with batchsize = $\{32,\ 64,\ 128\}$.}
\label{fig:LeNet1}
\end{figure*}

\begin{figure*}[htpb]
\centering
\subfigure{\includegraphics[width=0.32\linewidth]{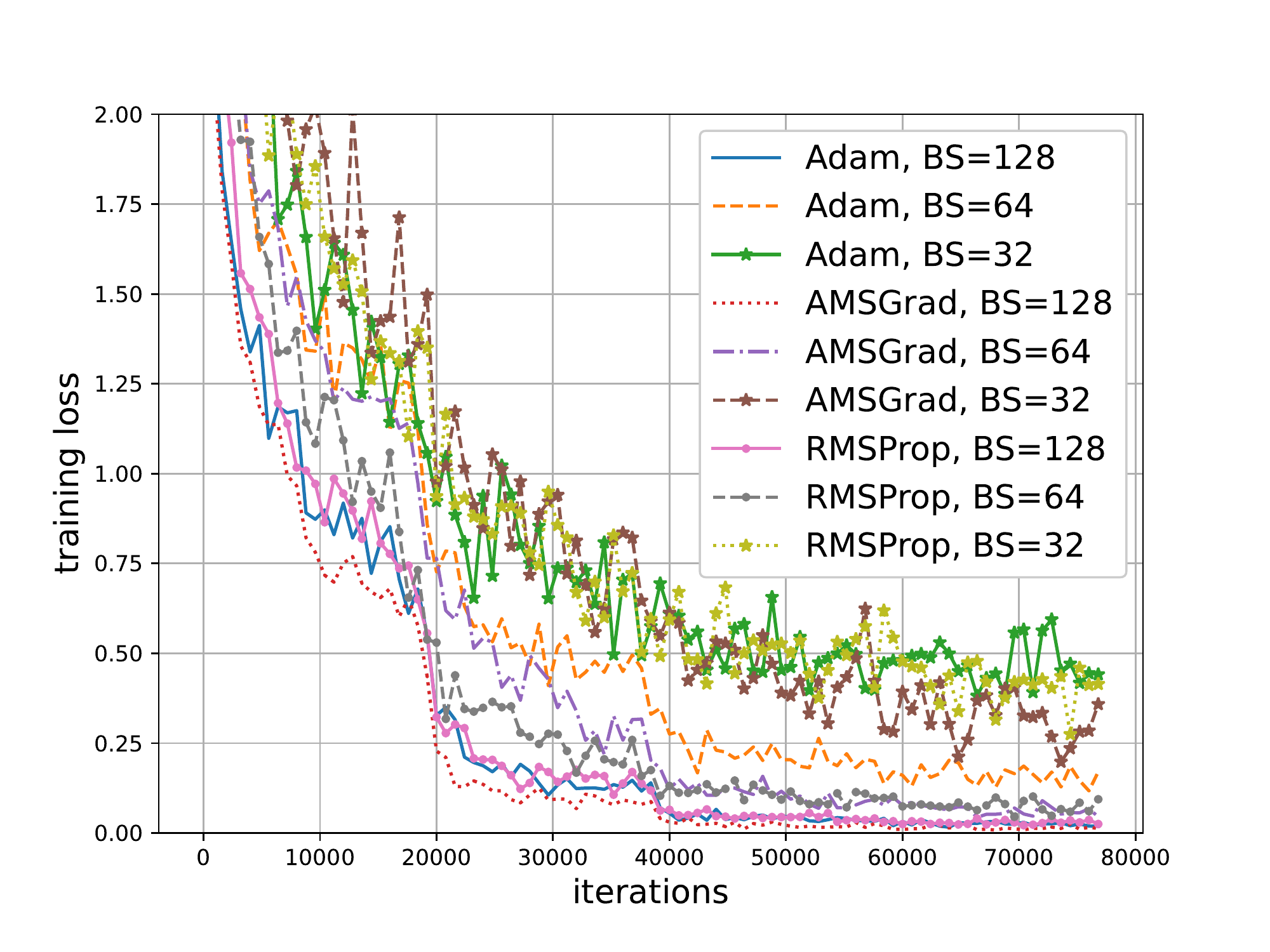}}
\subfigure{\includegraphics[width=0.32\linewidth]{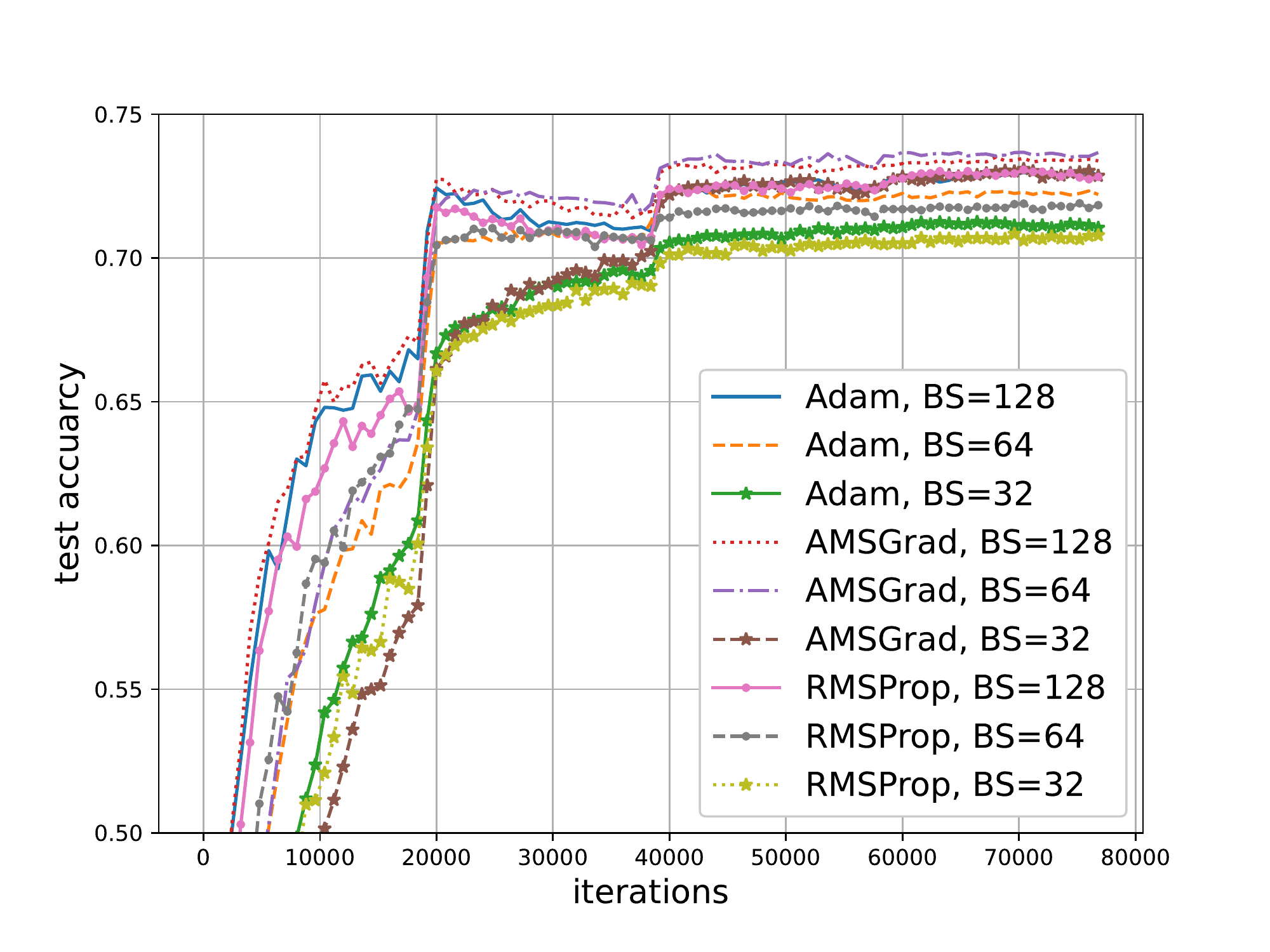}}
\subfigure{\includegraphics[width=0.32\linewidth]{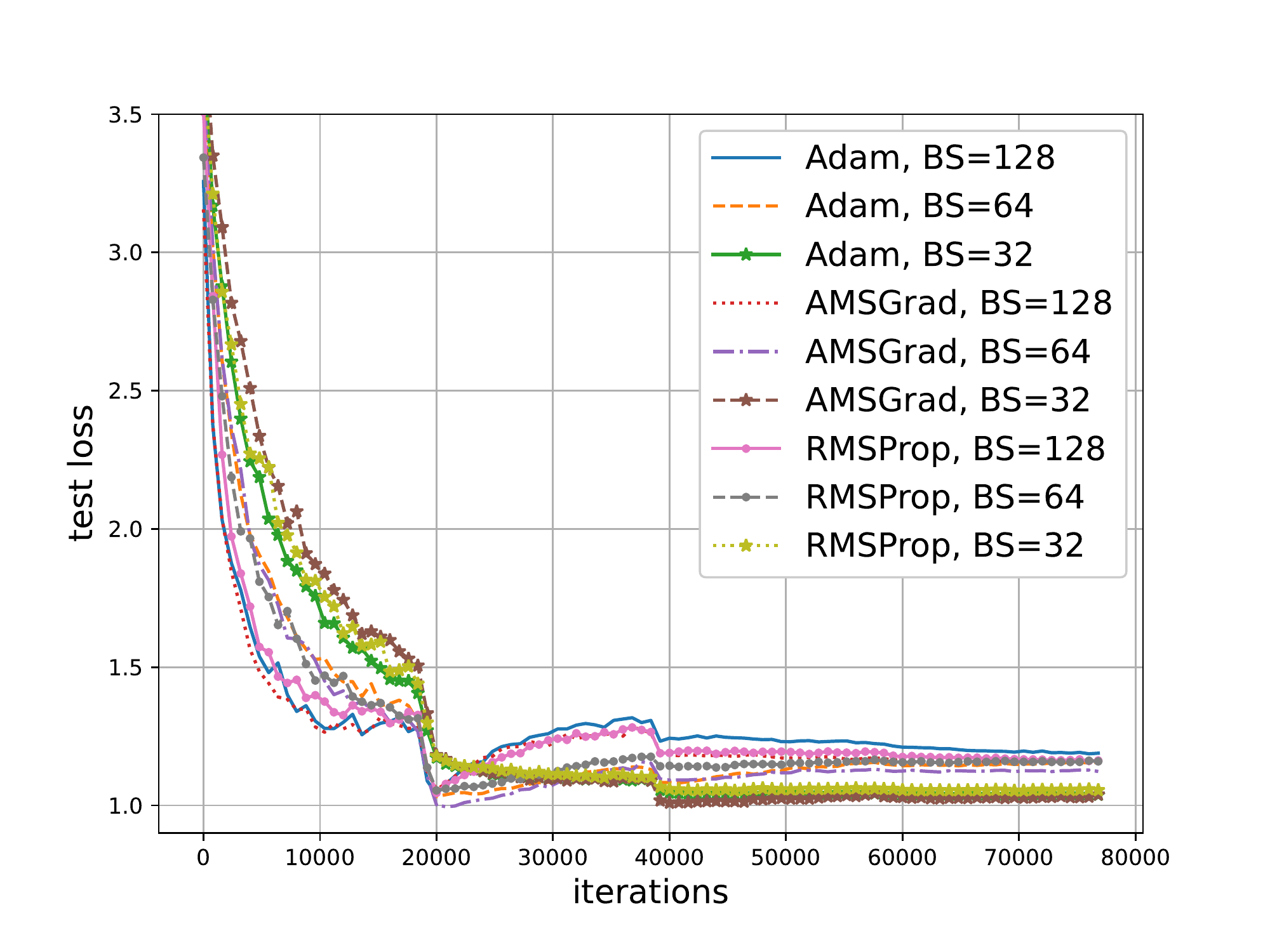}}
\caption{Performance profiles of mini-batch Adam, RMSProp and AMSGrad on CIFAR100, with batchsize = $\{32,\ 64,\ 128\}$.}
\label{fig:ResNet1}
\vspace{-0.3cm}
\end{figure*}

\subsubsection{Transformer XL on WikiText-103}
\textcolor{black}{Also, we applied mini-batch Adam to train a base model of Transformer XL \citep{dai2019transformer} on the dataset WikiText-103 \citep{merity2016pointer}. The base model of Transformer XL contains 16 self-attention layers. In each self-attention layer, there are 10 heads, and the encoding dimension of each head is set to 41. The WikiText-103 dataset is a collection of over 100 million tokens extracted from the set of verified ‘Good’ and ‘Featured’ articles on Wikipedia.   We adopt the same parameter settings provided by the authors but test on batch size $\{30,60,120\}$. The results are shown in Figure \ref{fig:Transformer}. Also, as it is shown in the figure, a larger batch size can give a lower training loss in all experiments. Meanwhile, in the figure, AMSGrad and Adam achieve much better performance than SGD-M, which shows the benefit of using adaptive methods instead of SGD-based methods, just like what was mentioned in \cite{zhang2019adaptive}. }

\begin{figure*}[htpb]
\centering
\subfigure{\includegraphics[width=0.32\linewidth]{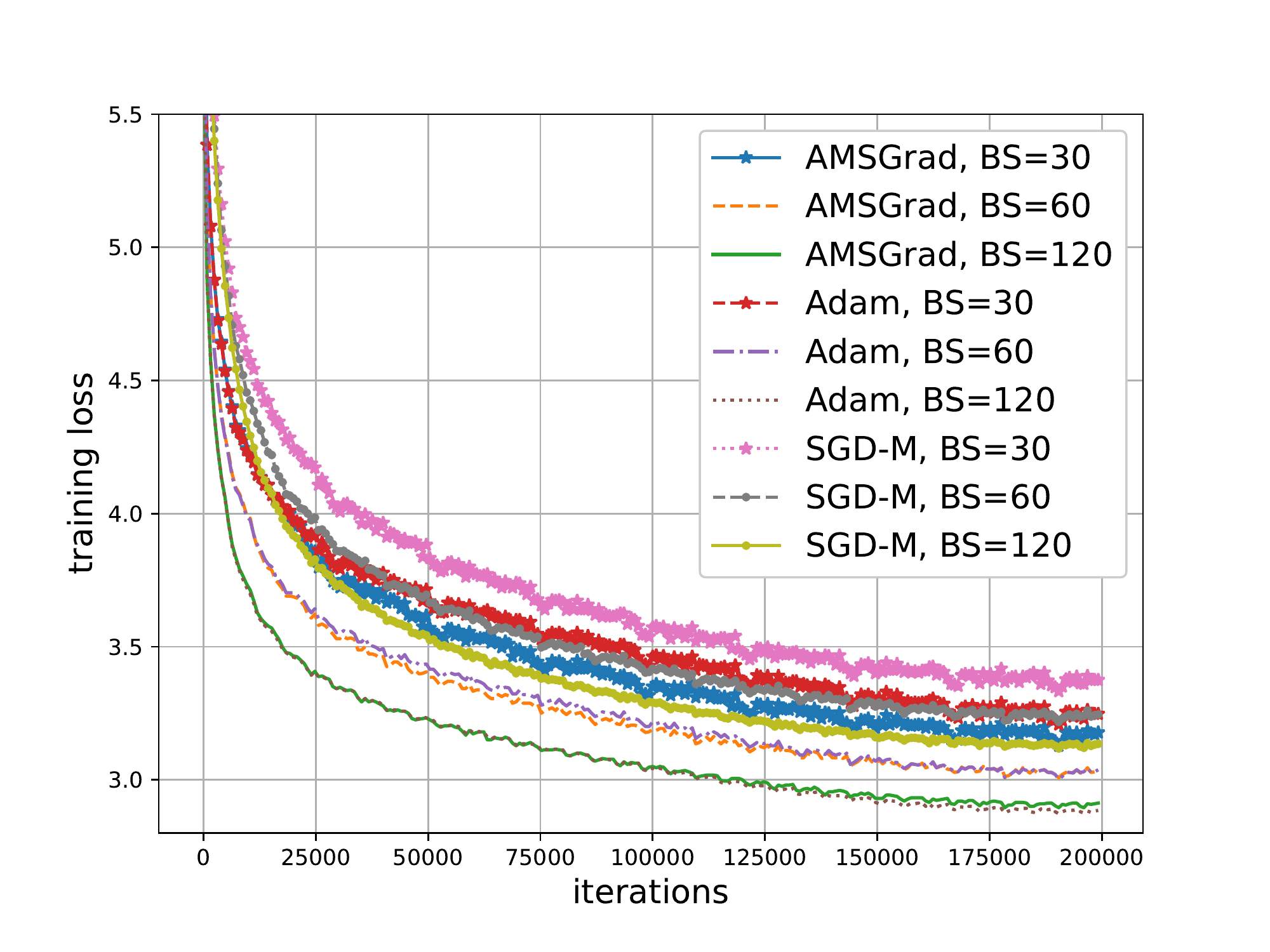}}
\subfigure{\includegraphics[width=0.32\linewidth]{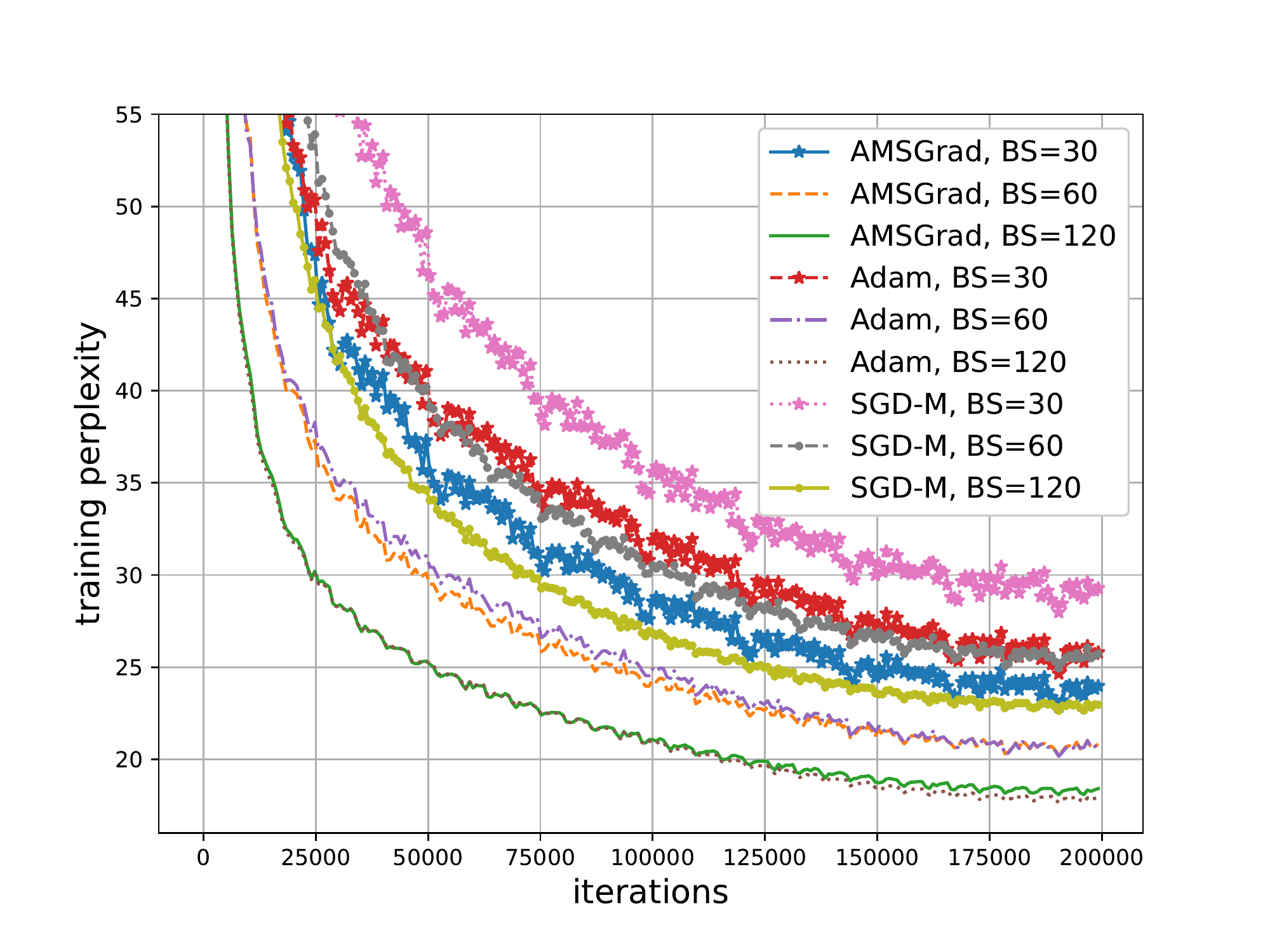}}
\subfigure{\includegraphics[width=0.32\linewidth]{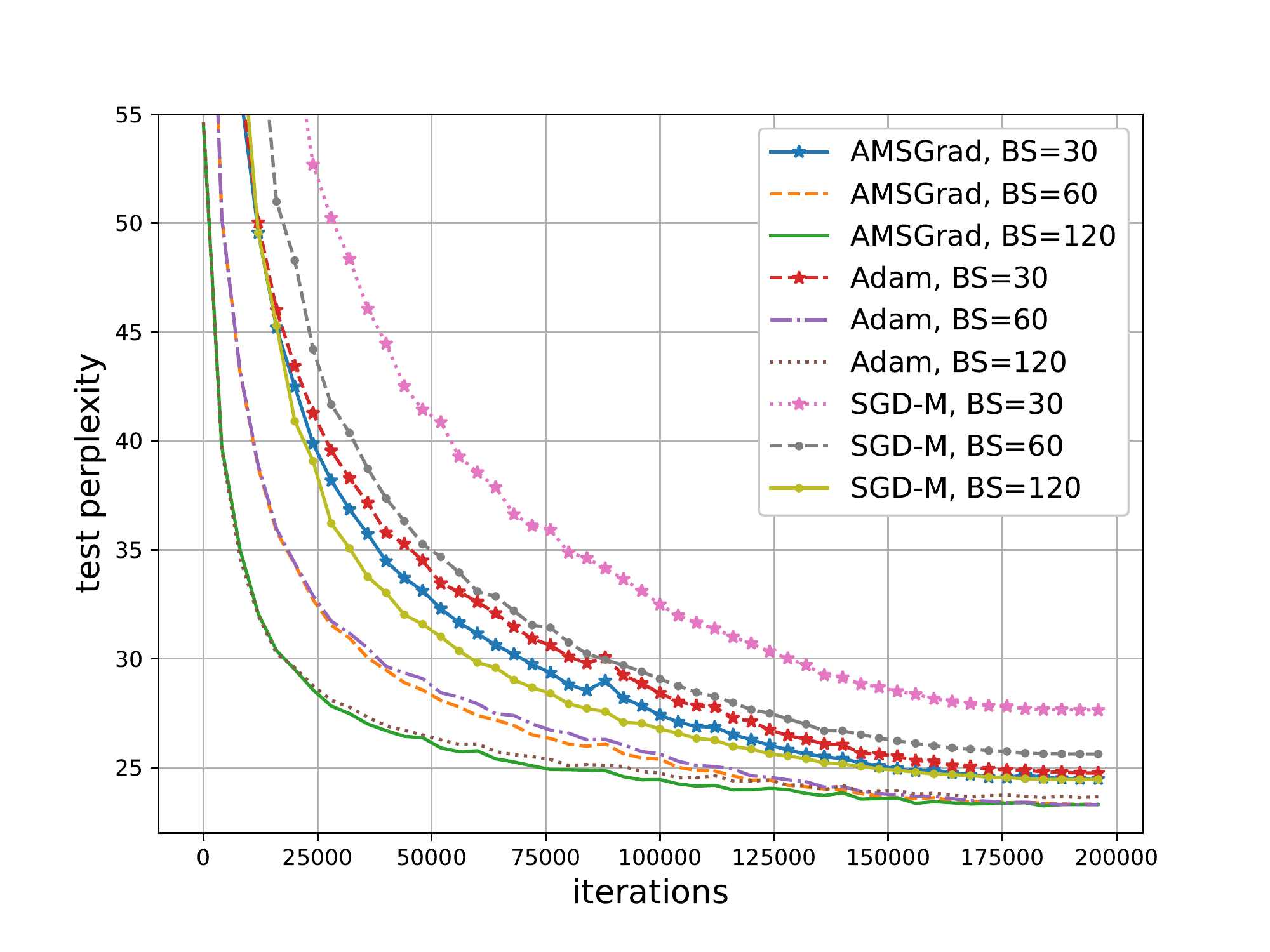}}
\caption{\textcolor{black}{Performance profiles of mini-batch Adam, SGD-M and AMSGrad on WikiText-103, with batchsize = $\{30,\ 60,\ 120\}$.}}
\label{fig:Transformer}
\vspace{-0.3cm}
\end{figure*}

\section{Conclusions}
In this work, we delved into the convergences of Adam, and presented an easy-to-check sufficient condition to guarantee their convergences in the non-convex stochastic setting. This sufficient condition merely depends on the base learning rate and the linear combination parameter of second-order moments. Relying on this sufficient condition, we found that the divergences of Adam are possibly due to the incorrect parameter settings. Besides, when encountering the practice Adam, we theoretically showed that the number of samples will linearly speed up the convergence in both the mini-batch setting and distributed setting, which closes the gap between theory and practice. At last, the correctness of theoretical results has also been verified via the counterexample and training deep neural networks on real-world datasets.

\section*{Acknowledgement}
This work is supported by the Major Science and Technology Innovation 2030 “Brain Science and Brain-like Research” key project (No. 2021ZD0201405).


\appendix

\section{Key Lemma to prove Theorem \ref{convergence_in_expectation} and Theorem \ref{practical_adam}}
In this section we provide the necessary lemmas for the proofs of Theorem \ref{convergence_in_expectation} and Theorem \ref{practical_adam}. 
First, we give some notations for simplifying the following proof.

\paragraph{Notations} We use bold letters to represent vectors. The $k$-th component of a vector $\bm{v}_t$ is denoted as ${v}_{t, k}$. The inner product between two vectors $\bm{v}_t$ and $\bm{w}_t$ is denoted as $\langle \bm{v}_t, \bm{w}_t \rangle$. Other than that, all computations that involve vectors shall be understood in the component-wise way. We say a vector $\bm{v}_t \geq 0$ if every component of $\bm{v}_t$ is non-negative, and $\bm{v}_t \geq \bm{w}_t$ if $v_{t,k} \geq w_{t,k}$ for all $k=1, 2, \ldots, d$. The $\ell_1$ norm of a vector $\bm{v}_t$ is defined as $\norm{\bm{v}_t}_1 = \sum_{k=1}^d |{v}_{t, k}|$. The $\ell_2$ norm is defined as $\norm{\bm{v}_t}^2 =\langle \bm{v}_t, \bm{v}_t \rangle = \sum_{k=1}^d |{v}_{t,k}|^2$. Given a positive vector $\bm{\hat{\eta}}_t$, it will be helpful to define the following weighted norm: $\norm{\bm{v}_t}^2_{\bm{\eta}_t} = \langle \bm{v}_t, \bm{\hat{\eta}}_t \bm{v}_t \rangle = \sum_{k=1}^d \hat{\eta}_{t, k}|{v}_{t, k}|^2$.

\begin{lemma}\label{lem5}
Given $S_0 > 0$ and a non-negative sequence $\{s_t\}$, let $S_t = S_0 + \sum_{i=1}^t s_i$ for $t \geq 1$. Then the following estimate holds 
\begin{equation}
\sum_{t=1}^T \frac{s_t}{S_t} \leq \log(S_T) - \log(S_0).
\end{equation}
\end{lemma}
\begin{proof}
The finite sum $\sum_{t=1}^T {s_t}/{S_t}$ can be interpreted as a Riemann sum  $\sum_{t=1}^T (S_t - S_{t-1})/S_t.$
Since $1/x$ is decreasing on the interval $(0, \infty)$, we have 
$$\sum_{t=1}^T \frac{S_t-S_{t-1}}{S_t} \leq \int_{S_0}^{S_T} \frac{1}{x} d x = \log(S_T) - \log(S_0).$$
The proof is completed.
\end{proof}

\begin{lemma}[Abel's Lemma - Summation by parts]\label{lem2-003}
Let $\{u_t\}$ and $\{s_t\}$ be two non-negative sequences. Let $S_t = \sum_{i=1}^t s_i$ for $t \geq 1$. Then 
\begin{equation}
\sum_{t=1}^T u_t s_t = \sum_{t=1}^{T-1} (u_t - u_{t+1})S_t + u_T S_T.
\end{equation}
\end{lemma}
\begin{proof}
Let $S_0 = 0$. Then
\begin{equation}
\sum_{t=1}^T u_t s_t = \sum_{t=1}^T u_t (S_t - S_{t-1}) = \sum_{t=1}^{T-1} u_tS_t - \sum_{t=1}^{T-1} u_{t+1}S_t + u_TS_T = \sum_{t=1}^{T-1}(u_t - u_{t+1})S_t + u_TS_T.
\end{equation}
The proof is completed.
\end{proof}

\begin{lemma}\label{lem2-002}
Let $\{\theta_t\}$ and $\{\alpha_t\}$ satisfy the restrictions (R2) and (R3). For any $i \leq t$, we have
\begin{equation}
\chi_t \leq C_0 \chi_i \text{~~and~~} \alpha_t \leq C_0 \alpha_i.
\end{equation}

\end{lemma}
\begin{proof}
For any $i \leq t$, since the sequence $\{a_t\}$ is non-increasing, we have $a_t \leq a_i$. Hence,
\begin{equation*}
\chi_t = \frac{\alpha_t}{\sqrt{1-\theta_t}} \leq C_0 a_t \leq C_0 a_i \leq C_0 \frac{\alpha_i}{\sqrt{1-\theta_i}} = C_0 \chi_i,
\end{equation*}
which proves the first inequality. On the other hand, since $\{\theta_t\}$ is non-decreasing, it holds
\begin{equation*}
\alpha_t \leq C_0 \frac{\sqrt{1-\theta_t}}{\sqrt{1-\theta_i}} \alpha_i \leq C_0 \alpha_i = C_0 \alpha_i.
\end{equation*}
The proof is completed.
\end{proof}

\medskip

Let $\Theta_{(t,i)} = \prod_{j=i+1}^t \theta_j$ for $i < t$ and $\Theta_{(t,t)}=1$ by convention. 

\begin{lemma}\label{lem2-001}
Fix a constant $\theta'$ with $\beta^2 < \theta' < \theta$. Let $C_1$ be as given as Eq.~\eqref{Constant-C1} in the main paper. For any $i \leq t$, we have
\begin{equation}
\Theta_{(t,i)} \geq C_1 (\theta')^{t-i}. 
\end{equation}
\end{lemma}
\begin{proof}
For any $i \leq t$, since $\theta_j \geq \theta'$ for $j \geq N$, and $\theta_j < \theta'$ for $j < N$, we have
\[\Theta_{(t,i)} = \prod_{j=i+1}^t\theta_j \geq \left(\prod_{j=i+1}^N\theta_j\right)(\theta')^{t-N} = \left(\prod_{j=i+1}^N({\theta_j}/{\theta'})\right)(\theta')^{t-i}
\geq \left(\prod_{j=1}^N ({\theta_j}/{\theta'})\right)(\theta')^{t-i}.\]
We take the constant $C_1 = \prod_{j=1}^N (\theta_j/\theta')$, where $N$ is the maximum of the indices for which $\theta_j < \theta'$. The proof is completed.
\end{proof}
\begin{remark}
If $\theta_t = \theta$ is a constant, we have $\Theta_{(t,i)} = \theta^{t-i}$. In this case we can take $\theta' = \theta$ and $C_1 = 1$. 
\end{remark}

\begin{lemma}\label{lem1-001}
Let $\gamma := \beta^2/{\theta'}$. We have the following estimate
\begin{equation}
\bm{m_t}^2 \leq \frac{1}{C_1(1-\gamma)(1-\theta_t)}\bm{v}_t,~\forall t.
\end{equation}
\end{lemma}
\begin{proof}
Let $B_{(t, i)} = \prod_{j = i+1}^t \beta_j$ for $i < t$ and $B_{(t,t)} = 1$ by convention. By the iteration formula $\bm{m}_t = \beta_t \bm{m}_{t-1} + (1-\beta_t)\bm{g}_t$ and $\bm{m}_0 = \bm{0}$, we have
\begin{equation*}
\bm{m}_t = \sum_{i=1}^t \left(\prod_{j=i+1}^t\beta_j\right) (1-\beta_i)\bm{g}_i = \sum_{i=1}^t B_{(t,i)}(1-\beta_i)\bm{g}_i.
\end{equation*}
Similarly, by $\bm{v}_t = \theta_t \bm{v}_{t-1} + (1-\theta_t)\bm{g}_t^2$ and $\bm{v}_0 = \bm{\epsilon}$, we have
\begin{equation*}
\bm{v}_t = \left(\prod_{j=1}^t\theta_j\right)\bm{\epsilon} + \sum_{i=1}^t \left(\prod_{j=i+1}^t \theta_j\right)\left({1-\theta_i}\right) \bm{g}_i^2
\geq \sum_{i=1}^t \Theta_{(t,i)}(1-\theta_i) \bm{g}_i^2.
\end{equation*}
It follows by arithmetic inequality that
\begin{equation*}
\begin{split}
\bm{m}_t^2 &= \left( \sum_{i=1}^t \frac{(1-\beta_i)B_{(t,i)}}{\sqrt{(1-\theta_i)\Theta_{(t,i)}}} \sqrt{(1-\theta_i)\Theta_{(t,i)}} \bm{g}_i\right)^2 \\
&\leq \left(\sum_{i=1}^t \frac{(1 - \beta_i)^2B_{(t,i)}^2}{(1-\theta_i)\Theta_{(t,i)}}\right)
\left(\sum_{i=1}^t \Theta_{(t,i)}(1-\theta_i)\bm{g}_i^2\right) 
\leq \left(\sum_{i=1}^t \frac{(1-\beta_i)^2B_{(t,i)}^2}{(1-\theta_i)\Theta_{(t,i)}}\right) \bm{v}_t.
\end{split}
\end{equation*}
Note that $\{\theta_t\}$ is non-decreasing by (\textbf{R}2), and $B_{(t,i)} \leq \beta^{t-i}$ by (\textbf{R}1). By Lemma \ref{lem2-001}, we have
\begin{equation*}\label{1-005} 
\sum_{i=1}^t \frac{(1-\beta_i)^2B_{(t,i)}^2}{(1-\theta_i)\Theta_{(t,i)}}
\leq \frac{1}{C_1(1-\theta_t)}\sum_{i=1}^t \left(\frac{\beta^2}{\theta'}\right)^{t-i} \leq \frac{1}{C_1(1-\theta_t)}\sum_{k=0}^{t-1} \gamma^k 
\leq \frac{1}{C_1(1-\gamma)(1-\theta_t)}. 
\end{equation*}
The proof is completed.
\end{proof}


\medskip

Let $\bm{\Delta}_t := \bm{x}_{t+1} - \bm{x}_t = -\alpha_t\bm{m}_t/\sqrt{\bm{v}_t}$. Let $\bm{\hat{v}}_t = \theta_t \bm{v}_{t-1} + (1-\theta_t) \bm{\delta_t}^2$, where $\bm{\delta_t}^2 = \mathbb{E}_t \left[\bm{g}_t^2\right]$ and let $\bm{\hat{\eta}_t} = \alpha_t/\sqrt{\bm{\hat{v}_t}}$. 

\begin{lemma}\label{lem1-002} 
With the notations above, the following equality holds
\begin{equation}\label{1-008}
\begin{split}
\bm{\Delta}_t - \frac{\beta_t\alpha_t}{\sqrt{\theta_t}\alpha_{t-1}} \bm{\Delta}_{t-1}
= -(1 - \beta_t)\bm{\hat{\eta}}_t\bm{g}_t + \bm{\hat{\eta}}_t\bm{g}_t\frac{(1-\theta_t)\bm{g}_t}{\sqrt{\bm{v}}_t }\bm{A}_t
+ \bm{\hat{\eta}}_t\bm{\delta}_t\frac{(1-\theta_t)\bm{g}_t}{\sqrt{\bm{v}_t}}\bm{B}_t,
\end{split}
\end{equation}
where
\begin{equation*}
\begin{split}
\bm{A}_t &= \frac{\beta_t\bm{m}_{t-1}}{\sqrt{\bm{v}_t}+\sqrt{\theta_t\bm{v}_{t-1}}}+\frac{(1-\beta_t)\bm{g}_t}{\sqrt{\bm{v}}_t+\sqrt{\bm{\hat{v}}_t}},\\
\bm{B}_t &= \left(\frac{\beta_t\bm{m}_{t-1}}{\sqrt{\theta_t\bm{v}_{t-1}}}\frac{\sqrt{1-\theta_t}\bm{g}_t}{\sqrt{\bm{v}}_t + \sqrt{\theta_t\bm{v}_{t-1}}}\frac{\sqrt{1-\theta_t}\bm{\delta}_t}{\sqrt{\bm{\hat{v}}_t}+\sqrt{\theta_t\bm{v}_{t-1}}}\right) 
- \frac{(1-\beta_t)\bm{\delta}_t}{\sqrt{\bm{v}_t}+\sqrt{\bm{\hat{v}}_t}}.
\end{split}
\end{equation*}
\end{lemma}
\begin{proof}
We have
\begin{equation}
\begin{split}
\bm{\Delta}_t - \frac{\beta_t\alpha_t}{\sqrt{\theta_t}\alpha_{t-1}} \bm{\Delta}_{t-1} =~& -\frac{\alpha_t\bm{m_t}}{\sqrt{\bm{v}_t}} + \frac{\beta_t\alpha_{t}\bm{m}_{t-1}}{\sqrt{\theta_t\bm{v}_{t-1}}}
= -\alpha_t \left(\frac{\bm{m}_t}{\sqrt{\bm{v}_t}} - \frac{\beta_t \bm{m}_{t-1}}{\sqrt{\theta_t \bm{v}_{t-1}}}\right) \\
=~& -\underbrace{\frac{(1-\beta_t)\alpha_t \bm{g}_t}{\sqrt{\bm{v}_t}}}_{\text{(I)}} + \underbrace{\beta_t \alpha_t \bm{m}_{t-1} \left( \frac{1}{\sqrt{\theta_t \bm{v}_{t-1}}} - \frac{1}{\sqrt{\bm{v}_t}}\right)}_{\text{(II)}}. 
\end{split}
\end{equation}
For (I) we have
\begin{equation}\label{1-010}
\begin{split}
\text{(I)} =~& \frac{(1-\beta_t)\alpha_t\bm{g}_t}{\sqrt{\bm{\hat{v}}}_t} + (1-\beta_t)\alpha_t\bm{g}_t\left(\frac{1}{\sqrt{\bm{v}}_t}- \frac{1}{\sqrt{\bm{\hat{v}}_t}}\right)\\
=~& (1-\beta_t)\bm{\hat{\eta}}_t\bm{g}_t + (1-\beta_t)\alpha_t\bm{g}_t\frac{(1-\theta_t)(\bm{\delta}_t^2-\bm{g}_t^2)}{\sqrt{\bm{v}_t}\sqrt{\bm{\hat{v}}_t}(\sqrt{\bm{v}_t}+\sqrt{\bm{\hat{v}}_t})}\\
=~& (1-\beta_t)\bm{\hat{\eta}}_t\bm{g}_t + \bm{\hat{\eta}}_t\bm{\delta}_t\frac{(1-\theta_t)\bm{g}_t}{\sqrt{\bm{v}_t}}\frac{(1-\beta_t)\bm{\delta}_t}{\sqrt{\bm{v}_t}+\sqrt{\bm{\hat{v}}_t}} -
\bm{\hat{\eta}}_t\bm{g}_t\frac{(1-\theta_t)\bm{g}_t}{\sqrt{\bm{v}_t}}\frac{(1-\beta_t)\bm{g}_t}{\sqrt{\bm{v}_t}+\sqrt{\bm{\hat{v}}}_t}.
\end{split}
\end{equation}
For (II) we have
\begin{equation}\label{1-011}
\begin{split}
\text{(II)} &= \beta_t \alpha_t \bm{m}_{t-1} \frac{(1-\theta_t)\bm{g}_t^2}{\sqrt{\bm{v}_t}\sqrt{\theta_t \bm{v}_{t-1}}(\sqrt{\bm{v}_t} + \sqrt{\theta_t \bm{v}_{t-1}})} \\
&= \beta_t \alpha_t \bm{m}_{t-1} \frac{(1-\theta_t)\bm{g}_t^2}{\sqrt{\bm{v}_t}\sqrt{\bm{\hat{v}}_t} (\sqrt{\bm{v}_t} + \sqrt{\theta_t \bm{v}_{t-1}})} 
+ \beta_t \alpha_t \bm{m}_{t-1} \frac{(1-\theta_t)\bm{g}_t^2}{\sqrt{\bm{v}_t}(\sqrt{\bm{v}_t} + \sqrt{\theta_t \bm{v}_{t-1}})}\left(\frac{1}{\sqrt{\theta_t\bm{v}_{t-1}}}-\frac{1}{\sqrt{\bm{\hat{v}}_t}}\right)\\
&= \bm{\hat{\eta}}_t \bm{g}_t \frac{(1-\theta_t)\bm{g}_t}{\sqrt{\bm{v}_t}}\left(\frac{\beta_t\bm{m}_{t-1}}{\sqrt{\bm{v}_{t}}+\sqrt{\theta_t \bm{v}_{t-1}}}\right) + \frac{\beta_t \alpha_t \bm{m}_{t-1}(1-\theta_t)^2\bm{g}_t^2\bm{\delta}_t^2}{\sqrt{\bm{v}_t}\sqrt{\bm{\hat{v}}_t}\sqrt{\theta_t \bm{v}_{t-1}}(\sqrt{\bm{v}_t} + \sqrt{\theta_t \bm{v}_{t-1}})(\sqrt{\bm{\hat{v}}_t}+\sqrt{\theta_t \bm{v}_{t-1}})} \\
&=  \bm{\hat{\eta}}_t \bm{g}_t \frac{(1-\theta_t)\bm{g}_t}{\sqrt{\bm{v}_t}}\left(\frac{\beta_t\bm{m}_{t-1}}{\sqrt{\bm{v}_{t}}+\sqrt{\theta_t \bm{v}_{t-1}}}\right)  +  \bm{\hat{\eta}}_t \bm{\delta}_t \frac{(1-\theta_t)\bm{g}_t}{\sqrt{\bm{v}_t}}\left(\frac{\beta_t\bm{m}_{t-1}}{\sqrt{\theta_t\bm{v}_{t-1}}}\frac{\sqrt{1-\theta_t}\bm{g}_t}{\sqrt{\bm{v}}_t + \sqrt{\theta_t\bm{v}_{t-1}}}\frac{\sqrt{1-\theta_t}\bm{\delta}_t}{\sqrt{\bm{\hat{v}}_t}+\sqrt{\theta_t\bm{v}_{t-1}}}\right).
\end{split}
\end{equation}
Combining Eq.~\eqref{1-010} and Eq.~\eqref{1-011}, we obtain the desired Eq.~\eqref{1-008}. The proof is completed.
\end{proof}



\medskip

\begin{lemma} \label{lem1-004}
Let $M_t = \mathbb{E} \left[\left\langle \bm{\nabla} f(\bm{x}_{t}), \bm{\Delta}_{t}\right\rangle + L\norm{\bm{\Delta}_{t}}^2\right]$ and $\chi_t = {\alpha_t}/{\sqrt{1-\theta_t}}$. Then for any $t\geq 2$, we have
\begin{equation}\label{2-012}
\begin{split}
M_t \leq & \frac{\beta_t\alpha_t}{\sqrt{\theta_t}\alpha_{t-1}} M_{t-1} + L\ \mathbb{E}\left[\norm{\bm{\Delta}_t}^2\right]
+ C_2G\chi_t \mathbb{E}\left[\norm{\frac{\sqrt{1-\theta_t}\bm{g}_t}{\sqrt{\bm{v}_t}}}^2\right]
- \frac{1-\beta}{2}\mathbb{E}\left[\norm{\bm{\nabla} f(\bm{x}_t)}^2_{\bm{\hat{\eta}}_t}\right]
\end{split}
\end{equation}
and
\begin{equation}\label{2-013}
M_1 \leq L\ \mathbb{E}\left[\norm{\bm{\Delta}_1}^2\right] + C_2G\chi_1 \mathbb{E}\left[\norm{\frac{\sqrt{1-\theta_t}\bm{g}_1}{\sqrt{v}_1}}^2\right],
\end{equation}
where $C_2 = 2\left(\frac{\beta/(1-\beta)}{\sqrt{C_1(1-\gamma)\theta_1}}+1\right)^2$.
\end{lemma}
\begin{proof}
First, for $t \geq 2$ we have
\begin{equation}\label{1-013}
\begin{split}
\mathbb{E}\langle \bm{\nabla} f(\bm{x}_t), \bm{\Delta}_t \rangle 
= \underbrace{\frac{\beta_t\alpha_t}{\sqrt{\theta_t}\alpha_{t-1}} \mathbb{E}\langle \bm{\nabla} f(\bm{x}_t), \bm{\Delta}_{t-1} \rangle}_{\text{(I)}} +
\underbrace{\mathbb{E}\left\langle \bm{\nabla} f(\bm{x}_t), \bm{\Delta}_t - \frac{\beta_t \alpha_t}{\sqrt{\theta_t}\alpha_{t-1}}\bm{\Delta}_{t-1}\right\rangle}_{\text{(II)}}.
\end{split}
\end{equation}
To estimate (I), by the Schwartz inequality and the Lipschitz continuity of the gradient of $f$, we have
\begin{equation}\label{1-015}
\begin{split}
\langle \bm{\nabla} f(\bm{x}_t), \bm{\Delta}_{t-1}\rangle 
\leq~&  \langle\bm{\nabla} f(\bm{x}_{t-1}), \bm{\Delta}_{t-1}\rangle + \langle \bm{\nabla} f(\bm{x}_t) - \bm{\nabla} f(\bm{x}_{t-1}), \bm{\Delta}_{t-1}\rangle \\
\leq~&  \langle\bm{\nabla} f(\bm{x}_{t-1}), \bm{\Delta}_{t-1}\rangle +  L \norm{\bm{x}_t - \bm{x}_{t-1}}\norm{\bm{\Delta}_{t-1}} \\
=~&  \langle\bm{\nabla} f(\bm{x}_{t-1}), \bm{\Delta}_{t-1}\rangle + L \norm{\bm{\Delta}_{t-1}}^2. \\
\end{split}
\end{equation}
Hence, we have
\begin{equation}
\text{(I)} \leq \frac{\beta_t\alpha_t}{\sqrt{\theta_t}\alpha_{t-1}} \mathbb{E}\left[\langle\bm{\nabla} f(\bm{x}_{t-1}), \bm{\Delta}_{t-1}\rangle + L \norm{\bm{\Delta}_{t-1}}^2\right] = \frac{\beta_t\alpha_t}{\sqrt{\theta_t}\alpha_{t-1}} M_{t-1}.
\end{equation}
To estimate (II), by Lemma \ref{lem1-002}, we have
\begin{equation}\label{1-016}
\begin{split}
& \mathbb{E}\left\langle \bm{\nabla} f(\bm{x}_t), \bm{\Delta}_t - \frac{\beta_t\alpha_t}{\sqrt{\theta_t}\alpha_{t-1}}\bm{\Delta}_{t-1}\right\rangle \\
=& -(1-\beta_t)\mathbb{E}\langle \bm{\nabla} f(\bm{x}_t), \bm{\hat{\eta}}_t\bm{g}_t \rangle
    \underbrace{- \mathbb{E}\left\langle \bm{\nabla} f(\bm{x}_t), \bm{\hat{\eta}}_t\bm{g}_t\frac{(1-\theta_t)\bm{g}_t}{\sqrt{\bm{v}}_t}\bm{A}_t \right\rangle}_{\text{(III)}} 
   \underbrace{- \mathbb{E}\left\langle \bm{\nabla} f(\bm{x}_t), \bm{\hat{\eta}}_t\bm{\delta}_t\frac{(1-\theta_t)\bm{g}_t}{\sqrt{\bm{v}}_t}\bm{B}_t \right\rangle}_{\text{(IV)}}.
\end{split}
\end{equation}
Note that $\bm{\hat{\eta}}_t$ is independent of $\bm{g}_t$ and $\mathbb{E}_t[\bm{g}_t] = \bm{\nabla} f(\bm{x}_t)$. Hence, for the first term in the right hand side of Eq.~\eqref{1-016}, we have
\begin{equation}\label{1-017}
\begin{split}
-(1-\beta_t)\mathbb{E}\langle \bm{\nabla} f(\bm{x}_t), \bm{\hat{\eta}}_t\bm{g}_t \rangle 
&= -(1-\beta_t)\mathbb{E}\left\langle \bm{\nabla} f(\bm{x}_t), \bm{\hat{\eta}}_t\mathbb{E}_t[\bm{g}_t] \right\rangle \\
&= -(1-\beta_t)\mathbb{E}\norm{\bm{\nabla} f(\bm{x}_t)}^2_{\bm{\hat{\eta}}_t}\\
&\leq -(1-\beta)\mathbb{E}\norm{\bm{\nabla} f(\bm{x}_t)}^2_{\bm{\hat{\eta}}_t}. 
\end{split}
\end{equation}
To estimate (III), we have
\begin{equation}\label{1-018}
\begin{split}
\text{(III)}~\leq \mathbb{E}\left\langle \frac{\sqrt{\bm{\hat{\eta}}_t}|\bm{\nabla} f(\bm{x}_t)||\bm{g}_t|}{\bm{\delta}_t}, \frac{\sqrt{\bm{\hat{\eta}}_t}\bm{\delta}_t|\bm{A}_t|(1-\theta_t)|\bm{g}_t|}{\sqrt{\bm{v}_t}}\right\rangle.
\end{split}
\end{equation}
Note that $\bm{\delta}_t \leq G$. Therefore,
\begin{equation}\label{1-023}
\sqrt{\bm{\hat{\eta}}_t} \bm{\delta}_t = \sqrt{\bm{\hat{\eta}}_t\bm{\delta}_t^2} = \sqrt{\frac{\alpha_t \bm{\delta}_t^2}{\sqrt{\bm{\hat{v}}_t}}}
\leq \sqrt{\frac{\alpha_t \bm{\delta}_t^2}{\sqrt{(1-\theta_t)\bm{\delta}_t^2}}} \leq \sqrt{\frac{G\alpha_t}{\sqrt{1-\theta_t}}} = \sqrt{G\chi_t}.
\end{equation}
On the other hand,
\begin{equation}
\begin{split}
|\bm{A}_t| 
= \left|\frac{\beta_t\bm{m}_{t-1}}{\sqrt{\bm{v}_t}+\sqrt{\theta_t\bm{v}_{t-1}}}+\frac{(1-\beta_t)\bm{g}_t}{\sqrt{\bm{v}}_t+\sqrt{\bm{\hat{v}}_t}}\right| 
\leq \frac{\beta_t|\bm{m}_{t-1}|}{\sqrt{\theta_t \bm{v}_{t-1}}} + \frac{(1-\beta_t)|\bm{g}_t|}{\sqrt{\bm{v}_t}}.
\end{split}
\end{equation}
By Lemma \ref{lem1-001}, we have 
\begin{equation}
\frac{|\bm{m}_{t-1}|}{\sqrt{\bm{v}_{t-1}}} \leq \frac{1}{\sqrt{C_1(1-\gamma)(1-\theta_t)}}.
\end{equation}
Meanwhile, 
\begin{equation} \frac{|\bm{g}_t|}{\sqrt{\bm{v}_t}} \leq \frac{|\bm{g}_t|}{\sqrt{(1-\theta_t)\bm{g}_t^2}} = \frac{1}{\sqrt{1-\theta_t}}.
\end{equation}
Hence, we have
\begin{equation}
\begin{split}
|\bm{A}_t| &\leq \frac{\beta_t}{\sqrt{C_1(1-\gamma)(1-\theta_t)\theta_t}} + \frac{1-\beta_t}{\sqrt{1-\theta_t}}
\leq \left(\frac{\beta_t/(1-\beta_t)}{\sqrt{C_1(1-\gamma)\theta_t}}+1\right)\frac{1-\beta_t}{\sqrt{1-\theta_t}} \\
&\leq \left(\frac{\beta/(1-\beta)}{\sqrt{C_1(1-\gamma)\theta_1}} + 1\right)\frac{1-\beta_t}{\sqrt{1-\theta_t}} := \frac{C_2'(1-\beta_t)}{\sqrt{1-\theta_t}},
\end{split}
\end{equation}
where $C_2'= \left(\frac{\beta/(1-\beta)}{\sqrt{C_1(1-\gamma)\theta_1}} + 1\right)$. The last inequality holds due to $\beta_t/(1-\beta_t) \leq \beta/(1-\beta)$ as $\beta_t \leq \beta$.
Therefore, we have
\begin{equation}\label{1-021}
\begin{split}
&\left\langle \frac{\sqrt{\bm{\hat{\eta}}_t}|\bm{\nabla} f(\bm{x}_t)||\bm{g}_t|}{\bm{\delta}_t}, \frac{\sqrt{\bm{\hat{\eta}}_t}\bm{\delta}_t|\bm{A}_t|(1-\theta_t)|\bm{g}_t|}{\sqrt{\bm{v}_t}}\right\rangle \\
\leq~& \left\langle \frac{\sqrt{\bm{\hat{\eta}}_t}|\bm{\nabla} f(\bm{x}_t)||\bm{g}_t|}{\bm{\delta}_t}, 
\sqrt{G\chi_t}C_2'(1-\beta_t)\frac{\sqrt{1-\theta_t}|\bm{g}_t|}{\sqrt{\bm{v}_t}}\right\rangle \\
\leq~& \frac{1-\beta_t}{4}\norm{\frac{\sqrt{\bm{\hat{\eta}}_t}|\bm{\nabla} f(\bm{x}_t)||\bm{g}_t|}{\bm{\delta}_t}}^2 +
{C_2'^2G}{(1-\beta_t)}\chi_t\norm{\frac{\sqrt{1-\theta_t}\bm{g}_t}{\sqrt{\bm{v}_t}}}^2 \\
\leq~& \frac{1-\beta_t}{4}\norm{\frac{\bm{\hat{\eta}}_t|\bm{\nabla} f(\bm{x}_t)|^2|\bm{g}_t|^2}{\bm{\delta}_t^2}}_1 +
{C_2'^2G}\chi_t\norm{\frac{\sqrt{1-\theta_t}\bm{g}_t}{\sqrt{\bm{v}_t}}}^2. 
\end{split}
\end{equation}
Note that $\bm{\delta}_t^2 = \mathbb{E}_t[\bm{g}_t^2]$. Hence,
\begin{equation}\label{1-022}
\mathbb{E}_t\norm{\frac{\bm{\hat{\eta}}_t|\bm{\nabla} f(\bm{x}_t)|^2|\bm{g}_t|^2}{\bm{\delta}_t^2}}_1
= \norm{\bm{\hat{\eta}}_t|\bm{\nabla} f(\bm{x}_t)|^2}_1 = \norm{\bm{\nabla} f(\bm{x}_t)}_{\bm{\hat{\eta}}_t}^2.
\end{equation}
Combining Eq.~\eqref{1-018}, Eq.~\eqref{1-021}, and Eq.~\eqref{1-022}, we obtain
\begin{equation}\label{1-024}
\begin{split}
\text{(III)} 
\leq~& \frac{1-\beta_t}{4}\mathbb{E}\left[\norm{\bm{\nabla} f(\bm{x}_t)}_{\bm{\hat{\eta}}_t}^2\right] + 
{C_2'^2G}\chi_t\mathbb{E}\norm{\frac{\sqrt{1-\theta_t}\bm{g}_t}{\sqrt{\bm{v}_t}}}^2.
\end{split}
\end{equation}
The term (IV) is estimated similarly as term (III). First, we have 
\begin{equation}
\begin{split}
|\bm{B}_t| \leq~& \left(\frac{\beta_t|\bm{m}_{t-1}|}{\sqrt{\theta_t\bm{v}_{t-1}}}\frac{\sqrt{1-\theta_t}|\bm{g}_t|}{\sqrt{\bm{v}}_t + \sqrt{\theta_t\bm{v}_{t-1}}}\frac{\sqrt{1-\theta_t}\bm{\delta}_t}{\sqrt{\bm{\hat{v}}_t}+\sqrt{\theta_t\bm{v}_{t-1}}}\right) 
+ \frac{(1-\beta_t)\bm{\delta}_t}{\sqrt{\bm{v}_t}+\sqrt{\bm{\hat{v}}_t}} \\
\leq~& \left(\frac{\beta/(1-\beta)}{\sqrt{C_1(1-\gamma)\theta_1}}+1\right)\frac{1-\beta_t}{\sqrt{1-\theta_t}} 
= \frac{C_2'(1-\beta_t)}{\sqrt{1-\theta_t}},
\end{split}
\end{equation}
where $C_2'$ is the constant defined above. We have
\begin{equation}\label{1-026}
\begin{split}
\text{(IV)} \leq~& \mathbb{E}\left\langle \sqrt{\bm{\hat{\eta}}_t}|\bm{\nabla} f(\bm{x}_t)|, \frac{\sqrt{\bm{\hat{\eta}}_t}\bm{\delta}_t|\bm{B}_t|(1-\theta_t)|\bm{g}_t|}{\sqrt{\bm{v}_t}}\right\rangle \\
\leq~& \mathbb{E}\left\langle \sqrt{\bm{\hat{\eta}}_t}|\bm{\nabla} f(\bm{x}_t)|, 
{\sqrt{G\chi_t}C_2'(1-\beta_t)}\frac{\sqrt{1-\theta_t}|\bm{g}_t|}{\sqrt{\bm{v}_t}}\right\rangle \\
\leq~& \frac{1-\beta_t}{4}\mathbb{E}\left[\norm{\bm{\nabla} f(\bm{x}_t)}_{\bm{\hat{\eta}}_t}^2\right] + {C_2'^2G}{\chi_t}\mathbb{E}\norm{\frac{\sqrt{1-\theta_t}\bm{g}_t}{\sqrt{\bm{v}_t}}}^2. \\
\end{split}
\end{equation}
Combining Eq.~\eqref{1-013}, Eq.~\eqref{1-015}, Eq.~\eqref{1-016}, Eq.~\eqref{1-017}, Eq.~\eqref{1-024}, and Eq.~\eqref{1-026}, 
we obtain
\begin{equation}\label{1-030}
\begin{split}
\mathbb{E}\langle \bm{\nabla} f(\bm{x}_t), \bm{\Delta}_t \rangle 
&~\leq \frac{\beta_t\alpha_t}{\sqrt{\theta_t}\alpha_{t-1}} M_{t-1} 
+ {2C_2'^2G}{\chi_t}\mathbb{E}\norm{\frac{\sqrt{1-\theta_t}\bm{g}_t}{\sqrt{\bm{v}_t}}}^2
- \frac{1-\beta_t}{2}\mathbb{E}\left[\norm{\bm{\nabla} f(\bm{x}_t)}_{\bm{\hat{\eta}}_t}^2\right]\\
&~\leq \frac{\beta_t\alpha_t}{\sqrt{\theta_t}\alpha_{t-1}} M_{t-1} 
+ {2C_2'^2G}{\chi_t}\mathbb{E}\norm{\frac{\sqrt{1-\theta_t}\bm{g}_t}{\sqrt{\bm{v}_t}}}^2
- \frac{1-\beta}{2}\mathbb{E}\left[\norm{\bm{\nabla} f(\bm{x}_t)}_{\bm{\hat{\eta}}_t}^2\right].
\end{split}
\end{equation}
Let $C_2$ denote the constant ${2(C_2')^2}$. Then $C_2 = 2 \left(\frac{\beta/(1-\beta)}{\sqrt{C_1(1-\gamma)\theta_1}}+1\right)^2$.
Thus, we obtain Eq.~\eqref{2-012} by adding the term $L\mathbb{E}\left[\norm{\bm{\Delta}_t}^2\right]$ to both sides of Eq.~\eqref{1-030}. 
Next, we estimate Eq.~\eqref{2-013}. When $t=1$, we have
\begin{equation}\label{2-031}
\begin{split}
M_1 =~& \mathbb{E}\left[-\left\langle \bm{\nabla} f(\bm{x}_1), \frac{\alpha_1\bm{m}_1}{\sqrt{\bm{v}_1}}\right\rangle + L \norm{\bm{\Delta}_1}^2\right] 
= \mathbb{E}\left[-\left\langle \bm{\nabla} f(\bm{x}_1), \frac{\alpha_1(1-\beta_1)\bm{g}_1}{\sqrt{\bm{v}_1}}\right\rangle + L \norm{\bm{\Delta}_1}^2\right].
\end{split}
\end{equation}
The same as what we did for term (I) in Lemma \ref{lem1-002}, we have
\begin{equation}
\frac{(1-\beta_1)\alpha_1\bm{g}_1}{\sqrt{\bm{v}_t}} = (1-\beta_1)\bm{\hat{\eta}}_1\bm{g}_1 + \bm{\hat{\eta}}_1\bm{\delta}_1\frac{(1-\theta_1)\bm{g}_1}{\sqrt{\bm{v}_1}}\frac{(1-\beta_1)\bm{\delta}_1}{\sqrt{\bm{v}_1}+\sqrt{\bm{\hat{v}}_1}} -
\bm{\hat{\eta}}_1\bm{g}_1\frac{(1-\theta_1)\bm{g}_1}{\sqrt{\bm{v}_1}}\frac{(1-\beta_1)\bm{g}_1}{\sqrt{\bm{v}_1}+\sqrt{\bm{\hat{v}}}_1}.
\end{equation}
Then the similar argument as Eq.~\eqref{1-021} implies that
\begin{equation}\label{2-033}
\begin{split}
\mathbb{E}\left[-\left\langle\bm{\nabla} f(\bm{x}_1), \frac{\alpha_1\bm{m}_1}{\sqrt{\bm{v}_1}}\right\rangle\right]
\leq~& C_2G \chi_1\mathbb{E}\left[\norm{\frac{\sqrt{1-\theta_t}\bm{g}_1}{\sqrt{\bm{v}_1}}}^2\right] -\frac{1-\beta_1}{2}\mathbb{E}\left[\norm{\bm{\nabla} f(\bm{x}_1)}_{\bm{\hat{\eta}}_1}^2\right]\\
\leq~& C_2G \chi_1\mathbb{E}\left[\norm{\frac{\sqrt{1-\theta_t}\bm{g}_1}{\sqrt{\bm{v}}_1}}^2\right].
\end{split}
\end{equation}
Combining Eq.~\eqref{2-031} and Eq.~\eqref{2-033}, and adding both sides by $L\mathbb{E}\left[\norm{\bm{\Delta}}_1^2\right]$, we obtain Eq.~\eqref{2-013}. This completes the proof.
\end{proof}

\medskip

\begin{lemma}\label{lem1-005}
The following estimate holds
\begin{equation}
\sum_{t=1}^T \norm{\bm{\Delta}_t}^2 \leq \frac{C_0^2\chi_1}{C_1(1-\sqrt{\gamma})^2}\sum_{t=1}^T\chi_t \norm{\frac{\sqrt{1-\theta_t}\bm{g}_t}{\sqrt{\bm{v}_t}}}^2.
\end{equation}
\end{lemma}

\begin{proof}
Note that $\bm{v}_t \geq \theta_t \bm{v}_{t-1}$, so we have $\bm{v}_t \geq \left(\prod_{j=i+1}^t \theta_j\right) \bm{v}_{i} = \Theta_{(t,i)}\bm{v}_i$. By Lemma \ref{lem2-001}, this follows that $\bm{v}_t \geq C_1 (\theta')^{t-i}\bm{v}_{i}$ for all $i \leq t$. On the other hand, 
\[ |\bm{m}_t| \leq \sum_{i=1}^t \left(\prod_{j=i+1}^t \beta_j\right)(1-\beta_i)|\bm{g}_i| 
\leq \sum_{i=1}^t \beta^{t-i}|\bm{g}_i|. \]
It follows that
\begin{equation}
\begin{split}
\frac{|\bm{m}_t|}{\sqrt{\bm{v}_t}} 
\leq \sum_{i=1}^t\frac{\beta^{t-i}|\bm{g}_i|}{\sqrt{\bm{v}_t}}
%
%
%
\leq \frac{1}{\sqrt{C_1}}\sum_{i=1}^t \left(\frac{\beta}{\sqrt{\theta'}}\right)^{t-i} \frac{|\bm{g}_i|}{\sqrt{\bm{v}_i}}
= \frac{1}{\sqrt{C_1}}\sum_{i=1}^t \sqrt{\gamma}^{t-i} \frac{|\bm{g}_i|}{\sqrt{\bm{v}_i}}.
\end{split}
\end{equation}
Since $\alpha_t = \chi_t\sqrt{1-\theta_t} \leq \chi_t \sqrt{1-\theta_i}$ for $i\leq t$, it follows that
\begin{equation}
\begin{split}
\norm{\bm{\Delta}_t}^2 = \norm{\frac{\alpha_t\bm{m}_t}{\sqrt{\bm{v}_t}}}^2
%
%
%
\leq~& \frac{\chi_t^2}{C_1}\norm{\sum_{i=1}^t \sqrt{\gamma}^{t-i}\frac{\sqrt{1-\theta_i}|\bm{g_i}|}{\sqrt{\bm{v}_i}}}^2 
\leq \frac{\chi_t^2}{C_1}\left(\sum_{i=1}^t \sqrt{\gamma}^{t-i} \right)\sum_{i=1}^t \sqrt{\gamma}^{t-i}\norm{\frac{\sqrt{1-\theta_i}\bm{g}_i}{\sqrt{\bm{v}_i}}}^2 \\
\leq~& \frac{\chi_t^2}{C_1(1-\sqrt{\gamma})}\sum_{i=1}^t \sqrt{\gamma}^{t-i}\norm{\frac{\sqrt{1-\theta_i}\bm{g}_i}{\sqrt{\bm{v}_i}}}^2.
\end{split}
\end{equation}
By Lemma \ref{lem2-002}, 
$ \chi_t \leq C_0\chi_i, \forall i \leq t$. 
Hence,
\begin{equation}
\begin{split}
\norm{\bm{\Delta}_t}^2 = \norm{\frac{\alpha_t\bm{m}_t}{\sqrt{\bm{v}_t}}}^2 \leq \frac{C_0^2\chi_1}{C_1(1-\sqrt{\gamma})}\sum_{i=1}^t \sqrt{\gamma}^{t-i}\chi_i\norm{\frac{\sqrt{1-\theta_i}\bm{g}_i}{\sqrt{\bm{v}_i}}}^2.
\end{split}
\end{equation}
It follows that
\begin{equation}
\begin{split}
\sum_{t=1}^T \norm{\bm{\Delta}_t}^2 
\leq~& \frac{C_0^2\chi_1}{C_1(1-\sqrt{\gamma})}\sum_{t=1}^T\sum_{i=1}^t \sqrt{\gamma}^{t-i}\chi_i\norm{\frac{\sqrt{1-\theta_i}\bm{g}_i}{\sqrt{\bm{v}_i}}}^2 \\
=~& \frac{C_0^2\chi_1}{C_1(1-\sqrt{\gamma})}\sum_{i=1}^T\left(\sum_{t=i}^T \sqrt{\gamma}^{t-i}\right)\chi_i\norm{\frac{\sqrt{1-\theta_i}\bm{g}_i}{\sqrt{\bm{v}_i}}}^2 \\
\leq~& \frac{C_0^2\chi_1}{C_1(1-\sqrt{\gamma})^2}\sum_{i=1}^T\chi_i\norm{\frac{\sqrt{1-\theta_i}\bm{g}_i}{\sqrt{\bm{v}_i}}}^2. 
\end{split}
\end{equation}
The proof is completed.
\end{proof}

\medskip

\begin{lemma}\label{lem1-006}
(Lemma \ref{lemma2}  in Section \ref{sufficient_condition})
Let $M_t = \mathbb{E} \left[\langle \bm{\nabla} f(\bm{x}_{t}), \bm{\Delta}_{t}\rangle + L\norm{\bm{\Delta}_{t}}^2\right]$. For $T\geq 1$ we have
\begin{equation}\label{1-037}
\begin{split}
\sum_{t=1}^T M_t 
%
\leq C_3\mathbb{E}\left[\sum_{t=1}^T\chi_t\norm{\frac{\sqrt{1-\theta_t}\bm{g}_t}{\sqrt{\bm{v}_t}}}^2\right]
- \frac{1-\beta}{2}\mathbb{E}\left[\sum_{t=1}^T \norm{\bm{\nabla} f(\bm{x}_t)}_{\bm{\hat{\eta}_t}}^2\right].
\end{split}
\end{equation}
where the constant $C_3$ is given by
\[ C_3= \frac{C_0}{\sqrt{C_1}(1-\sqrt{\gamma})}\left(\frac{C_0^2\chi_1L}{C_1(1-\sqrt{\gamma})^2} + 2\left(\frac{\beta/(1-\beta)}{\sqrt{C_1(1-\gamma)\theta_1}}+1\right)^2G\right).\]
\end{lemma}
\begin{proof}
Let $N_t = L\mathbb{E}\left[\norm{\bm{\Delta}_t}^2\right] + C_2G \chi_t\mathbb{E}\left[\norm{\frac{\sqrt{1-\theta_t}\bm{g}_t}{\sqrt{\bm{v}_t}}}^2\right]$. 
By Lemma \ref{lem1-004}, we have 
$M_1 \leq N_1$ 
and
\begin{equation}
M_t \leq \frac{\beta_t\alpha_t}{\sqrt{\theta_t}\alpha_{t-1}}M_{t-1} + N_t - \frac{1-\beta}{2}\mathbb{E}\left[\norm{\bm{\nabla} f(\bm{x}_t)}_{\bm{\hat{\eta}}_t}^2\right] \leq \frac{\beta_t\alpha_t}{\sqrt{\theta_t}\alpha_{t-1}}M_{t-1} + N_t.
\end{equation}
It is straightforward to acquire by induction that
\begin{equation}
\begin{split}
M_t &\leq \frac{\beta_t\alpha_t}{\sqrt{\theta_t}\alpha_{t-1}}\frac{\beta_{t-1}\alpha_{t-1}}{\sqrt{\theta_{t-1}}\alpha_{t-2}}M_{t-2} + \frac{\beta_t\alpha_t}{\sqrt{\theta_t}\alpha_{t-1}}N_{t-1} + N_t - \frac{1-\beta}{2}\mathbb{E}\left[\norm{\bm{\nabla} f(\bm{x}_t)}_{\bm{\hat{\eta}}_t}^2\right]\\
&~ \vdots \\
&\leq \frac{\alpha_tB_{(t, 1)}}{\alpha_1\sqrt{\Theta_{(t,1)}}}M_1 + \sum_{i=2}^t \frac{\alpha_tB_{(t,i)}}{\alpha_i\sqrt{\Theta_{(t,i)}}}N_i - \frac{1-\beta}{2}\mathbb{E}\left[\norm{\bm{\nabla} f(\bm{x}_t)}_{\bm{\hat{\eta}}_t}^2\right]\\
&\leq \sum_{i=1}^t \frac{\alpha_tB_{(t,i)}}{\alpha_i\sqrt{\Theta_{(t,i)}}}N_i - \frac{1-\beta}{2}\mathbb{E}\left[\norm{\bm{\nabla} f(\bm{x}_t)}_{\bm{\hat{\eta}}_t}^2\right].
\end{split}
\end{equation}
By Lemma \ref{lem2-002}, it holds $\alpha_t \leq C_0\alpha_i$ for any $i \leq t$. By Lemma \ref{lem2-001}, $\Theta_{(t,i)} \geq C_1(\theta')^{t-i}$. In addition, $B_{(t,i)}\leq \beta^{t-i}$. Hence,
\begin{equation}
\begin{split}
M_t \leq~& \frac{C_0}{\sqrt{C_1}}\sum_{i=1}^t \left(\frac{\beta}{\sqrt{\theta'}}\right)^{t-i}N_i - \frac{1-\beta}{2}\mathbb{E}\left[\norm{\bm{\nabla} f(\bm{x}_t)}^2_{\bm{\hat{\eta}}_t}\right] \\
=~& \frac{C_0}{\sqrt{C_1}}\sum_{i=1}^t \sqrt{\gamma}^{t-i}N_i - \frac{1-\beta}{2}\mathbb{E}\left[\norm{\bm{\nabla} f(\bm{x}_t)}^2_{\bm{\hat{\eta}}_t}\right].
\end{split}
\end{equation}
Hence,
\begin{equation}\label{1-042}
\begin{split}
\sum_{t=1}^T M_t 
\leq~& \frac{C_0}{\sqrt{C_1}} \sum_{t=1}^T\sum_{i=1}^t\sqrt{\gamma}^{t-i}N_i - \frac{1-\beta}{2}\mathbb{E}\left[\sum_{t=1}^T\norm{\bm{\nabla}f(\bm{x}_t)}^2_{\bm{\hat{\eta}}_t}\right] \\
=~& \frac{C_0}{\sqrt{C_1}}\sum_{i=1}^T \left(\sum_{t=i}^T\sqrt{\gamma}^{t-i}\right)N_i - \frac{1-\beta}{2}\mathbb{E}\left[\sum_{t=1}^T\norm{\bm{\nabla}f(\bm{x}_t)}^2_{\bm{\hat{\eta}}_t}\right]\\
=~& \frac{C_0}{\sqrt{C_1}(1-\sqrt{\gamma})}\sum_{t=1}^T N_t - \frac{1-\beta}{2}\mathbb{E}\left[\sum_{t=1}^T\norm{\bm{\nabla}f(\bm{x}_t)}^2_{\bm{\hat{\eta}}_t}\right].
\end{split}
\end{equation}
Finally, by Lemma \ref{lem1-005}, we have
\begin{equation}\label{1-043}
\begin{split}
\sum_{t=1}^T N_i =~& \mathbb{E}\left[L\sum_{t=1}^T\norm{\bm{\Delta}_t}^2 + C_2G\sum_{t=1}^T\chi_t\norm{\frac{\sqrt{1-\theta_t}\bm{g}_t}{\sqrt{\bm{v}_t}}}^2\right]\\
\leq~& \left(\frac{C_0^2\chi_1L}{C_1(1-\sqrt{\gamma})^2}+C_2G\right)\mathbb{E}\left[\sum_{t=1}^T\chi_t\norm{\frac{\sqrt{1-\theta_t}\bm{g}_t}{\sqrt{\bm{v}_t}}}^2\right].
\end{split}
\end{equation}
Let
\[\begin{split} 
C_3 =~& \frac{C_0}{\sqrt{C_1}(1-\sqrt{\gamma})}\left(\frac{C_0^2\chi_1L}{C_1(1-\sqrt{\gamma})^2} + C_2 G\right)\\
=~& \frac{C_0}{\sqrt{C_1}(1-\sqrt{\gamma})}\left(\frac{C_0^2\chi_1L}{C_1(1-\sqrt{\gamma})^2} + 2\left(\frac{\beta/(1-\beta)}{\sqrt{C_1(1-\gamma)\theta_1}}+1\right)^2G\right).
\end{split}\]
Combining Eq.~\eqref{1-042} and Eq.~\eqref{1-043}, we then obtain the desired estimate Eq.~\eqref{1-037}. The proof is completed.
\end{proof}

\medskip

\begin{lemma}\label{lem1-007}
The following estimate holds
\end{lemma}
\begin{equation}\label{1-044}
\mathbb{E}\left[\sum_{i=1}^t\norm{\frac{\sqrt{1-\theta_i}\bm{g}_i}{\sqrt{\bm{v}_i}}}^2\right]
\leq d\left[\log\left(1 + \frac{G^2}{\epsilon d}\right) + \sum_{i=1}^t\log(\theta_i^{-1})\right].
\end{equation}
\begin{proof}
Let $W_0 = 1$ and $W_t = \prod_{i=1}^T \theta_i^{-1}$. Let $w_t = W_t - W_{t-1} = (1 - \theta_t)\prod_{i=1}^{t}\theta_i^{-1} = (1-\theta_t)W_t$. We therefore have 
\[ \frac{w_t}{W_t} = 1-\theta_t, \quad \frac{W_{t-1}}{W_t} = \theta_t.\]
Note that $\bm{v}_0 = \bm{\epsilon}$ and $\bm{v}_t = \theta_t \bm{v}_{t-1} + (1-\theta_t)\bm{g}_t$, so it holds that $W_0\bm{v}_0 = \bm{\epsilon}$ and
$
W_t\bm{v}_t = W_{t-1}\bm{v}_{t-1} + w_t\bm{g}_t^2.
$
Then,
$
W_t\bm{v}_t = W_0\bm{v}_0 + \sum_{i=1}^t w_i\bm{g}_i^2 = \bm{\epsilon} + \sum_{i=1}^t w_i\bm{g}_i^2.
$
It follows that
\begin{equation}\label{1-049}
\begin{split}
\sum_{i=1}^t\norm{\frac{\sqrt{1-\theta_i}\bm{g}_i}{\sqrt{\bm{v}_i}}}^2 
=~& \sum_{i=1}^t \norm{\frac{(1-\theta_i)\bm{g}_t^2}{\bm{v}_i}}_1 
= \sum_{i=1}^t \norm{\frac{w_i\bm{g}_i^2}{W_i\bm{v}_i}}_1 
= \sum_{i=1}^t \norm{\frac{w_i\bm{g}_i^2}{\bm{\epsilon} + \sum_{\ell=1}^i w_\ell\bm{g}_\ell^2}}_1.
\end{split}
\end{equation}
Writing the norm in terms of coordinates, we obtain
\begin{equation}
\sum_{i=1}^t \norm{\frac{\sqrt{1-\theta_i}\bm{g}_i}{\sqrt{\bm{v}_i}}}^2 
= \sum_{i=1}^t \sum_{k=1}^d \frac{w_i g_{i,k}^2}{ {\epsilon} + \sum_{\ell=1}^i w_\ell g_{\ell,k}^2} 
= \sum_{k=1}^d\sum_{i=1}^t \frac{w_i g_{i,k}^2}{ {\epsilon} + \sum_{\ell=1}^i w_\ell g_{\ell,k}^2}.
\end{equation}
By Lemma \ref{lem2-001}, for each $k = 1,2,\ldots,d$,
\begin{equation}
\sum_{i=1}^t \frac{w_i g_{i,k}^2}{\epsilon + \sum_{\ell=1}^i w_\ell g_{\ell,k}^2} \leq \log\left(\epsilon + \sum_{\ell=1}^t w_\ell g_{\ell,k}^2\right) - \log(\epsilon) = \log\left(1 + \frac{1}{\epsilon}\sum_{\ell=1}^t w_\ell g_{\ell,k}^2\right).
\end{equation}
Hence, 
\begin{equation}\label{2-050}
\begin{split}
\sum_{i=1}^t\norm{\frac{\sqrt{1-\theta_i}\bm{g}_i}{\sqrt{\bm{v}_i}}}^2 
\leq~& \sum_{k=1}^d \log\left(1 + \frac{1}{\epsilon}\sum_{i=1}^t w_i g_{i,k}^2\right) \\
\leq~& d\log\left(\frac{1}{d}\sum_{k=1}^d\left(1 + \frac{1}{\epsilon}\sum_{i=1}^t w_i  g_{i,k}^2\right)\right)
= d\log\left(1 + \frac{1}{\epsilon d}\sum_{i=1}^t w_i\norm{\bm{g}_i}^2\right).
\end{split}
\end{equation}
The second inequality is due to the convex inequality $\frac{1}{d}\sum_{k=1}^d\log\left(z_i\right) \leq \log\left(\frac{1}{d}\sum_{k=1}^d z_i\right)$. Indeed, we have the more general convex inequality that $\mathbb{E}[\log(X)] \leq \log{\mathbb{E}[X]}$,
for any positive random variable $X$. Taking $X$ to be $1 + \frac{1}{\epsilon d}\sum_{i=1}^t w_i \norm{\bm{g}_i}^2$ in the right hand side of Eq.~\eqref{2-050}, we obtain that
\begin{equation}
\begin{split}
&\mathbb{E}\left[\sum_{i=1}^t\norm{\frac{\sqrt{1-\theta_i}\bm{g}_i}{\sqrt{\bm{v}_i}}}^2\right]
\leq d\ \mathbb{E}\left[\log\left(1 + \frac{1}{\epsilon d}\sum_{i=1}^t w_i \norm{\bm{g}_i}^2\right)\right] 
\leq d\log\left(1 + \frac{1}{\epsilon d}\sum_{i=1}^t w_i \mathbb{E}\left[\norm{\bm{g}_i}^2\right]\right)\\
\leq~& d\log\left(1 + \frac{G^2}{\epsilon d}\sum_{i=1}^t w_i\right)
= d\log\left(1 + \frac{G^2}{\epsilon d}(W_t-W_0)\right)
= d\log\left(1 + \frac{G^2}{\epsilon d}\left(\prod_{i=1}^t\theta_i^{-1} -1\right)\right)\\
\leq~& d\left[\log\left(1+\frac{G^2}{\epsilon d}\right) + \log\left(\prod_{i=1}^t\theta_i^{-1}\right)\right].
\end{split}
\end{equation}
The last inequality is due to the following trivial inequality:
\[\log(1+ ab) \leq \log(1+a+b+ab) = \log(1+a) + \log(1+b) \]
for any non-negative parameters $a$ and $b$. It then follows that 
\begin{equation}
\mathbb{E}\left[\sum_{i=1}^t\norm{\frac{\sqrt{1-\theta_i}\bm{g}_i}{\sqrt{\bm{v}_i}}}^2\right] \leq d\left[\log\left(1+\frac{G^2}{\epsilon d}\right) + \sum_{i=1}^t\log(\theta_i^{-1})\right].
\end{equation}
The proof is completed.
\end{proof}

\medskip

\begin{lemma}
We have the following estimate
\begin{equation}
\mathbb{E}\left[\sum_{t=1}^T \chi_t \norm{\frac{\sqrt{1-\theta_t}\bm{g}_t}{\sqrt{\bm{v}}_t}}^2\right] \leq C_0d\left[\chi_1\log\left(1 + \frac{G^2}{\epsilon d}\right) + \frac{1}{\theta_1}\sum_{t=1}^T \alpha_t\sqrt{1-\theta_t}\right].
\end{equation}
\end{lemma}

\begin{proof}
For simplicity of notations, let $\omega_t := \norm{\frac{\sqrt{1-\theta_t}\bm{g}_t}{\sqrt{\bm{v}_t}}}^2$, and $\Omega_t := \sum_{i=1}^t \omega_i$. Note that $\chi_t \leq C_0 a_t$. Hence, 
\begin{equation}\label{3-059}
\mathbb{E}\left[\sum_{t=1}^T \chi_t \norm{\frac{\sqrt{1-\theta_t}\bm{g}_t}{\sqrt{\bm{v}}_t}}^2\right]
\leq C_0\ \mathbb{E}\left[\sum_{t=1}^T a_t \omega_t\right].
\end{equation}
By Lemma \ref{lem2-003}, we have
\begin{equation}\label{3-060}
\mathbb{E}\left[\sum_{t=1}^T a_t\omega_t\right] = \mathbb{E}\left[\sum_{t=1}^{T-1} (a_t - a_{t+1}) \Omega_t + a_T\Omega_T\right].
\end{equation}
Let $S_t := \log\left(1 + \frac{G^2}{\epsilon d}\right) + \sum_{i=1}^t \log(\theta_i^{-1})$. By Lemma \ref{lem1-007}, we have
$\mathbb{E}[\Omega_t] \leq d S_t$.
Since $\{a_t\}$ is a non-increasing sequence, we have $a_t - a_{t+1} \geq 0$.  By Eq.~\eqref{3-060}, we have
\begin{equation}\label{3-061}
\begin{split}
& \mathbb{E}\left[\sum_{t=1}^{T-1} (a_t - a_{t+1}) \Omega_t + a_T\Omega_T\right]
\leq d\left(\sum_{t=1}^{T-1}(a_t - a_{t+1})S_t + a_T S_T\right) \\
=& d\left(a_1S_0 + \sum_{t=1}^T a_t(S_t - S_{t-1})\right) 
= d\left[a_1\log\left(1 + \frac{G^2}{\epsilon d}\right) + \sum_{t=1}^T a_t\log(\theta_t^{-1})\right].
\end{split}
\end{equation}
Note that $a_t \leq \chi_t$. Combining Eq.~\eqref{3-059}, Eq.~\eqref{3-060}, and Eq.~\eqref{3-061}, we have 
\begin{equation}
\begin{split}
\mathbb{E}\left[\sum_{t=1}^T \chi_t \norm{\frac{\sqrt{1-\theta_t}\bm{g}_t}{\sqrt{\bm{v}}_t}}^2\right] 
\leq~& C_0 d \left[\chi_1\log\left(1 + \frac{G^2}{\epsilon d}\right) + \sum_{t=1}^T \chi_t\log(\theta_t^{-1})\right] \\
=~& C_0 d \left[ \chi_1\log\left(1 + \frac{G^2}{\epsilon d}\right) + \sum_{t=1}^T \chi_t\log(\theta_t^{-1})\right].
\end{split}
\end{equation}
Note that $\log(1+x) \leq x$ for all $x > -1$. It follows that
$$\log(\theta_t^{-1}) = \log(1 + (\theta_t^{-1} - 1)) \leq \theta_t^{-1} - 1 \leq \frac{1-\theta_t}{\theta_1}.$$
Note that $\chi_t = \alpha_t/\sqrt{1-\theta_t}$. By Eq.~\eqref{3-059} and Eq.~\eqref{3-061}, we have
\begin{equation}
\mathbb{E}\left[\sum_{t=1}^T \chi_t \norm{\frac{\sqrt{1-\theta_t}\bm{g}_t}{\sqrt{\bm{v}}_t}}^2\right]
\leq C_0d\left[\chi_1\log\left(1+\frac{G^2}{\epsilon d}\right) - \frac{1}{\theta_1}\sum_{t=1}^T \alpha_t \sqrt{1-\theta_t}\right]. 
\end{equation}
The proof is completed.
\end{proof}

\medskip

\begin{lemma}\label{lem1-008} (Lemma \ref{lemma3}  in Section \ref{sufficient_condition})
Let $\tau$ be randomly chosen from $\{1, 2, \ldots, T\}$ with equal probabilities $p_\tau = 1/T$. We have the following estimate
\begin{equation}
    \mathbb{E}\left[\|\nabla f(x_\tau)\|\right] \leq \sqrt{\frac{C_0\sqrt{G^2 +\epsilon d}}{T\alpha_T}\ \mathbb{E}\left[\sum_{t=1}^T\norm{\bm{\nabla} f(\bm{x}_t)}_{\bm{\hat{\eta}}_t}^2\right]}.
\end{equation}
\end{lemma}
\begin{proof}
For any two random variables $X$ and $Y$, by the H\"{o}lder's inequality, we have
\begin{equation}\label{1-053}
\mathbb{E}[|XY|] \leq \mathbb{E}\left[|X|^p\right]^{1/p} \mathbb{E}\left[|Y|^q\right]^{1/q}.
\end{equation}
Let $X = \left(\frac{\norm{\bm{\nabla} f(\bm{x}_t)}^2}{\sqrt{\norm{\bm{\hat{v}}_t}_1}}\right)^{1/2}$, $Y = \norm{\bm{\hat{v}}_t}_1^{1/4}$, and let $p = 2$, $q = 2$. By Eq.~\eqref{1-053}, we have
\begin{equation}\label{1-054}
\mathbb{E}\left[\norm{\bm{\nabla} f(\bm{x}_t)}\right]^2 \leq \mathbb{E}\left[\frac{\norm{\bm{\nabla} f(\bm{x}_t)}^2}{\sqrt{\norm{\bm{\hat{v}}_t}_1}}\right] \mathbb{E}\left[\sqrt{\norm{\bm{\hat{v}}_t}_1} \right].
\end{equation}
On the one hand, we have
\begin{equation}\label{1-055}
\begin{split}
\frac{\norm{\bm{\nabla} f(\bm{x}_t)}^2}{\sqrt{\norm{\bm{\hat{v}}_t}_1}}
= \sum_{k=1}^d \frac{|\nabla_k f(\bm{x}_t)|^2}{\sqrt{\sum_{k=1}^d \hat{v}_{t,k}} }
\leq~& \alpha_t^{-1} \sum_{k=1}^d \frac{\alpha_t}{\sqrt{\hat{v}_{t,k}}}|\nabla_k f(\bm{x}_t)|^2 \\
=~& \alpha_t^{-1} \sum_{k=1}^d \hat{\eta}_{t,k}|\nabla_k f(\bm{x}_t)|^2
= \alpha_t^{-1}\norm{\bm{\nabla} f(\bm{x}_t)}_{\bm{\hat{\eta}_t}}^2.
\end{split}
\end{equation}
Since $\bm{\hat{v}}_t = \theta_t \bm{{v}}_{t-1} + (1-\theta_t) \bm{\delta}_t^2$, and all entries are non-negative, we have $\norm{\bm{\hat{v}}_t}_1 = \theta_t \norm{\bm{v}_{t-1}}_1 + (1-\theta_t)\norm{\bm{\delta}_t}^2$. 
Notice that $\bm{v}_t = \theta_t \bm{v}_{t-1} + (1 - \theta_t)\bm{g}_t^2$, $\bm{v}_0 = \bm{\epsilon}$, and $\mathbb{E}_t\left[\bm{g}_t^2\right] \leq G^2$. It is straightforward to prove by induction that $\mathbb{E}[\norm{\bm{v}_t}_1] \leq G^2 + \epsilon d$. Hence,
\begin{equation}\label{1-056}
\mathbb{E}\left[\sqrt{\norm{\bm{\hat{v}}_t}_1}\right] \leq \sqrt{\mathbb{E}[\norm{\bm{\hat{v}}_t}_1]} \leq \sqrt{G^2 + \epsilon d}.
\end{equation}
By Eq.~\eqref{1-054}, Eq.~\eqref{1-055}, and Eq.~\eqref{1-056}, we obtain
\begin{equation}
\mathbb{E}\left[\norm{\bm{\nabla} f(\bm{x}_t)}\right]^2 \leq \left(\alpha_t^{-1}\mathbb{E}\left[\norm{\bm{\nabla} f(\bm{x}_t)}_{\bm{\hat{\eta}_t}}^2\right]\right) \sqrt{G^2 + \epsilon d}.
\end{equation}
By Lemma \ref{lem2-002}, $\alpha_T \leq C_0 \alpha_t$ for any $t\leq T$, so $\alpha_t^{-1} \leq C_0\alpha_T^{-1}$. Then, we obtain
\begin{equation}
\mathbb{E}\left[\norm{\bm{\nabla} f(\bm{x}_t)}\right]^{2} \leq \frac{C_0\sqrt{G^2+\epsilon d}}{\alpha_T}\mathbb{E}\left[\norm{\bm{\nabla} f(\bm{x}_t)}_{\bm{\hat{\eta}_t}}^2\right],~\forall t\leq T.
\end{equation}
The lemma is followed by
\begin{equation}\begin{split}
\left(\mathbb{E}\left[\norm{\bm{\nabla} f(\bm{x}_\tau)}\right]\right) =~& \frac{1}{T}\sum_{t=1}^T \mathbb{E}\left[\norm{\bm{\nabla} f(\bm{x}_t)}\right] \\
\leq~& \frac{1}{T}\sum_{t=1}^T \sqrt{\frac{C_0\sqrt{G^2+\epsilon d}}{\alpha_T}\ \mathbb{E}\left[ \norm{\bm{\nabla} f(\bm{x}_t)}_{\bm{\hat{\eta}_t}}^2\right]}\\
\leq ~& \sqrt{\frac{C_0\sqrt{G^2+\epsilon d}}{T\alpha_T}\ \mathbb{E}\left[\sum_{t=1}^T \norm{\bm{\nabla} f(\bm{x}_t)}_{\bm{\hat{\eta}_t}}^2\right]}.
\end{split}\end{equation}
The proof is completed.
\end{proof}

\section{Proof of Theorem \ref{convergence_in_expectation} }
\begin{theorem*}
Let $\{\bm{x}_t\}$ be a sequence generated by Generic Adam for initial values $\bm{x}_1$, $\bm{m}_0 =\bm{0}$, and $\bm{v}_0 =\bm{\epsilon}$. Assume that $f$ and stochastic gradients $\bm{g}_t$ satisfy assumptions (A1)-(A4). Let $\tau$ be randomly chosen from $\{1, 2, \ldots, T\}$ with equal probabilities $p_\tau = 1/T$. We have the following estimate
\begin{equation}\label{1-058}
\mathbb{E}\left[\norm{\bm{\nabla} f(\bm{x}_\tau)} \right]
\leq \sqrt{\frac{C + C'\sum_{t=1}^T\alpha_t\sqrt{1-\theta_t}}{T\alpha_T}},
\end{equation}
where the constants $C$ and $C'$ are given by
\begin{equation*}
\begin{split}
C &= \frac{2C_0\sqrt{G^2+\epsilon d}}{1-\beta}\left(f(x_1) - f^* + C_3C_0d\ \chi_1\log\left(1+ \frac{G^2}{\epsilon d}\right)\right), \\
C' &= \frac{2C_0^2C_3d\sqrt{G^2 + \epsilon d}}{(1-\beta)\theta_1}.
\end{split}
\end{equation*}
\end{theorem*}
\begin{proof}
By the $L$-Lipschitz continuity of the gradient of $f$ and the descent lemma, we have
\begin{equation}
f(\bm{x}_{t+1}) \leq f(\bm{x}_t) + \langle \bm{\nabla} f(\bm{x}_t), \bm{\Delta}_t \rangle + \frac{L}{2} \norm{\bm{\Delta}_t}^2.
\end{equation}
Let $M_t := \mathbb{E}\left[\langle \bm{\nabla} f(\bm{x}_t), \bm{\Delta}_t \rangle + L \norm{\bm{\Delta}_t}^2\right]$. 
We have $\mathbb{E}[f(\bm{x}_{t+1})] \leq \mathbb{E}[f(\bm{x}_t)] + M_t$. 
Taking a sum for $t=1,2,\ldots,T$, we obtain
\begin{equation}
\mathbb{E}\left[f(\bm{x}_{T+1})\right] \leq f(\bm{x}_1) + \sum_{t=1}^T M_t.
\end{equation}
Note that $f(x)$ is bounded from below by $f^*$, so $\mathbb{E}[f(\bm{x}_{T+1})] \geq f^*$. Applying the estimate of Lemma \ref{lem1-006}, we have
\begin{equation}\label{1-062}
f^* \leq f(\bm{x}_1) + C_3\mathbb{E}\left[\sum_{t=1}^T\chi_t\norm{\frac{\sqrt{1-\theta_t}\bm{g}_t}{\sqrt{\bm{v}_t}}}^2\right]
- \frac{1-\beta}{2}\mathbb{E}\left[\sum_{t=1}^T \norm{\bm{\nabla} f(\bm{x}_t)}_{\bm{\hat{\eta}_t}}^2\right],
\end{equation}
where $C_3$ is the constant given in Lemma \ref{lem1-006}.
By applying the estimates in Lemma \ref{lem1-007} and Lemma \ref{lem1-008} for the second and third terms in the right hand side of Eq.~\eqref{1-062}, and appropriately rearranging the terms, we obtain
\begin{equation}
\begin{split}
&\left(\mathbb{E}\left[\norm{\bm{\nabla} f(\bm{x}^{T}_\tau)}\right]\right)
\leq \sqrt{\frac{C_0\sqrt{G^2+\epsilon d}}{T\alpha_T}\mathbb{E}\left[\sum_{t=1}^T\norm{\bm{\nabla} f(\bm{x}_t)}_{\bm{\hat{\eta}}_t}^2\right]}\\
\leq~& \sqrt{\frac{2C_0\sqrt{G^2+\epsilon d}}{(1-\beta)T\alpha_T}\left(f(\bm{x}_1) - f^* + 
C_3 \mathbb{E}\left[\sum_{t=1}^T\chi_t\norm{\frac{\sqrt{1-\theta_t}\bm{g}_t}{\sqrt{\bm{v}_t}}}\right]\right)}\\
\leq~& \sqrt{\frac{2C_0\sqrt{G^2+\epsilon d}}{(1-\beta)T\alpha_T}\left[f(\bm{x}_1) - f^* + C_3C_0d\ \chi_1\log\left(1+\frac{G^2}{\epsilon d}\right) - \frac{C_3C_0d}{\theta_1} \sum_{t=1}^T\alpha_t\sqrt{1-\theta_t}\right]} \\
=~& \sqrt{\frac{C + C'\sum_{t=1}^T \alpha_t\sqrt{1-\theta_t}}{T\alpha_T}},
\end{split}
\end{equation}
where 
\begin{equation*}
\begin{split}
C &= \frac{2C_0\sqrt{G^2+\epsilon d}}{1-\beta}\left(f(x_1) - f^* + C_3C_0d\ \chi_1\log\left(1+ \frac{G^2}{\epsilon d}\right)\right), \\
C' &= \frac{2C_0^2C_3d\sqrt{G^2 + \epsilon d}}{(1-\beta)\theta_1}.
\end{split}
\end{equation*}
The proof is completed.
\end{proof}

\section{Proof of Corollary \ref{poly-setting}}\label{proof-last}
\begin{corollary*}
Generic Adam with the above family of parameters converges as long as $0 < r \leq 2s < 2$, and its non-asymptotic convergence rate is given by 
\begin{equation*}
\mathbb{E}[\norm{\bm{\nabla} f(\bm{x}_\tau)}] \leq \left\{\begin{aligned}
& \mathcal{O}(T^{-r/4}), \quad &   r/2 + s < 1 \\
& \mathcal{O}(\sqrt{\log(T)/T^{1-s}}), \quad &  r/2 + s = 1 \\
& \mathcal{O}(1/T^{(1 - s)/2}), \quad &  r/2 + s > 1
\end{aligned}\right..
\end{equation*}
\end{corollary*}
\begin{proof}
It is not hard to verify that the following equalities hold:
\begin{align*}
\sum_{t=K}^T \alpha_t\sqrt{1-\theta_t} 
&= \eta\sqrt{\alpha}\sum_{t=K}^T t^{-(r/2 + s)} \nonumber\\
&=\left\{
\begin{aligned}
& \mathcal{O}(T^{1-(r/2 +s)}),  & r/2 + s < 1 \\
& \mathcal{O}(\log(T)),  &  r/2 + s = 1 \\
& \mathcal{O}(1),  & r/2 + s > 1
\end{aligned}
\right..
\end{align*} 
In this case, $T\alpha_T = \eta T^{1-s}$. 
Therefore, by Theorem \ref{convergence_in_expectation} the non-asymptotic convergence rate is given by
\[
\norm{\bm{\nabla} f(\bm{x}_\tau)}^2 \leq \left\{
\begin{aligned}
& \mathcal{O}(T^{-r/4}),  &  r/2 + s < 1 \\
& \mathcal{O}(\sqrt{\log(T)/T^{1-s}}),  & r/2 + s = 1 \\
& \mathcal{O}(1/T^{(1-s)/2}),  & r/2 + s > 1
\end{aligned}
\right..\]
To guarantee convergence, then $0 < r \leq 2s < 2$.
\end{proof}

\section{Proof of Theorem \ref{practical_adam}}
\begin{theorem*}
For any $T>0$, if we take $\alpha_t = \frac{\alpha}{\sqrt{T}},\ \beta_t = \beta, \theta_t = 1-\frac{\theta}{T}$, which satisfies $\gamma = \frac{\beta}{1-\frac{\theta}{T}} < 1$ and $\theta_t \geq \frac{1}{4}$, then it holds that
\begin{equation*}
\mathbb{E}\left[\norm{\bm{\nabla} f(\bm{x}_\tau)} \right]
\leq \sqrt{\frac{C_5}{\sqrt{T}}} = \mathcal{O}(T^{-1/4}),
\end{equation*}
where
\begin{equation*}
\begin{split}
C_5&=\frac{2\sqrt{G^2\!+\!\epsilon d}}{\alpha(1-\beta)}\left[f(x_1) - f^*+C_6d\frac{\alpha}{\sqrt{\theta}}\log\big(1\!+\! \frac{G^2}{\epsilon d}\big) + \frac{4C_6d\alpha}{\sqrt{\theta}}\right],\\
C_6& = \frac{1}{1-\sqrt{\gamma}}\left[\frac{\alpha L}{\sqrt{\theta}(1-\sqrt{\gamma})^2} + 2\big(\frac{2\beta/(1-\beta)}{\sqrt{C_1(1-\gamma)}}+1\big)^2G\right].
\end{split}
\end{equation*}
\end{theorem*}
\begin{proof}
Based on Theorem \ref{convergence_in_expectation}, by plugging $\alpha_t, \beta_t$ and $\theta_t$ in the conclusion of Theorem \ref{convergence_in_expectation}, we can get the desired result.
\end{proof}

\section{Key Lemma to prove Theorem \ref{minibatchtheorem}}

In this section, we provide the additional lemmas for the proofs of Theorem \ref{minibatchtheorem}.

\begin{lemma}
\label{reduce_variance}
With the definitions in Algorithm \ref{minibatchadam}, for any $t=1,2,\cdots,T$ we have the following estimation:
\[
\mathbb{E}[\|\bar{\bm{g}}_t-\nabla f(\bm{x}_t)\|^2] \leq \frac{\sigma^2}{s}.
\]
\end{lemma}
\begin{proof}
\[
\begin{split}
\mathbb{E}[\|\bar{\bm{g}}_t-\nabla f(\bm{x}_t)\|^2] &= \mathbb{E}\left[\left\|\frac{1}{s}\sum_{i=1}^s \bm{g}_t^{(i)} - \nabla f(\bm{x}_t)\right\|^2\right] \\
&\leq \mathbb{E}\left[\frac{1}{s^2}\sum_{i=1}^s\|\bm{g}_t^{(i)} - \nabla f(\bm{x}_t)\|^2\right] \leq \frac{\sigma^2}{s}.
\end{split}
\]
The second inequality holds, because $\bm{g}_t^{(i)}$ are independent and have the same expectation ($\mathbb{E}\left[\bm{g}_t^{(i)}\right] = \mathbb{E}[\nabla f(\bm{x}_t)]$).
\end{proof}

\begin{lemma}\label{sumgv}
The following estimate holds:
\[
\mathbb{E}\left[\sum_{t=1}^T\left\|\frac{\sqrt{1-\theta_t}\bar{\bm{g}}_t}{\sqrt{\bm{v}_t}}\right\|^2\right] 
\leq d\theta + 2d \sqrt[4]{1 + \frac{2\sigma^2}{d\epsilon s}}+\sqrt{\sqrt{\frac{2\theta}{d\epsilon T}}\mathbb{E}\left[\sqrt{\sum_{t=1}^T\|\nabla f(\bm{x}_t)\|^2}\right]}.
\]
\end{lemma}




\begin{proof}
With the similar proof in Lemma \ref{lem1-007}, it holds that
\begin{equation}
\begin{aligned}
\sum_{t=1}^T\left\|\frac{\sqrt{1-\theta_t} \bar{\bm{g}}_t }{\sqrt{\bm{v}_t}}\right\|^2 
&\leq d\log\left(1 + \frac{1}{d\epsilon}\sum_{k=1}^Tw_k\|\bar{\bm{g}}_k\|^2\right)\\
&\leq d\log\left(1 + \frac{2}{d\epsilon}\sum_{k=1}^Tw_k\left(\|\bar{\bm{g}}_k-\nabla f(\bm{x}_k)\|^2 + \|\nabla f(\bm{x}_k)\|^2\right) \right)\\
&\leq 2d\log\left(\sqrt{1 + \frac{2}{d\epsilon}\sum_{k=1}^Tw_k\|\bar{\bm{g}}_k-\nabla f(\bm{x}_k)\|^2} +\sqrt{\frac{2}{d\epsilon}\sum_{k=1}^Tw_k\|\nabla f(\bm{x}_k)\|^2} \right).
\end{aligned}
\end{equation}

Thus, by taking expectation on both side, we can obtain
\begin{equation}
\begin{split}
\mathbb{E} \left[\sum_{t=1}^T\left\|\frac{\sqrt{1-\theta_t}\bar{\bm{g}}_t}{\sqrt{\bm{v}_t}} \right\|^2\right]&\leq \mathbb{E}\left[2d\log\left(\sqrt{1 + \frac{2}{d\epsilon}\sum_{k=1}^Tw_k\|\bar{\bm{g}}_k-\nabla f(\bm{x}_k)\|^2} +\sqrt{\frac{2}{d\epsilon}\sum_{k=1}^Tw_k\|\nabla f(\bm{x}_k)\|^2} \right)\right]\\
&\leq 2d\log\left(\sqrt{1 + \frac{2}{d\epsilon}\sum_{k=1}^Tw_k\mathbb{E}\|\bar{\bm{g}}_k-\nabla f(\bm{x}_k)\|^2} +\sqrt{\frac{2}{d\epsilon}}\mathbb{E}\left[\sqrt{\sum_{k=1}^Tw_k\|\nabla f(\bm{x}_k)\|^2}\right] \right)\\
&\leq 2d\log\left(\sqrt{1 + \frac{2\sigma^2W_T}{d\epsilon s}}+\sqrt{\frac{2w_T}{d\epsilon}}\mathbb{E}\left[\sqrt{\sum_{t=1}^T\|\nabla f(\bm{x}_t)\|^2}\right] \right),\\
\end{split}
\end{equation}
\textcolor{black}{where the last inequality uses Lemma \ref{reduce_variance}.}

Meanwhile, we have
\[
\begin{split}
&\log W_T = T\log \left(\frac{T}{T-\theta}\right) \leq T \left(\frac{T}{T-\theta} - 1\right) \leq \theta,\\
&W_T \geq 1,\\
&\frac{w_T}{W_T} = 1-\theta_T = \frac{\theta}{T}.
\end{split}
\]
Therefore, it holds that
\begin{equation}
\begin{split}
 \mathbb{E} \left[\sum_{t=1}^T\left\|\frac{\sqrt{1-\theta_t}\bar{\bm{g}}_t}{\sqrt{\bm{v}_t}} \right\|^2\right] &\leq d\log W_T + 2d\log \left(\sqrt{1 + \frac{2\sigma^2}{d\epsilon s}}+\sqrt{\frac{2\theta}{d\epsilon T}}\mathbb{E}\left[\sqrt{\sum_{t=1}^T\|\nabla f(\bm{x}_t)\|^2}\right] \right)\\
&\leq d\theta + 2d \sqrt[4]{1 + \frac{2\sigma^2}{d\epsilon s}}+\sqrt{\sqrt{\frac{2\theta}{d\epsilon T}}\mathbb{E}\left[\sqrt{\sum_{t=1}^T\|\nabla f(\bm{x}_t)\|^2}\right]}.
\end{split}
\end{equation}
Hence, we obtain the desired result.
\end{proof}

\begin{lemma}\label{Mtmulbit}(Lemma \ref{lemma4} in Section \ref{proof_sketch})
By the definition of $M_t$, it holds that 
\[\sum_{t=1}^T M_t \leq C_{7}\mathbb{E}\sum_{t=1}^T \left\|\frac{\sqrt{1-\theta_t}\bar{\bm{g}}_t}{\sqrt{\bm{v}_t}}\right\|^2 - \frac{1-\beta}{2}\sum_{t=1}^T\mathbb{E}\|\nabla f\left(\bm{x}_t\right)\|_{\hat{\eta}_t}^2,
\]
where 
\[
C_{7} = \frac{1}{1-\sqrt{\gamma}}\left(\frac{\alpha^2L}{\theta\left(1-\sqrt{\gamma}\right)^2} + \frac{2\left(\frac{2\beta/\left(1-\beta\right)}{\sqrt{\left(1-\gamma\right)}}+1\right)^2G\alpha}{\sqrt{\theta}} \right).
\]
\end{lemma}
\begin{proof}
Using Lemma \ref{lem1-006}, by plugging $C_0 = C_1 = 1$, $\chi_t = \frac{\alpha}{\sqrt{\theta}}$ and $\theta_t \geq \frac{1}{4}$, it holds that
\[\sum_{t=1}^T M_t \leq C_{7}\mathbb{E}\sum_{t=1}^T \left\|\frac{\sqrt{1-\theta_t}\bar{\bm{g}}_t}{\sqrt{\bm{v}_t}}\right\|^2 - \frac{1-\beta}{2}\sum_{t=1}^T\mathbb{E}\|\nabla f\left(\bm{x}_t\right)\|_{\hat{\eta}_t}^2,
\]
where 
\[
C_{7} = \frac{1}{1-\sqrt{\gamma}}\left(\frac{\alpha^2L}{\theta\left(1-\sqrt{\gamma}\right)^2} + \frac{2\left(\frac{2\beta/\left(1-\beta\right)}{\sqrt{\left(1-\gamma\right)}}+1\right)^2G\alpha}{\sqrt{\theta}} \right).
\]
\end{proof}

\begin{lemma}\label{multitau}
(Lemma \ref{lemma5}  in Section \ref{sec_pratical})
The following estimation always holds:
\[
\mathbb{E}\left[\sqrt{\sum_{t=1}^T \|\nabla f(\bm{x}_t)\|^2}\right]^2 \leq \left( \frac{\sqrt{T}\sqrt{2\sigma^2\theta+\epsilon s d}}{\sqrt{s}\alpha} + \sqrt{\frac{2\theta}{\alpha^2}} \mathbb{E}\sqrt{\sum_{t=1}^T\|\nabla f(\bm{x}_t)\|^2}\right)\mathbb{E}\left[\sum_{t=1}^T \|\nabla f(\bm{x}_t)\|_{\hat{\eta}_t}^2\right ].
\]
\end{lemma}
\begin{proof}
First, we have
\[
\begin{split}
\sqrt{\|\hat{v}_t\|_1}  &=  \sqrt{ \sum_{k = 1}^{t-1} \left(\theta'\right)^{t-k}\left(1-\theta_k\right)\|\bar{\bm{g}}_k\|^2 + \left(1-\theta_t\right)\mathbb{E}_t(\|\bar{\bm{g}}_t\|^2)+ \left(\theta'\right)^t\epsilon d} \\
&\leq  \sqrt{ \sum_{k = 1}^{t-1} \left(1-\theta_k\right)\|\bar{\bm{g}}_k\|^2 + \left(1-\theta_t\right) \mathbb{E}_t(\|\bar{\bm{g}}_t\|^2)+ \epsilon d} \\
&\leq  \sqrt{ 2\sum_{k = 1}^{t-1} \left(1-\theta_k\right)\|\bar{\bm{g}}_k-\nabla f(\bm{x}_k)\|^2 +   \left(1-\theta_t\right)\sigma^2/s + 2\sum_{k=1}^t\left(1-\theta_k\right) \|\nabla f(\bm{x}_k)\|^2+ \epsilon d} \\
&\leq  \sqrt{ 2\sum_{k = 1}^{T-1} \left(1-\theta_k\right)\|\bar{\bm{g}}_k-\nabla f(\bm{x}_k)\|^2 + \left(1-\theta_t\right) \sigma^2/s  + \epsilon d}  + \sqrt{ 2\sum_{k=1}^T\left(1-\theta_k\right) \|\nabla f(\bm{x}_k)\|^2}. 
\end{split}
\]
Therefore,  it holds
\[
\max \sqrt{\|\hat{v}_t\|_1} \leq  \sqrt{ 2\sum_{k = 1}^{T-1} \left(1-\theta_k\right)\|\bar{\bm{g}}_k-\nabla f(\bm{x}_k)\|^2 + \left(1-\theta_T\right) \sigma^2/s  + \epsilon d} + \sqrt{ 2\sum_{k=1}^T\left(1-\theta_k\right) \|\nabla f(\bm{x}_k)\|^2}.
\]
Then we can obtain
\begin{equation}
\label{lemma32_1}
\begin{split}
&\mathbb{E} \left[ \max \sqrt{\|\hat{v}_t\|_1} \right]\\
&\leq \mathbb{E} \sqrt{ 2\sum_{k = 1}^{T-1} \left(1-\theta_k\right)\|\bar{\bm{g}}_k-\nabla f(\bm{x}_k)\|^2 +\left(1-\theta_T\right) \sigma^2/s + \epsilon d} +\mathbb{E}\sqrt{ 2\sum_{k=1}^T\left(1-\theta_k\right) \|\nabla f(\bm{x}_k)\|^2}\\
&\leq \sqrt{\mathbb{E}  2\sum_{k = 1}^{T-1} \left(1-\theta_k\right)\|\bar{\bm{g}}_k-\nabla f(\bm{x}_k)\|^2 + \left(1-\theta_T\right) \sigma^2/s + \epsilon d} +\mathbb{E}\sqrt{ 2\sum_{k=1}^T\left(1-\theta_k\right) \|\nabla f(\bm{x}_k)\|^2}\\
&\leq \frac{\sqrt{2\sigma^2\theta+\epsilon s d}}{\sqrt{s}} + \sqrt{\frac{2\theta}{T}} \mathbb{E}\sqrt{\sum_{k=1}^T\|\nabla f(\bm{x}_k)\|^2}.
\end{split}
\end{equation}
Meanwhile,  we have
\begin{equation}
\label{lemma32_2}
\begin{split}
\mathbb{E}\left[\frac{\sum_{t=1}^T \|\nabla f(\bm{x}_t)\|^2}{\max \sqrt{\|\hat{v}_t\|_1}}\right] &\leq  \mathbb{E}\left[\sum_{t=1}^T \frac{\|\nabla f(\bm{x}_t)\|^2}{\sqrt{\|\hat{v}_t\|_1}}\right] \leq \mathbb{E}\left[\sum_{t=1}^T \sum_{k=1}^d \frac{\nabla_k f(\bm{x}_t)^2}{\sqrt{\hat{v}_{t,k}}}\right]\\
&= \frac{\sqrt{T}}{\alpha} \mathbb{E}\left[\sum_{t=1}^T \sum_{k=1}^d \frac{\nabla_k f(\bm{x}_t)^2\alpha}{\sqrt{T}\sqrt{\hat{v}_{t,k}}}\right]
= \frac{\sqrt{T}}{\alpha}\mathbb{E}\left[\sum_{t=1}^T \|\nabla f(\bm{x}_t)\|_{\hat{\eta}_t}^2\right ].\\
\end{split}
\end{equation}
With inequalities \eqref{lemma32_1} and \eqref{lemma32_2}, we can obtain
\[
\begin{split}
\mathbb{E}\left[\sqrt{\sum_{t=1}^T \|\nabla f(\bm{x}_t)\|^2}\right]^2 &\leq \mathbb{E} \left[\max \sqrt{\|\hat{v}_t\|_1}\right]\mathbb{E}\left[\frac{\sum_{t=1}^T \|\nabla f(\bm{x}_t)\|^2}{\max \sqrt{\|\hat{v}_t\|_1}}\right]\\
&\leq \left( \frac{\sqrt{T}\sqrt{2\sigma^2\theta+\epsilon s d}}{\sqrt{s}\alpha} + \sqrt{\frac{2\theta}{\alpha^2}} \mathbb{E}\sqrt{\sum_{t=1}^T\|\nabla f(\bm{x}_t)\|^2}\right)\mathbb{E}\left[\sum_{t=1}^T \|\nabla f(\bm{x}_t)\|_{\hat{\eta}_t}^2\right ].
\end{split}
\]
\end{proof}

\begin{lemma}
\label{4ineq} (Lemma \ref{lemma6} in Section \ref{sec_pratical})
For any $x\in \mathbb{R}$, if $x^2 \leq (A+Bx)(C+D\sqrt{x})$, then $x \leq (4BD)^2 + 4BC  + (4AD)^{2/3} + \sqrt{4AC}$.
\end{lemma}

\begin{proof}
We discuss the solution of $x$ in 4 different situations.
First, when $Bx\geq A$ and $D\sqrt{x} \geq C$, we have
\[
x^2 \leq (A+Bx)(C+D\sqrt{x}) \leq 4BD x^{3/2}.
\]
Therefore, $x\leq (4BD)^2$.

Secondly, when $Bx\geq A$ and $D\sqrt{x}\leq C$, we have
\[
x^2 \leq (A+Bx)(C+D\sqrt{x}) \leq 4BC x.
\]
Hence, $x \leq 4BC$.

Thirdly, when $Bx \leq A$ and $D\sqrt{x} \geq C$, it holds that
\[
x^2 \leq (A+Bx)(C+D\sqrt{x}) \leq 4AD\sqrt{x}.
\]
And we can obtain $x \leq (4AD)^{2/3}$.

Last, when $Bx \leq A$ and $D\sqrt{x} \leq C$, it holds that
\[
x^2 \leq (A+Bx)(C+D\sqrt{x}) \leq 4AC.
\]
Then we have $x \leq \sqrt{4AC}$.

Therefore, combining four different conditions, we have $x \leq (4BD)^2 + 4BC + (4AD)^{2/3} + \sqrt{4AC}$.
\end{proof}

\section{Proof of Theorem \ref{minibatchtheorem}}
\begin{theorem*}
For any $T>0$, if we take $\alpha_t = \frac{\alpha}{\sqrt{T}}$, $\beta_t = \beta$ , $\theta_t = 1 - \frac{\theta}{T}$, which satisfies $\gamma = \frac{\beta_t}{\theta_t} < 1$ and $\theta_t \geq \frac{1}{4}$, then there exists $t \in \{1,2,\cdots,T\}$ such that 
\[
\begin{split}
&\mathbb{E} \left[\|\nabla f(\bm{x}_t)\|\right] 
\leq \frac{1}{\sqrt{T}} \left( (4C_{8}C_{10})^{1/2} + (4C_{8}C_{11})^{2/3} + (4C_{9}C_{10}) + (4C_{9}C_{11})^2) \right),
\end{split}
\]
where
\[
\begin{aligned}
C_{8} &= \frac{\sqrt{T}\sqrt{2\sigma^2\theta+\epsilon s d}}{\sqrt{s}\alpha} ,\ 
C_{9} = \sqrt{\frac{2\theta}{\alpha^2}},\ 
C_{10} = \frac{2C_4}{1-\beta}\sqrt[4]{\frac{2\theta}{d\epsilon T}},\\
C_{11} &= \frac{2}{1-\beta}\left(f(x_1) - f^* + C_4 d\theta + 2 C_4 d \sqrt[4]{1 + \frac{2\sigma^2}{d\epsilon s}}\right),\\
C_{7} &= \frac{1}{1-\sqrt{\gamma}}\left(\frac{\alpha^2L}{\theta\left(1-\sqrt{\gamma}\right)^2} + \frac{2\left(\frac{2\beta/\left(1-\beta\right)}{\sqrt{1-\gamma }}+1\right)^2G\alpha}{\sqrt{\theta}} \right). \\
\end{aligned}
\]
In addition, by taking $\epsilon = \frac{1}{sd}$, it holds that
\[
\begin{split}
\mathbb{E} \left[\|\nabla f(\bm{x}_t)\|\right]
&=  \mathcal{O}(T^{-1}s^{1/2}d^{1/2} + T^{-1/2}d + T^{-1/3}s^{-1/4} + T^{-1/4}s^{-1/4}d^{1/2}).
\end{split}
\]
\end{theorem*}

\begin{proof}
First, according to the gradient Lipschitz condition of $f$, it holds
\[
\begin{split}
f(x_{t+1}) &\leq f(\bm{x}_t) + \langle\nabla f(\bm{x}_t), \Delta_t\rangle + \frac{L}{2} \|\Delta_t\|^2\\
& =  f(\bm{x}_t) + \langle\nabla f(\bm{x}_t), \Delta_t\rangle + L \|\Delta_t\|^2 .
\end{split}
\]
Recall that $M_t = \mathbb{E}[\langle\nabla f(\bm{x}_t), \Delta_t\rangle+L\|\Delta_t\|^2$. 
Then we have
\[
\begin{split}
    f^*&\leq \mathbb{E}[fx_{T+1})] \leq f(x_1) + \sum_{t=1}^T M_t\\
    & \leq f(x_1) + \sum_{t=1}^T M_t\leq f(x_1) + C_{7}\mathbb{E}\sum_{t=1}^T \left\|\frac{\sqrt{1-\theta_t}\bar{\bm{g}}_t}{\sqrt{\bm{v}_t}}\right\|^2 - \frac{1-\beta}{2}\sum_{t=1}^T\mathbb{E}\|\nabla f\left(\bm{x}_t\right)\|_{\hat{\eta}_t}^2.\\
\end{split}
\]
Using Lemma \ref{sumgv}, \ref{Mtmulbit} and \ref{multitau} with rearranging the corresponding terms, we have
\begin{equation}
\label{qua_inq}
\begin{split}
  &\mathbb{E}\left[\sqrt{\sum_{t=1}^T \|\nabla f(\bm{x}_t)\|^2}\right]^2\\
  &\leq \left( \frac{\sqrt{T}\sqrt{2\sigma^2\theta+\epsilon s d}}{\sqrt{s}\alpha} + \sqrt{\frac{2\theta}{\alpha^2}} \mathbb{E}\sqrt{\sum_{t=1}^T\|\nabla f(\bm{x}_t)\|^2}\right)\mathbb{E}\left[\sum_{t=1}^T \|\nabla f(\bm{x}_t)\|_{\hat{\eta}_t}^2\right ]\\
  &\leq \left( \frac{\sqrt{T}\sqrt{2\sigma^2\theta+\epsilon s d}}{\sqrt{s}\alpha} + \sqrt{\frac{2\theta}{\alpha^2}}  \mathbb{E}\sqrt{\sum_{t=1}^T\|\nabla f(\bm{x}_t)\|^2}   \right)\frac{2}{1-\beta}\left(f(x_1) - f^* + C_7\mathbb{E}\sum_{t=1}^T\left\|\frac{\sqrt{1-\theta_t}\bm{g}_t}{\sqrt{\bm{v}_t}}\right\|^2\right)\\
&\leq \left( \frac{\sqrt{T}\sqrt{2\sigma^2\theta+\epsilon s d}}{\sqrt{s}\alpha} + \sqrt{\frac{2\theta}{\alpha^2}}  \mathbb{E}\sqrt{\sum_{t=1}^T\|\nabla f(\bm{x}_t)\|^2}   \right)\\
  &\qquad \frac{2}{1-\beta}\left(f(x_1) - f^* + C_7 d\theta + C_7 2d \sqrt[4]{1 + \frac{2\sigma^2}{d\epsilon s}}+C_7\sqrt{\sqrt{\frac{2\theta}{d\epsilon T}}\mathbb{E}\left[\sqrt{\sum_{t=1}^T\|\nabla f(\bm{x}_t)\|^2}\right]}\right).
\end{split}
\end{equation}
Before using Lemma \ref{4ineq}, we list the order of 4 terms in Lemma \ref{4ineq} as follows:
\[
\begin{split}
&(4BD)^2 = O\left(\sqrt{\frac{1}{d\epsilon T}}\right),\ 4BC = O(d + d (ds\epsilon)^{-1/4} ),\\
&(4AD)^{2/3} = O(T^{1/6}(s^{-1/2}(\epsilon d)^{-1/4} + (\epsilon d)^{1/4})),\\
&\sqrt{4AC} = O(T^{1/4} d^{1/2} (s^{-1/4} + (\epsilon d)^{1/4})(1 + (ds\epsilon)^{-1/8})).
\end{split}
\]
Then, it holds that  $\sqrt{T} {\min}_t \mathbb{E} \|\nabla f(\bm{x}_t)\| \leq \mathbb{E}\left[\sqrt{\sum_{t=1}^T \|\nabla f(\bm{x}_t)\|^2}\right] $. By dividing $\sqrt{T}$ on both side, we can get the desired result.
\end{proof}

\vskip 0.2in
\bibliography{egbib}
\end{document}